\documentclass[10pt,journal]{IEEEtran}

\usepackage{amsmath}
\usepackage{amssymb}
\usepackage{amsthm}
\usepackage[ruled, linesnumbered]{algorithm2e}
\usepackage{bbm}
\usepackage{bm}
\usepackage{breakcites}
\usepackage{booktabs}
\usepackage{cases}
\usepackage[font=small, justification=raggedright]{caption}
\usepackage{color}
\usepackage{enumitem}
\usepackage{float}
\usepackage{graphicx}
\usepackage{graphics}
\usepackage{grffile} 
\usepackage{hyperref}
\usepackage{lscape}
\usepackage{listings}
\usepackage{latexsym}
\usepackage{mathrsfs}
\usepackage{multirow}
\usepackage{moresize}
\usepackage{pstricks}
\usepackage{pst-grad} 
\usepackage{pst-plot} 
\usepackage{verbatim}
\usepackage{syntonly}
\usepackage{subfig}
\usepackage{filecontents}

\newtheorem{theorem}{Theorem}
\newtheorem{lemma}{Lemma}
\theoremstyle{definition}

\newtheorem{remark}{Remark}

\newcommand{\changes}[1]{\textcolor{black}{#1}}

\begin{document}
    %
    \title{Subsampling Generative Adversarial Networks: Density Ratio Estimation in Feature Space with Softplus Loss}
    %
    %
    %
    
    \author{Xin~Ding,
        Z.~Jane~Wang,~\IEEEmembership{Fellow,~IEEE}
        and~William~J.~Welch
        \thanks{Xin Ding and William J. Welch are with the Department
            of Statistics, University of British Columbia, Vancouver,
            BC, V6T 1Z4 Canada (e-mail: xin.ding@stat.ubc.ca, will@stat.ubc.ca) \textit{(Corresponding author: Xin Ding)}.}
        \thanks{Jane Z. Wang is with the Department
            of Electrical and Computer Engineering, University of British Columbia, Vancouver, BC, V6T 1Z4 Canada (e-mail: zjanew@ece.ubc.ca).}
        \thanks{Manuscript received \today; revised \today.}}

    \markboth{Preprint}
    {Shell \MakeLowercase{\textit{et al.}}: Bare Demo of IEEEtran.cls for Journals}
    
    \maketitle
    
    \begin{abstract}
        Filtering out unrealistic images from trained generative adversarial networks (GANs) has attracted considerable attention recently. Two density ratio based subsampling methods---Discriminator Rejection Sampling (DRS) and Metropolis-Hastings GAN (MH-GAN)---were recently proposed, and their effectiveness in improving GANs was demonstrated on multiple datasets. However, DRS and MH-GAN are based on discriminator-based density ratio estimation (DRE) methods, so they may not work well if the discriminator in the trained GAN is far from optimal. Moreover, they do not apply to some GANs (e.g., MMD-GAN). In this paper, we propose a novel Softplus (SP) loss for DRE. Based on it, we develop a sample-based DRE method in a feature space learned by a specially designed and pre-trained ResNet-34, termed DRE-F-SP. We derive the rate of convergence of a density ratio model trained under the SP loss. Then, we propose three density ratio based subsampling methods for GANs based on DRE-F-SP. Our subsampling methods do not rely on the optimality of the discriminator and are suitable for all types of GANs. We empirically show our subsampling approach can substantially outperform DRS and MH-GAN on a synthetic dataset, \changes{CIFAR-10, MNIST and CelebA}, using multiple GANs.
    \end{abstract}
    
    \begin{IEEEkeywords}
        Generative adversarial networks, density ratio estimation, subsampling GANs
    \end{IEEEkeywords}

    %
    \IEEEpeerreviewmaketitle

    \section{Introduction}
    \IEEEPARstart{G}{\textit{enerative}} \textit{adversarial networks} (GANs) first introduced by \cite{goodfellow2014generative} are well-known and powerful generative models for image synthesis and have been applied to various types of image-related tasks \cite{wang2018perceptual, wang2018deeply, hsu2018sigan, quan2018compressed, gao2019universal, wei2019facial}. The vanilla GANs proposed by \cite{goodfellow2014generative} consist of two neural networks: a generator and a discriminator. The generator is trained to generate fake images to fool the discriminator while the discriminator is trained to distinguish fake images from real ones. To enhance the quality of fake images generated from a vanilla GAN, many subsequent works have aimed to improve its training, such as large-scale training (e.g., BigGAN \cite{brock2018large}), novel normalization (e.g., SN-GAN \cite{miyato2018spectral}), advanced GAN architectures (e.g., SA-GAN \cite{zhang2019self}), and different loss functions (e.g., WGAN loss \cite{arjovsky2017wasserstein, gulrajani2017improved} based on the \textit{Wasserstein distance} \cite{villani2008optimal} and MMD-GAN loss \cite{li2017mmd} based on \textit{maximum mean discrepancy} \cite{gretton2012kernel}). Instead of improving the training procedure, we are more interested in this article in post-processing fake images from a trained GAN, i.e., subsampling fake images to filter out unrealistic images.
    
    Two density ratio based subsampling methods for GANs were proposed recently and demonstrated to be effective. \textit{Discriminator Rejection Sampling} (DRS) \cite{azadi2018discriminator} is based on \textit{rejection sampling} (RS) to accept or reject a fake image generated from a trained GAN, and
    \textit{Metropolis-Hastings GAN} (MH-GAN) \cite{turner2018metropolis} utilizes the \textit{Metropolis-Hastings algorithm} (MH) to sample from a trained GAN. Denote the true data distribution by $p_r(\bm{x})$ and the distribution of fake images by $p_g(\bm{x})$. The key step of these two subsampling methods is \textit{density ratio estimation} (DRE) where the density ratio $p_r(\bm{x})/p_g(\bm{x})$ is estimated. When GANs are trained with the standard adversarial loss function defined in \cite{goodfellow2014generative}, given a fixed generator, the optimal discriminator $D^*(\bm{x})$ and the density ratio $r(\bm{x})=p_r(\bm{x})/p_g(\bm{x})$ satisfy the  relationship
    \begin{equation}
    \label{eq:relation_DR_Disc}
    r(\bm{x})=\frac{p_r(\bm{x})}{p_g(\bm{x})}=\frac{D^*(\bm{x})}{1-D^*(\bm{x})}.
    \end{equation}
    This property is leveraged by \cite{azadi2018discriminator, turner2018metropolis} to estimate the density ratio $p_r(\bm{x})/p_g(\bm{x})$, and hence DRS and MH-GAN rely heavily on an assumption of optimality of the discriminator.  In practice, however, the quality of the discriminator is difficult to guarantee in GAN training. Moreover, this property no longer holds if a GAN is trained with other loss functions such as the WGAN loss \cite{arjovsky2017wasserstein, gulrajani2017improved} or the MMD-GAN loss \cite{li2017mmd}. Thus, strictly speaking, DRS and MH-GAN are not suitable for WGANs and MMD-GANs. To reduce the reliance of DRS and MH-GAN on the quality of a trained discriminator and broaden their application to different GANs, direct estimation of the density ratio from samples is needed.
    
    Previous research on density ratio estimation for images includes \cite{nam2015direct, khan2019deep, grover2019bias}. \cite{nam2015direct, khan2019deep} propose use of a \textit{convolutional neural network} (CNN) to model the true density ratio function. \cite{nam2015direct} models the density ratio function by a CNN with only two convolutional layers and fits this shallow CNN under the \textit{unconstrained least-squares importance fitting} (uLSIF) loss function. A deeper CNN structure that contains six convolutional layers along with two new loss functions (called DSKL and BARR, respectively) are proposed by \cite{khan2019deep}. However, the loss functions used by \cite{nam2015direct, khan2019deep} to train CNNs are not bounded from below. Hence, if \textit{stochastic gradient descent} (SGD) or a variant is used for the optimization, the training loss keeps decreasing without converging as long as the CNN has enough capacity. Rather than using a neural network to model the true density ratio function, \cite{grover2019bias} leverages the relationship between the true density ratio function and a Bayes optimal classifier (BOC) to estimate density ratios. The BOC is learned from samples and used for classifying real and fake samples. However, this method suffers from the difficulty of achieving the optimality of the BOC. 
    
    In this paper, we focus on improving density ratio based subsampling methods for GANs \cite{azadi2018discriminator, turner2018metropolis} by proposing a novel sample-based \textit{density ratio estimation} (DRE) method. Our contributions can be summarized as follows:
    \begin{itemize}
    	\item We propose in Section \ref{sec:SP_loss} a novel loss function called \textit{Softplus} (SP) loss for density ratio estimation with neural networks.
    	\item We derive in Section \ref{sec:convergence} the rate of convergence of a density ratio model trained with the SP loss under Bregman divergence.
    	\item In Section \ref{sec:DRE-F} we further propose a density ratio estimation method for image data: \textit{Density Ratio Estimation in Feature Space with Softplus Loss} (DRE-F-SP). We model the true density ratio function by a 5-layer multilayer perceptron (MLP) in a feature space learned by a specially designed and pre-trained ResNet-34, and the MLP is trained under the SP loss. 
    	\item Then, in Section \ref{sec:DRE-F-SP_Sampling}, we incorporate the proposed DRE-F-SP into the RS and MH schemes of \cite{azadi2018discriminator, turner2018metropolis}. We also apply the \textit{sampling-importance resampling} (SIR) scheme based on the DRE-F-SP because of the high efficiency of the SIR. These three subsampling methods for GANs are denoted by DRE-F-SP+RS, DRE-F-SP+MH, and DRE-F-SP+SIR, respectively.
    	\item Finally, in Section \ref{sec:experiment}, we conduct experiments on a synthetic dataset and the CIFAR-10 dataset to justify our proposed subsampling methods. The experiments show that they can substantially outperform DRS and MH-GAN. In addition to the main study, we also conduct an ablation study on both the synthetic dataset and the CIFAR-10 dataset, respectively, to demonstrate that the novel SP loss is one source of the improvement. On CIFAR-10, we conduct a second ablation study to show the density ratio estimation in the feature space is another source of the improvement. Moreover, our experiments show that our subsampling methods can improve different types of GANs, e.g., DCGAN, WGAN-GP, and MMD-GAN. \changes{Extra experiments on MNIST and CelebA in the supplemental material also demonstrate the superiority of our methods.} Codes for these experiments can be found at \url{https://github.com/UBCDingXin/DDRE_Sampling_GANs}. 
    \end{itemize}
    
    \section{Related works}
    \subsection{Generative Adversarial Networks}
    A vanilla GAN \cite{goodfellow2014generative} is composed of two neural networks---a generator $G(\bm{z},\theta)$ and a discriminator $D(\bm{x},\phi)$ ($\theta$ and $\phi$ are parameters). The generator takes as input a sample from a simple prior $\bm{z}\in\mathcal{Z}\sim q(\bm{z})$ (e.g., $N(\bm{0},\bm{I})$) and outputs a fake image $\bm{x}^g\in\mathcal{X}\sim p_g(\bm{x})$. The discriminator takes an image $\bm{x}$ from $\mathcal{X}$ as input and outputs the probability $D(\bm{x})$ that $\bm{x}$ is from $p_r(\bm{x})$. These two networks are trained alternately with opposite objective functions. The discriminator is trained to assign a high probability to a real image $\bm{x}^r\sim p_r(\bm{x})$ but a low probability to a fake image $\bm{x}^g\sim p_g(\bm{x})$. Conversely, the training purpose of the generator $G(\bm{z},\theta)$ is to make the discriminator assign a high probability to a fake image $\bm{x}^g$, which is equivalent to making $p_g(\bm{x})$ as close as possible to $p_r(\bm{x})$. The standard loss functions defined by \cite{goodfellow2014generative} for the generator and the discriminator are shown as follows:
    \begin{equation}
    \label{eq:standard_gan_loss}
    \begin{aligned}
    L_D(\phi)&=-\mathbbm{E}_{\bm{x}\sim p_r}\left[\log D(\bm{x},\phi)\right]\\
    &\quad -\mathbbm{E}_{\bm{z}\sim q}\left[\log (1-D(G(\bm{z},\theta),\phi)) \right],\\
    L_G(\theta)&=-\mathbbm{E}_{\bm{z}\sim q}\left[\log D(G(\bm{z},\theta),\phi) \right].
    \end{aligned}
    \end{equation}
    It has been demonstrated by \cite{goodfellow2014generative} that, for a fixed $G$, minimizing $L_D$ results in the optimal discriminator $D^*$:
    \begin{equation}
    \label{eq:optimal_D}
    D^*(\bm{x}) = \frac{p_r(\bm{x})}{p_r(\bm{x})+p_g(\bm{x})}. 
    \end{equation}
    Thus, Eq.\eqref{eq:relation_DR_Disc} can be obtained by simply rearranging Eq.\eqref{eq:optimal_D}. If we denote all layers before the final {\em Sigmoid} layer in a discriminator $D(\bm{x})$ by $\tilde{D}(\bm{x})$, then $D(\bm{x})$ can be rewritten as
    \begin{equation*}
    D(\bm{x}) = \sigma(\tilde{D}(\bm{x})) =\frac{1}{1+e^{-\tilde{D}(\bm{x})}},
    \end{equation*}
    where $\sigma$ denotes a {\em Sigmoid} function. Thus, Eq.\eqref{eq:relation_DR_Disc} can also be rewritten as
    \begin{equation}
    \label{eq:DR_and_D_tilde}
    r(\bm{x})=e^{\tilde{D}^*(\bm{x})}.
    \end{equation}
    
    There are several variants of vanilla GANs, such as WGANs and MMD-GANs. 
    Comparing with vanilla GANs, the generator and discriminator of these variants have different structures and are trained with loss functions different from Eq.\eqref{eq:standard_gan_loss}. In this case, the optimal discriminator $D^*$ in Eq.\eqref{eq:optimal_D} may not be obtained so computing the density ratio as Eq.\eqref{eq:DR_and_D_tilde} may not be applicable. Please see \cite{arjovsky2017wasserstein, gulrajani2017improved, li2017mmd} for more details. 
    
    \subsection{Discriminator Rejection Sampling and Metropolis-Hastings GAN} \label{sec:DRS_and_MH-GAN}
    \textit{Discriminator Rejection Sampling} (DRS) \cite{azadi2018discriminator} filters out bad fake images by using rejection sampling and discriminator-based density ratio estimation. To estimate the density ratio required in rejection sampling, DRS \cite{azadi2018discriminator} takes a pre-trained GAN and proposes to further train the discriminator only on some hold-out real images and the same number of fake images with early stopping.  Then, the trained discriminator is assumed to be the optimal discriminator $D^*(\bm{x})$, and a density ratio at $\bm{x}$ can be computed by evaluating $\exp(\tilde{D}^*(\bm{x}))$ in Eq.\eqref{eq:DR_and_D_tilde}.  A key step in the rejection sampling of DRS is to estimate $M=\max_{\bm{x}}p_r(\bm{x})/p_g(\bm{x})$ by evaluating $\exp(\tilde{D}^*(\bm{x}))$ on 10,000 further fake images. This $M$ may be replaced by a larger density ratio if we find one in subsequent sampling. In regular rejection sampling, a proposed fake sample $\bm{x}^\prime$ is accepted with probability 
    \begin{equation}
    \label{eq:RS_accept_prob}
    p=\frac{p_r(\bm{x}^\prime)}{Mp_g(\bm{x}^\prime)}=\frac{r(\bm{x}^\prime)}{M}.
    \end{equation}
    However, to deal with acceptance probabilities that are too small when the target distribution is high dimensional, \cite{azadi2018discriminator} uses another acceptance probability 
    \begin{equation*}
    p=\sigma(\hat{F}(\bm{x},M,\epsilon,\gamma)),
    \end{equation*}
    where 
    \begin{equation}
    \label{eq:DRS_F_hat}
    \begin{aligned}
    &\hat{F}(\bm{x},M,\epsilon,\gamma)\\
    =&\tilde{D}^*(\bm{x})-\log M-\log\left(1-e^{\tilde{D}^*(\bm{x})-\log M-\epsilon} \right)-\gamma,\\
    \triangleq&F(\bm{x})-\gamma,
    \end{aligned}
    \end{equation}
    $\epsilon$ is a small constant (e.g., $10^{-14}$) for numerical stability and $\gamma$ is a hyper-parameter to control the overall acceptance probability. 
    
    \textit{Metropolis-Hastings GAN} (MH-GAN) \cite{turner2018metropolis} applies the Metropolis-Hastings algorithm to correct the sampling bias of an imperfect generator with information from a calibrated discriminator $D^*$. To be more specific, MH-GAN constructs a Markov chain $\{\bm{x}_1,\bm{x}_2,\ldots\}$ where $\bm{x}_k$ is generated as follows: (1) Draw $\bm{x}^\prime$ from the proposal distribution $p(\bm{x}|\bm{x}_{k-1})=p_g(\bm{x})$ and $u$ from $\textrm{Uniform}(0,1)$; (2) The acceptance probability $p$ is defined as
    \begin{equation}
    \label{eq:MH_accept_prob}
    p=\min\left(1, \frac{r(\bm{x}^\prime)}{r(\bm{x}_{k-1})}\right),
    \end{equation}
    where $r(\bm{x}_{k-1})$ and $r(\bm{x}^\prime)$ are computed based on Eq.\eqref{eq:relation_DR_Disc}; (3) If $u\leq p$, then $\bm{x}_{k}=\bm{x}^\prime$; otherwise $\bm{x}_{k}=\bm{x}_{k-1}$.    This generation-acceptance/rejection procedure is recursively repeated $K$ times and results in a Markov chain of length $K$. To produce independent filtered images, MH-GAN builds one Markov chain per filtered image and for each chain only the last image $\bm{x}_K$ is taken. MH-GAN also includes calibration to refine the trained discriminator. It places either a logistic, isotonic, or beta regression on top of $\tilde{D}$ and trains the regression model on $n_{\text{hold}}$ fake images and $n_{\text{hold}}$ hold-out real images to distinguish between fake and real. Then the calibrated discriminator is built via $D^*(\bm{x})=C(\tilde{D}(\bm{x}))$, where $C$ is the trained regression model. In our experiment, by default, we use the calibrated discriminator to compute density ratios when implementing MH-GAN. This calibration technique can also be applied to WGANs (or similar GANs) to let the calibrated discriminator output class probabilities rather than class scores. However, this calibration is not suitable for MMD-GAN because the ``discriminator'' of MMD-GAN outputs a reconstructed image instead of class scores or class probabilities.
    
    Both of the above methods rely heavily on the optimality of the discriminator to estimate the density ratio, but such optimality is hard to guarantee in practice. In this paper, we focus on improving the density ratio estimation step while keeping most of the other procedures in DRS and MH-GAN unchanged.
    
    \subsection{Sampling-Importance Resampling} \label{sec:SIR}
    
    When a target distribution $p_r(\bm{x})$ is difficult to sample directly,
    sampling-importance resampling (SIR) \cite{robert2010introducing, bolic2005resampling} generates samples from an easier proposal distribution $p_g(\bm{x})$ and then takes subsamples. Specifically, SIR generates  $\left\{\bm{x}_1^g,\cdots,\bm{x}_n^g\right\}$ from $p_g$ and takes subsamples with replacement from them using probability 
    $$w_i=\frac{p_r(\bm{x}_i^g)/p_g(\bm{x}_i^g)}{\sum_{i=1}^np_r(\bm{x}_i^g)/p_g(\bm{x}_i^g)} $$ 
    for $\bm{x}_i^g$.
    The probability $w_i$ is also known as the normalized importance weight for $\bm{x}_i^g$. If $n$ is large enough, resampling from $\left\{\bm{x}_1^g,\cdots,\bm{x}_n^g\right\}$ in this way approximates samples generated from $p_r$.
    
    \subsection{Density Ratio Estimation in Pixel Space}\label{sec:DRE_images}
    To estimate the density ratio for a given image $\bm{x}$, \cite{nam2015direct, khan2019deep} model the true density ratio function $r(\bm{x})=p_r(\bm{x})/p_g(\bm{x})$ by a CNN $\hat{r}(\bm{x};\bm{\alpha})$, i.e.,
    \begin{equation}
    \label{eq:r_hat}
    \hat{r}(\bm{x};\bm{\alpha})\longrightarrow r(\bm{x}),
    \end{equation}
    where $\bm{\alpha}$ is the learnable parameter. The CNN $\hat{r}(\bm{x};\bm{\alpha})$ is trained on samples from both $p_r$ and $p_g$ to map a given image to its density ratio and the estimated density ratio at $\bm{x}$ can be obtained by evaluating the fitted CNN at $\bm{x}$. This type of density ratio estimation method consists of two components: a neural network $\hat{r}(\bm{x};\bm{\alpha})$ (used to model the true density ratio function $r(\bm{x})$) and a loss function. \cite{nam2015direct} proposes a CNN with only two convolutional layers to model the density ratio function and trains this CNN by the uLSIF loss defined as
    \begin{equation}
    \label{eq:uLSIF_loss}
    \widehat{L}_{\text{uLSIF}}(\bm{\alpha}) = \frac{1}{2n_g}\sum_{i=1}^{n_g}\hat{r}^2(\bm{x}_i^g;\bm{\alpha}) - \frac{1}{n_r}\sum_{i=1}^{n_r}\hat{r}(\bm{x}^r_i;\bm{\alpha}).
    \end{equation}
    We denote this DRE method by \textbf{DRE-P-uLSIF}, where \textbf{P} stands for working in the pixel space in contrast to the feature-based methods of Section \ref{sec:DRE-F}. There are two reasons, however, why uLSIF loss is not well-defined for training a neural network to model the true density ratio function:
    \begin{enumerate}[label=\alph*)]
    	\item Due to the strong expression capacity of neural networks, training $\hat{r}(\bm{x};\bm{\alpha})$ under the uLSIF loss may encourage $\hat{r}(\bm{x};\bm{\alpha})$ to memorize all training data by simply assigning almost zero density ratio to all fake images (no matter realistic or not) but very large density ratio to all real images. In this case, if we use the SGD optimizer or its variants, the training loss \changes{may keep decreasing without converging}. 
    	
    	\item To prevent $\hat{r}(\bm{x};\bm{\alpha})$ from simply ``memorizing" training data, we may add extra constraints on $\hat{r}(\bm{x};\bm{\alpha})$. Since $\int r(\bm{x})p_g(\bm{x})d\bm{x} = \int p_r(\bm{x})d\bm{x} = 1$, a natural constraint on $\hat{r}(\bm{x};\bm{\alpha})$ is
    	\begin{equation}
    	\label{eq:expectation_dr_fake_samples}
    	\int \hat{r}(\bm{x};\bm{\alpha})p_g(\bm{x})d\bm{x} = 1.
    	\end{equation}
    	An empirical approximation to this constraint is
    	\begin{equation}
    	\frac{1}{n_g}\sum_{i=1}^{n_g}\hat{r}(\bm{x}_i^g;\bm{\alpha}) = 1.
    	\end{equation}
    	We can apply this constraint by adding a penalty term to the uLSIF loss, i.e., 
    	\begin{equation}
    	\label{eq:optim_uLSIF_penlaty}
    	\min_{\bm{\alpha}} \left\{\widehat{L}_{\text{uLSIF}}(\bm{\alpha}) + \lambda \hat{Q}(\bm{\alpha})\right\},
    	\end{equation} 
    	where
    	\begin{equation}
    	\label{eq:penalty_hat}
    	\hat{Q}(\bm{\alpha}) = \left(\frac{1}{n_g}\sum_{i=1}^{n_g}\hat{r}(\bm{x}_i^g;\bm{\alpha}) - 1\right)^2.
    	\end{equation}
    	However, due to the unbounded nature of the uLSIF loss (the range of $\widehat{L}_{\text{uLSIF}}(\bm{\alpha})$ is $(-\infty, \infty)$ given $\hat{r}(\bm{x})\geq 0$), the penalty term $\lambda \hat{Q}(\bm{\alpha})$ \changes{can not stop $\widehat{L}_{\text{uLSIF}}(\bm{\alpha})$ from going to negative infinity during training}, no matter how large $\lambda$ is. In this case, the penalty term has no effect.
    \end{enumerate}

    Two new DRE methods are given by \cite{khan2019deep} in which a 6-layer-CNN is adapted to model the true density ratio function.  The methods differ only in their two new training loss functions---DSKL and BARR---and they are denoted by \textbf{DRE-P-DSKL} and \textbf{DRE-P-BARR}, respectively.  The two new loss functions are defined as follows:
    \begin{equation}
    \label{eq:DSKL}
    \widehat{L}_{\text{DSKL}}(\bm{\alpha})=-\frac{1}{n_r}\sum_{i=1}^{n_r}\log \hat{r}(\bm{x}_i^r;\bm{\alpha})+\frac{1}{n_g}\sum_{i=1}^{n_g}\log\hat{r}(\bm{x}_i^g;\bm{\alpha}),
    \end{equation}
    \begin{equation}
    \label{eq:BARR}
    \widehat{L}_{\text{BARR}}(\bm{\alpha})=-\frac{1}{n_r}\sum_{i=1}^{n_r}\log \hat{r}(\bm{x}_i^r;\bm{\alpha})+\lambda\left|\frac{1}{n_g}\sum_{i=1}^{n_g}\hat{r}(\bm{x}_i^g;\bm{\alpha})-1\right|.
    \end{equation}
    \cite{khan2019deep} suggests setting $\lambda=10$ in Eq.\eqref{eq:BARR}. Unfortunately, these two new loss functions still suffer from the same problems besetting uLSIF, so they are still unsuitable for density ratio estimation with neural networks. 
    
    Different from \cite{nam2015direct, khan2019deep}, \cite{grover2019bias} estimates the density ratio using a relationship between the true density ratio $r(\bm{x})$ and a BOC $c(\bm{x})$:
    \begin{equation}
    r(\bm{x})=\frac{p_r(\bm{x})}{p_g(\bm{x})}=\gamma\frac{c(\bm{x})}{1-c(\bm{x})},
    \end{equation}
    where $\gamma$ is a prior odds that an image is fake and $c$ is a binary classifier which distinguishes between images from $p_r$ and $p_g$. A CNN is trained by \cite{grover2019bias} on an equal number of real and fake samples. This trained CNN is used as the BOC and $\gamma$ is assumed to be 1.
    
    \subsection{Fitting Density Ratio Models Under Bregman Divergence}\label{sec:BR_div}
    The uLSIF loss \eqref{eq:uLSIF_loss} is a special case of the Bregman (BR) divergence, based on which we propose a novel loss call Softplus loss in Section \ref{sec:SP_loss}. BR divergence \cite{bregman1967relaxation, varshney2011bayes}, an extension of the squared Euclidean distance, measures the distance between two points $t^*$ and $t$ in terms of a function $f$ as follows:
    \begin{equation}
    \label{eq:BR_div}
    BR^\prime_f(t^*|t)= f(t^*)-f(t)-\triangledown f(t)(t^*-t),
    \end{equation}
    where $f:\Omega\rightarrow \mathbbm{R}$ is a continuously differentiable and strictly convex function defined on a closed set $\Omega$. Assume $f$ is defined on $\Omega=[\min(m_1,m_2), \allowbreak \max(M_1,M_2)]$, where $m_1=\min r(\bm{x})$, $m_2=\min \hat{r}(\bm{x};\bm{\alpha})$, $M_1=\max r(\bm{x})$, $M_2=\max \hat{r}(\bm{x};\bm{\alpha})$, and $\hat{r}(\bm{x};\bm{\alpha})$ is a density ratio model with a learnable parameter $\bm{\alpha}$. The BR divergence defined based on $f$ is used by \cite{sugiyama2012density} to quantify the discrepancy between $r(\bm{x})$ and $\hat{r}(\bm{x};\bm{\alpha})$ as follows:
    \begin{equation}
    \label{eq:BR_Prime_div_DR}
    \begin{aligned}
    BR_f^\prime(\bm{\alpha}) = &\int{p_g(x)}\left[ f(r(\bm{x}))-f(\hat{r}(\bm{x};\bm{\alpha})) - \triangledown f(\hat{r}(\bm{x}))(r(\bm{x})\right.\\
    &\quad\left. -\hat{r}(\bm{x};\bm{\alpha})) \right]d\bm{x}\\
    =&\quad C+BR_f(\bm{\alpha}),
    \end{aligned}
    \end{equation} 
    where $C=\int p_g(\bm{x})f(r(\bm{x}))d\bm{x}$ does not depend on $\hat{r}(\bm{x};\bm{\alpha})$ and 
    \begin{equation}
    \label{eq:BR_div_DR}
    \begin{aligned}
    BR_f(\bm{\alpha}) &= \int p_g(\bm{x})\triangledown f(\hat{r}(\bm{x};\bm{\alpha}))\hat{r}(\bm{x};\bm{\alpha})d\bm{x} \\
    &- \int p_g(\bm{x})f(\hat{r}(\bm{x};\bm{\alpha}))d\bm{x} - \int p_r(\bm{x})\triangledown f(\hat{r}(\bm{x};\bm{\alpha})) d\bm{x}.
    \end{aligned}
    \end{equation}
    An empirical approximation to $BR_f(\bm{\alpha})$ is 
    \begin{equation}
    \label{eq:BR_div_DR_emp}
    \begin{aligned}
    \widehat{BR}_f(\bm{\alpha}) &= \frac{1}{n_g}\sum_{i=1}^{n_g}\triangledown f(\hat{r}(\bm{x}_i^g;\bm{\alpha}))\hat{r}(\bm{x}_i^g;\bm{\alpha}) \\
    &- \frac{1}{n_g}\sum_{i=1}^{n_g}f(\hat{r}(\bm{x}_i^g;\bm{\alpha})) - \frac{1}{n_r}\sum_{i=1}^{n_r}\triangledown f(\hat{r}(\bm{x}_i^r;\bm{\alpha})).
    \end{aligned}
    \end{equation}
    With $f$ appropriately chosen,  $\widehat{BR}_f(\bm{\alpha})$ can be used as a loss function to fit $\hat{r}(\bm{x};\bm{\alpha})$. For example, the uLSIF loss \eqref{eq:uLSIF_loss} is a special case of Eq.\eqref{eq:BR_div_DR_emp} when $f(t) = 0.5(t-1)^2$.

    \section{Method}\label{sec:method}
    \subsection{Softplus Loss Function for Density Ratio Estimation}\label{sec:SP_loss}
    Motivated by the two shortcomings of uLSIF \eqref{eq:uLSIF_loss}, DSKL \eqref{eq:DSKL} and BARR \eqref{eq:BARR}, we propose a novel loss function called Softplus (SP) loss for density ratio estimation with neural networks. The SP loss is a special case of $BR_f(\bm{\alpha})$ in \eqref{eq:BR_div_DR} when $f(t)$ is the softplus function 
    \begin{equation}
    \label{eq:softplus_function}
    \eta(t)=\ln(1+e^t).
    \end{equation}
    The derivative of the softplus function $\eta(t)$ is the sigmoid function 
    \begin{equation}
    \label{eq:sigmoid_function}
    \sigma(t)=\frac{e^t}{1+e^t}.
    \end{equation}
    The second derivative of $\eta(t)$ is $\sigma(t)(1-\sigma(t))$ which is positive so the softplus function is strictly convex. Then, the SP loss and its empirical approximation are defined as
    \begin{equation}
    \label{eq:SP_loss}
    \begin{aligned}
    SP(\bm{\alpha}) &= \int \sigma(\hat{r}(\bm{x};\bm{\alpha}))\hat{r}(\bm{x};\bm{\alpha})p_g(\bm{x})d\bm{x}\\
    &-\int \eta(\hat{r}(\bm{x}))p_g(\bm{x})d\bm{x}-\int \sigma(\hat{r}(\bm{x};\bm{\alpha}))p_r(\bm{x})d\bm{x},
    \end{aligned}
    \end{equation}
    and   
    \begin{equation}
    \label{eq:SP_loss_emp}
    \begin{aligned}
    \widehat{SP}(\bm{\alpha})&=\frac{1}{n_g}\sum_{i=1}^{n_g}\sigma(\hat{r}(\bm{x}_i^g;\bm{\alpha}))\hat{r}(\bm{x}_i^g;\bm{\alpha}) \\
    &\quad- \frac{1}{n_g}\sum_{i=1}^{n_g}\eta(\hat{r}(\bm{x}_i^g;\bm{\alpha}))-\frac{1}{n_r}\sum_{i=1}^{n_r}\sigma(\hat{r}(\bm{x}_i^r;\bm{\alpha}))\\
    &=\frac{1}{n_g}\sum_{i=1}^{n_g}\left[\sigma(\hat{r}(\bm{x}_i^g;\bm{\alpha}))\hat{r}(\bm{x}_i^g;\bm{\alpha}) - \eta(\hat{r}(\bm{x}_i^g;\bm{\alpha})) \right]\\ &\quad-\frac{1}{n_r}\sum_{i=1}^{n_r}\sigma(\hat{r}(\bm{x}_i^r;\bm{\alpha})),
    \end{aligned}
    \end{equation}
    where $\hat{r}(\bm{x};\bm{\alpha})$ is the density ratio model in Eq. \eqref{eq:r_hat}.
    
    \begin{theorem}\label{thm:SP_bound}
    	The empirical SP loss \eqref{eq:SP_loss_emp} is bounded from below, i.e., $\widehat{SP}(\bm{\alpha}) > -\ln 2 -1$.
    \end{theorem}
    \begin{proof}
    	We define
    	\begin{equation*}
    	g(t) =\sigma(t)\cdot t-\eta(t)= e^t\cdot t \cdot (1+e^t)^{-1}-\ln(1+e^t),\quad t\geq 0.
    	\end{equation*}
    	Since $g^\prime(t) = e^t\cdot t \cdot (1+e^t)^{-2}>0$ when $t\geq 0$, $g(t)$ is monotonically increasing on its domain and $\min_tg(t)=g(0)=-\ln2$. Moreover, $\sigma(t)$ is lower bounded by -1. Therefore, the empirical SP loss has a lower bound, i.e., $\widehat{SP}(\bm{\alpha}) > -\ln 2 -1$.
    \end{proof}
    
    Then we propose to train the density ratio model $\hat{r}(\bm{x};\bm{\alpha})$ by minimizing the following penalized SP loss:
    \begin{equation}
    \label{eq:penalized_SP}
    \min_{\bm{\alpha}} \left\{\widehat{SP}(\bm{\alpha}) + \lambda \hat{Q}(\bm{\alpha})\right\},
    \end{equation}
    where $\hat{Q}(\bm{\alpha})$ is defined in Eq.\eqref{eq:penalty_hat} and $\lambda$ (a hyper-parameter) controls the penalty strength. Now the penalty term may take effect if a proper $\lambda$ is chosen.
    
    {\setlength{\parindent}{0cm}
    	\textbf{Hyperparameter Selection.}} To select the optimal hyperparameter $\lambda^*$, we evaluate a trained density ratio model $\hat{r}(\bm{x};\bm{\alpha})$ on $n_r$ real images which are used for training and $n_v$ hold-out real images $\{\bm{x}_1^v,\cdots,\bm{x}_{n_v}^v\}$ separately. Then we have two sets of density ratios: $\{\hat{r}(\bm{x}_1^r;\bm{\alpha}),\cdots,\hat{r}(\bm{x}_{n_r}^r;\bm{\alpha})\}$ and $\{\hat{r}(\bm{x}_1^v;\bm{\alpha}),\cdots,\hat{r}(\bm{x}_{n_v}^v;\bm{\alpha})\}$. If the model does not overfit the training images, these two sets should have similar distributions. We use the two-sample Kolmogorov-Smirnov (KS) test statistic \cite{Chakravarti1967KStest} to quantify the divergence between these two distributions. The optimal hyperparameter $\lambda^*$ is selected to minimize this KS test statistic. Other metrics (e.g., Kullback–Leibler (KL) divergence \cite{van2014renyi}) may also be useful for the hyperparameter selection.
    
    \subsection{Rate of Convergence}\label{sec:convergence}
    In this section, we derive the rate of convergence of a density ratio model trained with our proposed Softplus loss under the Bregman divergence in the GAN setting (i.e., $n_g$ is large enough). Let $\mathcal{H}=\{h\in\mathcal{H}: \bm{x}\mapsto h(\bm{x}) \}$ denote the set of potential functions that can be represented by the density ratio model $\hat{r}(\bm{x},\bm{\alpha})$ (i.e., the \textit{Hypothesis Space}). Also let $\sigma    \circ\mathcal{H}=\{h\in\mathcal{H}: \bm{x}\mapsto\sigma(h(\bm{x})) \}$, where $\sigma$ is the Sigmoid function. 
    
    \begin{lemma}
    	\label{lem:Rademacher_complexity_SP}
    	Let $\hat{\mathcal{R}}_{p_r,n_r}(\mathcal{H})$ and $\hat{\mathcal{R}}_{p_r,n_r}(\sigma\circ\mathcal{H})$ denote the empirical Rademacher complexities of $\mathcal{H}$ and $\sigma\circ\mathcal{H}$ respectively, where $\hat{\mathcal{R}}_{p_r,n_r}(\mathcal{H})$ and $\hat{\mathcal{R}}_{p_r,n_r}(\sigma\circ\mathcal{H})$ are defined based on independent samples $\{\bm{x}_1,\cdots,\bm{x}_{n_r}\}$ from $p_r(\bm{x})$. The following inequality holds:
    	\begin{equation*}
    	\hat{\mathcal{R}}_{p_r,n_r}(\sigma\circ\mathcal{H})\leq \frac{1}{4}\hat{\mathcal{R}}_{p_r,n_r}(\mathcal{H}),
    	\end{equation*}
    	where 
    	\begin{equation*}
    	\begin{aligned}
    	\hat{\mathcal{R}}_{p_r,n_r}(\mathcal{H})&=\mathbb{E}_\rho\left\{\sup_{h\in\mathcal{H}}\left|\frac{1}{n_r}\sum_{i=1}^{n_r}\rho_ih(\bm{x}_i^r) \right| \right\}, \\ \hat{\mathcal{R}}_{p_r,n_r}(\sigma\circ\mathcal{H})&=\mathbb{E}_\rho\left\{\sup_{h\in\mathcal{H}}\left|\frac{1}{n_r}\sum_{i=1}^{n_r}\rho_i\sigma(h(\bm{x}_i^r)) \right| \right\},
    	\end{aligned}
    	\end{equation*}
    	and $\rho_1,\cdots,\rho_{n_r}$ are independent Rademacher random variables whose distribution is $P(\rho_i=1)=P(\rho_i=-1)=0.5$.
    \end{lemma}
    \begin{proof}
    	Since the Sigmoid function $\sigma$ is $\frac{1}{4}$-Lipschitz continuous, the inequality can be obtained by applying Talagrand's Lemma (Lemma 4.2 in \cite{foundation_ml}).
    \end{proof}
    
    \begin{theorem}[Rademacher Bound]
    	\label{thm:Rademacher_bound}
    	If a hypothesis space $\mathcal{H}$ is a class of functions $h$ such that $0\leq h(\bm{x})\leq1$, then for $\forall\delta\in(0,1)$ with probability at least $1-\delta$,
    	\begin{equation}
    	\label{eq:Rademacher_bound}
    	\sup_{h\in\mathcal{H}}\left|\mathbb{E}_{\bm{x}\sim p}h(\bm{x})-\frac{1}{n}\sum_{i=1}^nh(\bm{x}_i)\right|\leq 2\hat{\mathcal{R}}_{p,n}(\mathcal{H})+\sqrt{\frac{4}{n}\log\left(\frac{2}{\delta}\right)},
    	\end{equation}
    	where the $\bm{x}_i$'s are independently drawn from a distribution $p$ and $\hat{\mathcal{R}}_{p,n}(\mathcal{H})$ is the empirical Rademacher complexity of the hypothesis space $\mathcal{H}$ defined on these $n$ samples. 
    \end{theorem}
    \begin{proof}
    	The proof of Theorem \ref{thm:Rademacher_bound} can be found in \cite{lafferty2010}.
    \end{proof}
    
    Let $BR_f(h)$ be the Bregman divergence between the true density ratio function $r$ in Eq.\eqref{eq:relation_DR_Disc} and a function $h$ in the hypothesis space $\mathcal{H}$. Let $\widehat{BR}_f(h)$ be the empirical approximation of $BR_f(h)$. If $f$ is replaced by the Softplus function $\eta$, then 
    \begin{equation}
    \begin{aligned}
    BR_f(h)&=\mathbb{E}_{\bm{x}\sim p_g(\bm{x})}\left[\sigma(h(\bm{x}))h(\bm{x})-\eta(h(\bm{x})) \right]\\
    &\quad-\mathbb{E}_{\bm{x}\sim p_r(\bm{x})}\left[\sigma(h(\bm{x})) \right],
    \end{aligned}
    \end{equation}
    \begin{equation}
    \begin{aligned}
    \widehat{BR}_f(h)&=\frac{1}{n_g}\sum_{i=1}^{n_g}\left[\sigma(h(\bm{x}^g_i))h(\bm{x}^g_i)-\eta(h(\bm{x}^g_i)) \right]\\
    &\quad-\frac{1}{n_r}\sum_{i=1}^{n_r}\sigma(h(\bm{x}^r_i)).
    \end{aligned}
    \end{equation} 
    Note that the $BR_f(h)$ and $\widehat{BR}_f(h)$ are equivalent to Eq.\eqref{eq:SP_loss} and Eq.\eqref{eq:SP_loss_emp} respectively. For simplicity, we only consider the Softplus loss without a penalty term. \changes{Following the notation in Appendix D of \cite{uehara2016generative},} we define $r_0$ and $r_s$ as
    \begin{equation*}
    r_0=\arg\min_{h\in\mathcal{H}}BR_f(h), \quad
    r_s=\arg\min_{h\in\mathcal{H}}\widehat{BR}_f(h).
    \end{equation*}
    Note that $BR_f(h)$ reaches its minimum if and only if $h=r$ but $r$ may be not in $\mathcal{H}$. If $r\notin\mathcal{H}$, then $BR_f(r_0)-BR_f(r)$ is a positive constant; otherwise $BR_f(r_0)-BR_f(r)=0$. However, in practice, we can only optimize $\widehat{BR}_f(h)$ instead of $BR_f(h)$. Therefore, we are interested in the distance of $r_s$ from $r$ under the Bregman divergence, i.e., $BR_f(r_s)-BR_f(r)$.
    
    Before we introduce our main theorem for the rate of convergence, we need some more notation. Denote by $\mathcal{\bm{A}}$ the parameter space of the density ratio model $\hat{r}(\bm{x}_i^r;\bm{\alpha})$. Note that the hypothesis space $\mathcal{H}$ is determined by the parameter space $\mathcal{\bm{A}}$. Denote $\sigma(h(\bm{x}))h(\bm{x})-\eta(h(\bm{x}))$ by $g(\bm{x};\bm{\alpha})$.
    \begin{theorem}
    	\label{thm:rate_of_convergence}
    	If (i) the $f$ in the Bregman divergence is the Softplus function $\eta$ in \eqref{eq:softplus_function}, (ii) $n_g$ is large enough, (iii), $\mathcal{\bm{A}}$ is compact, (iv) $\forall g(\bm{x};\bm{\alpha})$ is continuous at $\bm{\alpha}$, (v) $\forall g(\bm{x};\bm{\alpha}), \exists $ a function $g^u(\bm{x})$ that does not depend on $\bm{\alpha}$, s.t. $|g(\bm{x};\bm{\alpha})|\leq g^u(\bm{x})$, and (vi) $\mathbbm{E}_{\bm{x}\sim p_g}g^u(\bm{x})<\infty$, then $\forall \delta\in(0,1)$ and $\forall \delta^\prime\in (0,\delta]$ with probability at least $1-\delta$,
    	\begin{equation}
    	\label{eq:rate_of_convergence}
    	\begin{aligned}
    	BR_f(r_s)-BR_f(r)&\leq \frac{1}{n_g}+\hat{\mathcal{R}}_{p_r,n_r}(\mathcal{H})+2\sqrt{\frac{4}{n_r}\log\left(\frac{2}{\delta^\prime}\right)}\\
    	&\quad+BR_f(r_0)-BR_f(r).
    	\end{aligned}
    	\end{equation}
    \end{theorem}    
    
    \begin{proof}
    	\changes{Following Eq.(11) of Appendix D in \cite{uehara2016generative},} we first decompose $BR_f(r_s)-BR_f(r)$ as follows
    	\begin{equation}
    	\label{eq:decompose_genearlization_gap}
    	\begin{aligned}
    	&BR_f(r_s)-BR_f(r) \\
    	= &BR_f(r_s)-\widehat{BR}_f(r_s)+\widehat{BR}_f(r_s)-BR_f(r_0)\\
    	&\quad+BR_f(r_0)-BR_f(r)\\
    	\leq& \left|BR_f(r_s)-\widehat{BR}_f(r_s)\right|+\left|\widehat{BR}_f(r_s)-BR_f(r_0)\right|\\
    	&\quad+BR_f(r_0)-BR_f(r)\\
    	\leq& 2\sup_{h\in\mathcal{H}}|BR_f(h)-\widehat{BR}_f(h)|+BR_f(r_0)-BR_f(r).\\
    	\end{aligned}
    	\end{equation}
    	The second term in Eq.\eqref{eq:decompose_genearlization_gap} is a constant so we just need to bound the first term, and if $f=\eta$, the first term can be further decomposed with an upper bound as follows
    	\begin{equation*}
    	\begin{aligned}
    	&\sup_{h\in\mathcal{H}}\left|BR_f(h)-\widehat{BR}_f(h)\right|\\
    	\leq&\sup_{h\in\mathcal{H}}\left|\mathbb{E}_{\bm{x}\sim p_g(\bm{x})}\left[\sigma(h(\bm{x}))h(\bm{x})-\eta(h(\bm{x})) \right] \vphantom{\sum_{i=1}^{n_g}}  \right. \\
    	&\quad\left.- \frac{1}{n_g}\sum_{i=1}^{n_g}\left[\sigma(h(\bm{x}^g_i))h(\bm{x}^g_i)-\eta(h(\bm{x}^g_i)) \right]\right|\\
    	&\quad+\sup_{h\in\mathcal{H}}\left| \mathbb{E}_{\bm{x}\sim p_r(\bm{x})}\left[\sigma(h(\bm{x})) \right]- \frac{1}{n_r}\sum_{i=1}^{n_r}\sigma(h(\bm{x}^r_i))\right|.
    	\end{aligned}
    	\end{equation*}
    	
    	Since $\mathcal{\bm{A}}$ is compact, $g(\bm{x};\bm{\alpha})$ is continuous at $\bm{\alpha}$, $|g(\bm{x};\bm{\alpha})|\leq g^u(\bm{x})$, and $\mathbbm{E}_{\bm{x}\sim p_g}g^u(\bm{x})<\infty$, based on the uniform law of large numbers \cite{noteULLN, jennrich1969asymptotic}, for $\forall \epsilon>0$,
    	\begin{equation*}
    	\begin{aligned}
    	\lim_{n_g\rightarrow\infty}&P\left\{\sup_{h\in\mathcal{H}}\left|\mathbb{E}_{\bm{x}\sim p_g(\bm{x})}\left[\sigma(h(\bm{x}))h(\bm{x})-\eta(h(\bm{x})) \right] \vphantom{\sum_{i=1}^{n_g}}\right.\right.\\
    	&\quad\left.\left.- \frac{1}{n_g}\sum_{i=1}^{n_g}\left[\sigma(h(\bm{x}^g_i))h(\bm{x}^g_i)-\eta(h(\bm{x}^g_i)) \right]\right|>\epsilon \right\}=0.
    	\end{aligned}
    	\end{equation*}
    	Since $n_g$ is large enough, let $\epsilon={1}/{2n_g}$, $\forall \delta_1\in(0,1)$ with probability at least $1-\delta_1$, whereupon
    	\begin{equation}
    	\label{eq:first_bound}
    	\begin{aligned}
    	\sup_{h\in\mathcal{H}}&\left|\mathbb{E}_{\bm{x}\sim p_g(\bm{x})}\left[\sigma(h(\bm{x}))h(\bm{x})-\eta(h(\bm{x})) \right]\vphantom{\sum_{i=1}^{n_g}}\right. \\
    	&\left.- \frac{1}{n_g}\sum_{i=1}^{n_g}\left[\sigma(h(\bm{x}^g_i))h(\bm{x}^g_i)-\eta(h(\bm{x}^g_i)) \right]\right|\leq \frac{1}{2n_g}.
    	\end{aligned}
    	\end{equation}
    	Moreover, based on Theorem \ref{thm:Rademacher_bound} and Lemma \ref{lem:Rademacher_complexity_SP}, $\forall \delta_2\in(0,1)$ with probability at least $1-\delta_2$,
    	\begin{equation}
    	\label{eq:second_bound}
    	\begin{aligned}
    	&\sup_{h\in\mathcal{H}}\left| \mathbb{E}_{\bm{x}\sim p_r(\bm{x})}\left[\sigma(h(\bm{x})) \right]- \frac{1}{n_r}\sum_{i=1}^{n_r}\sigma(h(\bm{x}^r_i))\right|\\
    	\leq& 2\hat{\mathcal{R}}_{p_r,n_r}(\sigma\circ\mathcal{H})+\sqrt{\frac{4}{n_r}\log\left(\frac{2}{\delta_2} \right)}\\
    	\leq&\frac{1}{2}\hat{\mathcal{R}}_{p_r,n_r}(\mathcal{H})+\sqrt{\frac{4}{n_r}\log\left(\frac{2}{\delta_2} \right)}.
    	\end{aligned}
    	\end{equation}
    	With $\delta=\max\{\delta_1,\delta_2 \}$ and $\delta^\prime=\delta_2$, combining Eq.\eqref{eq:first_bound} and Eq.\eqref{eq:second_bound} leads to the upper bound in Theorem \ref{thm:rate_of_convergence}.
    \end{proof}    
    
    \begin{remark}
    	If $f(t)=0.5(t-1)^2$, then $BR_f(h)$ is the uLSIF loss. In this case, \cite{uehara2016generative} gives an upper bound for $BR_f(r_s)-BR_f(r)$ and at least one term in this upper bound is proportional to a constant $M$ (\cite{uehara2016generative} assumes all elements in $\mathcal{H}$ are bounded by $M$). However, in real practice, $M$ may be quite large so the upper bound provided by \cite{uehara2016generative} may be too loose, which helps explain why the SP loss outperforms the uLSIF loss in our experiments.
    \end{remark}
    
    \begin{remark}
    	The $\hat{\mathcal{R}}_{p_r,n_r}(\mathcal{H})$ term on the right hand side of the inequality \eqref{eq:rate_of_convergence} implies we should not use a density ratio model that is too complex. Therefore, we propose to estimate the density ratio by a simple multilayer perceptron in the feature space learned by a pre-trained deep CNN in Section \ref{sec:DRE-F}.
    \end{remark}
    
    \subsection{Density Ratio Estimation in Feature Space}\label{sec:DRE-F}
    
    In this section, we propose a novel density ratio estimation method called \textit{density ratio estimation in feature space under Softplus loss} (DRE-F-SP). 
    Assume we have $n_r$ real images 
    $\bm{x}_1^r, \bm{x}_2^r, \cdots, \bm{x}_{n_r}^r \sim p_r(\bm{x}),$
    and $n_g$ fake images $\bm{x}_1^g, \bm{x}_2^g, \cdots, \bm{x}_{n_g}^g \sim p_g(\bm{x}).$ The distributions $p_r(\bm{x})$ and $p_g(\bm{x})$ are both unknown. Rather than estimating density ratios in the pixel space \cite{nam2015direct, khan2019deep} (the density ratio model $\hat{r}(\bm{x};\bm{\alpha})$ directly maps an image to its density ratio) or using the property of a well-trained GAN model \cite{azadi2018discriminator, turner2018metropolis}, we model the true density ratio function via a small-scale \textit{multilayer perceptron} (MLP) in a feature space learned by a pre-trained deep CNN. This deep CNN takes an image as input and outputs a class label. The architecture of this CNN is specially designed to let one of its hidden layers output a feature map $\bm{y}$ that has the same dimension as the input $\bm{x}$. In our experiment, we build such a CNN by adding an extra fully connected layer which can output such feature map $\bm{y}$ on top of all convolutional layers of the ResNet-34 \cite{he2016deep}. We train this specially designed ResNet-34 on a set of labelled samples with the cross-entropy loss. Denote the fully connected layer which is used to output the feature map $\bm{y}$ and other layers before it in this pre-trained ResNet-34 as $\phi(\bm{x})$, then $\phi$ defines a mapping of a raw image $\bm{x}$ in the pixel space $\mathcal{X}$ to a high-level feature $\bm{y}$ in the feature space $\mathcal{Y}$, i.e,
    \begin{equation}
    \label{eq:pre_trained_CNN}
    \bm{y}=\phi(\bm{x}).
    \end{equation}
    In the remainder of this paper, we simply call $\phi(\bm{x})$ ResNet-34. We define $q_r$ and $q_g$ as the distribution of real and fake features respectively. Since $\mathcal{X}$ and $\mathcal{Y}$ have the same dimension, the Jacobian matrix ${\partial \bm{y}}/{\partial \bm{x}}$ is a square matrix and the relationship between the distributions of $\bm{x}$ and $\bm{y}$ can be summarized as follows:
    \begin{equation*}
    p_r(\bm{x}) = q_r(\bm{y})\cdot\left\|\frac{\partial \bm{y}}{\partial \bm{x}}\right\|,\quad p_g(\bm{x}) = q_g(\bm{y})\cdot\left\|\frac{\partial \bm{y}}{\partial \bm{x}}\right\|,
    \end{equation*}
    where $\left\|\frac{\partial \bm{y}}{\partial \bm{x}}\right\|$ is the absolute value of the Jacobian determinant \changes{and assumed to be positive}. Then the true density ratio function $r(\bm{x})$ can be equivalently expressed in the features space via
    \begin{equation}
    \label{eq:DR_F}
    \psi(\phi(\bm{x}))=\psi(\bm{y})= \frac{q_r(\bm{y})}{q_g(\bm{y})}=\frac{q_r(\bm{y})\cdot \left\|\frac{\partial \bm{y}}{\partial \bm{x}}\right\|}{q_g(\bm{y})\cdot \left\|\frac{\partial \bm{y}}{\partial \bm{x}}\right\|}=\frac{p_r(\bm{x})}{p_g(\bm{x})}=r(\bm{x}),
    \end{equation}
    where $\psi(\bm{y})$ denotes the true density ratio function in the feature space. Note that the Jacobian determinant is cancelled so we only need to model the density ratio function $\psi(\bm{y})$ in the feature space. We propose to model $\psi(\bm{y})$ by a 5-layer multilayer perceptron $\hat{\psi}(\bm{y};\bm{\beta})$ with a learnable parameter $\bm{\beta}$, and $\hat{\psi}(\bm{y};\bm{\beta})$ is trained by minimizing the following penalized SP loss:
    \begin{equation}
    \label{eq:DRE_F_SP}
    \begin{aligned}
    &\min_{\bm{\beta}} \left\{\widehat{SP}(\bm{\beta}) + \lambda \hat{Q}(\bm{\beta})\right\}\\
    =&\min_{\bm{\beta}}\left\{\frac{1}{n_g}\sum_{i=1}^{n_g}\left[\sigma(\hat{\psi}(\bm{y}_i^g;\bm{\beta}))\hat{\psi}(\bm{y}_i^g;\bm{\beta}) - \eta(\hat{\psi}(\bm{y}_i^g;\bm{\beta})) \right]\right.\\ 
    &-\frac{1}{n_r}\sum_{i=1}^{n_r}\sigma(\hat{\psi}(\bm{y}_i^r;\bm{\beta})) \left. + \lambda \left(\frac{1}{n_g}\sum_{i=1}^{n_g}\hat{\psi}(\bm{y}_i^g;\bm{\beta})-1 \right)^2
    \right\}.
    \end{aligned}
    \end{equation}
    Eq.\eqref{eq:DRE_F_SP} is adapted from Eq.\eqref{eq:SP_loss_emp} by replacing $\hat{r}(\bm{x};\bm{\alpha})$ with $\hat{\psi}(\bm{y};\bm{\beta})$. Then $\hat{\psi}(\bm{y};\bm{\beta})$ can be seen as a density ratio model in the feature space, and  $\hat{\psi}(\phi(\bm{x});\bm{\beta})$ can be seen as a density ratio model in the pixel space. Their workflows are visualized in Fig. \ref{fig:digram_DRE_DR}. We implement DRE-F-SP by Alg. \ref{alg:DRE-F-SP}.
    
    \changes{
    	\begin{remark}
    		\label{remark:CNN_in_feature_space}
    		The density ratio model in the feature space is not necessarily a MLP. It can be another small-scale neural network such as a CNN as long as its complexity is moderate.
    	\end{remark}
    }
    
    \begin{figure*}[h]
    	\centering
    	\includegraphics[width=0.80\textwidth, height=6.2cm]{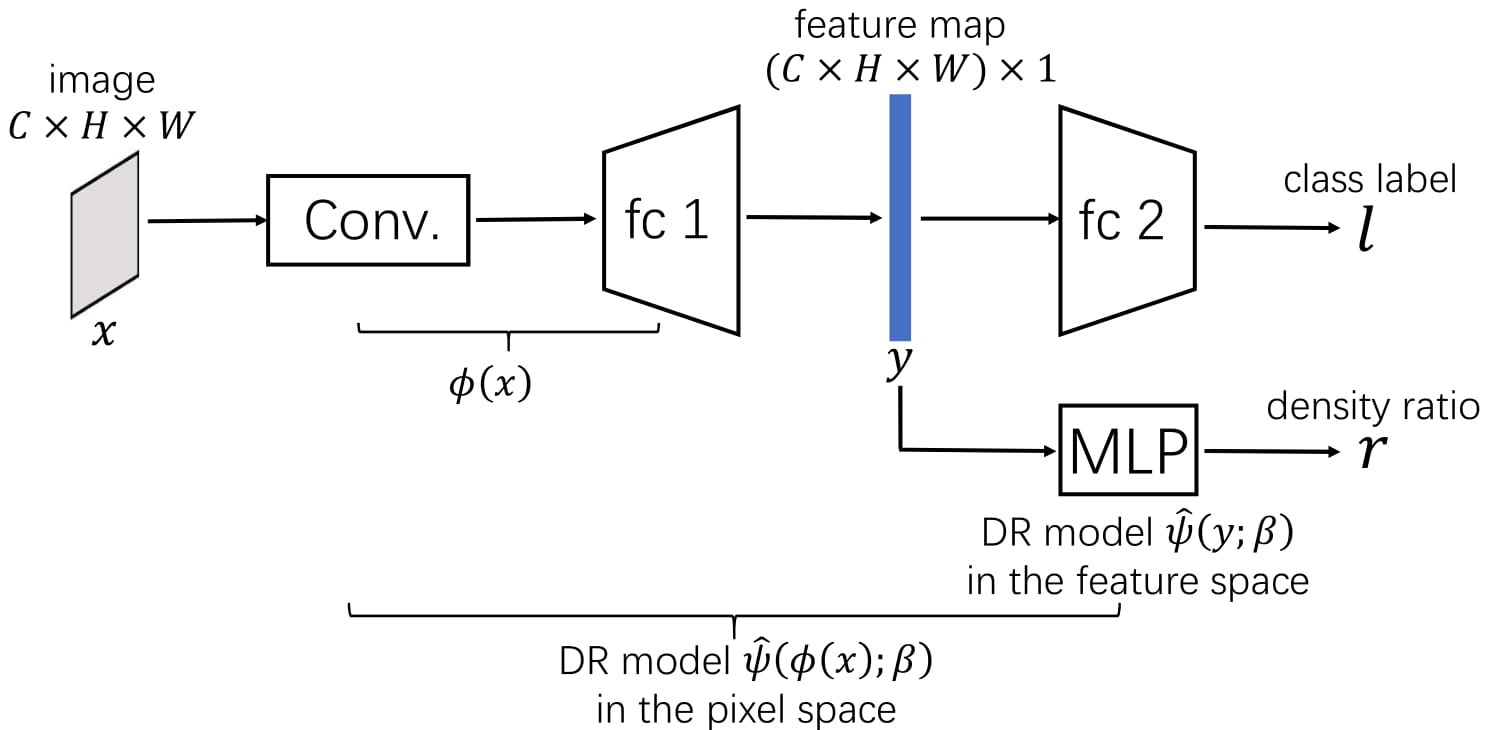}
    	\caption{The workflow of our proposed density ratio model $\hat{\psi}(\bm{y};\beta)$ in the feature space in Eq.\eqref{eq:DRE_F_SP}. The composition of $\hat{\psi}$ and $\phi$ leads to the density ratio model $\hat{\psi}(\phi(\bm{x});\beta)$ in the pixel space. The pre-trained ResNet-34 $\phi(\bm{x})$ \eqref{eq:pre_trained_CNN} takes an image $\bm{x}$ as input and outputs a feature map $\bm{y}$ from the fully-connected layer fc 1. If we flatten the image $\bm{x}$, then $\text{flat}(\bm{x})$ and $\bm{y}$ are two $CHW\times 1$ vectors, where $C$, $H$ and $W$ denote the number of channels, height and width of the image $\bm{x}$. }
    	\label{fig:digram_DRE_DR}
    \end{figure*}

    \begin{algorithm}[h]
    	\footnotesize
    	\SetAlgoLined
    	\KwData{$n_r$ real samples $\{\bm{x}^r_1,\cdots,\bm{x}^r_{n_r}\}$, a generator $G$, a pre-trained CNN $\phi(\bm{x})$ \eqref{eq:pre_trained_CNN}, a untrained MLP $\hat{\psi}(\bm{y};\bm{\beta})$ in \eqref{eq:DRE_F_SP} and a preset hyperparameter $\lambda$}
    	\KwResult{a trained density ratio model $\hat{r}(\bm{x})=\hat{\psi}(\phi(\bm{x});\bm{\beta})=\hat{\psi}(\bm{y};\bm{\beta})$ }
    	Initialize $\bm{\beta}$\;
    	\For{$k=1$ \KwTo $K$}{
    		Sample a mini-batch of $m$ real samples $\bm{x}^r_{(1)},\cdots,\bm{x}^r_{(m)}$ from $\{\bm{x}^r_1,\cdots,\bm{x}^r_{n_r}\}$\;
    		Sample a mini-batch of $m$ fake samples $\bm{x}^g_{(1)},\cdots,\bm{x}^g_{(m)}$ from $G$\;
    		Update $\bm{\beta}$ via the SGD or its variants with the gradient of Eq.\eqref{eq:DRE_F_SP}
    		\begin{equation*}
    		\begin{aligned}
    		&\frac{\partial}{\partial\bm{\beta}} \left\{\frac{1}{m}\sum_{i=1}^{m}\left[\sigma(\hat{\psi}(\phi(\bm{x}_{(i)}^g);\bm{\beta}))\hat{\psi}(\phi(\bm{x}_{(i)}^g);\bm{\beta}) - \eta(\hat{\psi}(\phi(\bm{x}_{(i)}^g);\bm{\beta})) \right] \right.\\ 
    		&\quad\left. -\frac{1}{m}\sum_{i=1}^{m}\sigma(\hat{\psi}(\phi(\bm{x}_{(i)}^r);\bm{\beta})) + \lambda \left(\frac{1}{m}\sum_{i=1}^{m}\hat{\psi}(\phi(\bm{x}_{(i)}^g);\bm{\beta})-1 \right)^2\right\}
    		\end{aligned}
    		\end{equation*} 
    	}
    	\caption{Density Ratio Estimation in the Feature Space Under Penalized SP Loss (DRE-F-SP)}
    	\label{alg:DRE-F-SP}
    \end{algorithm}

    \subsection{Application of DRE-F-SP in Subsampling GANs}\label{sec:DRE-F-SP_Sampling}
    Fig. \ref{fig:digram_DRE_sampler} describes the workflow of a density ratio based subsampling method for GANs. Each density ratio based subsampling method consists of two components: a DRE method and a sampler. DRE methods can be as proposed in \cite{azadi2018discriminator, turner2018metropolis, nam2015direct, khan2019deep, grover2019bias} or our DRE-F-SP. A sampler here is a density ratio based sampling scheme such as the rejection sampling scheme (RS sampler) in DRS, the Metropolis-Hastings algorithm (MH sampler) in MH-GAN and the sampling-importance resampling scheme (SIR sampler) in Section \ref{sec:SIR}. Moreover, a neural network based DRE method can also be decomposed into two components: a density ratio model and a loss function. For example, our DRE-F-SP uses the composition of a pre-trained ResNet-34 and a 5-layer MLP as the density ratio model and trains the density ratio model with the SP loss.
    
    We propose three density ratio based subsampling methods for GANs, which are called DRE-F-SP+RS, DRE-F-SP+MH, and DRE-F-SP+SIR, respectively. These three methods utilize the same DRE method (i.e., DRE-F-SP) but three different samplers (i.e., RS, MH, and SIR). We provide three corresponding algorithms Alg. \ref{alg:DRE-F-SP+RS}--\ref{alg:DRE-F-SP+SIR} to implement them. In some scenarios, the RS sampler and MH sampler suffer from low acceptance rates, and consequently they may take a very long time. The SIR sampler does not suffer from this problem, so it is more efficient than the RS and MH samplers, but the SIR sampler may perform poorly if we subsample from a small pool of fake images.
    
    \begin{figure*}[h]
    	\centering
    	\includegraphics[width=0.90\textwidth]{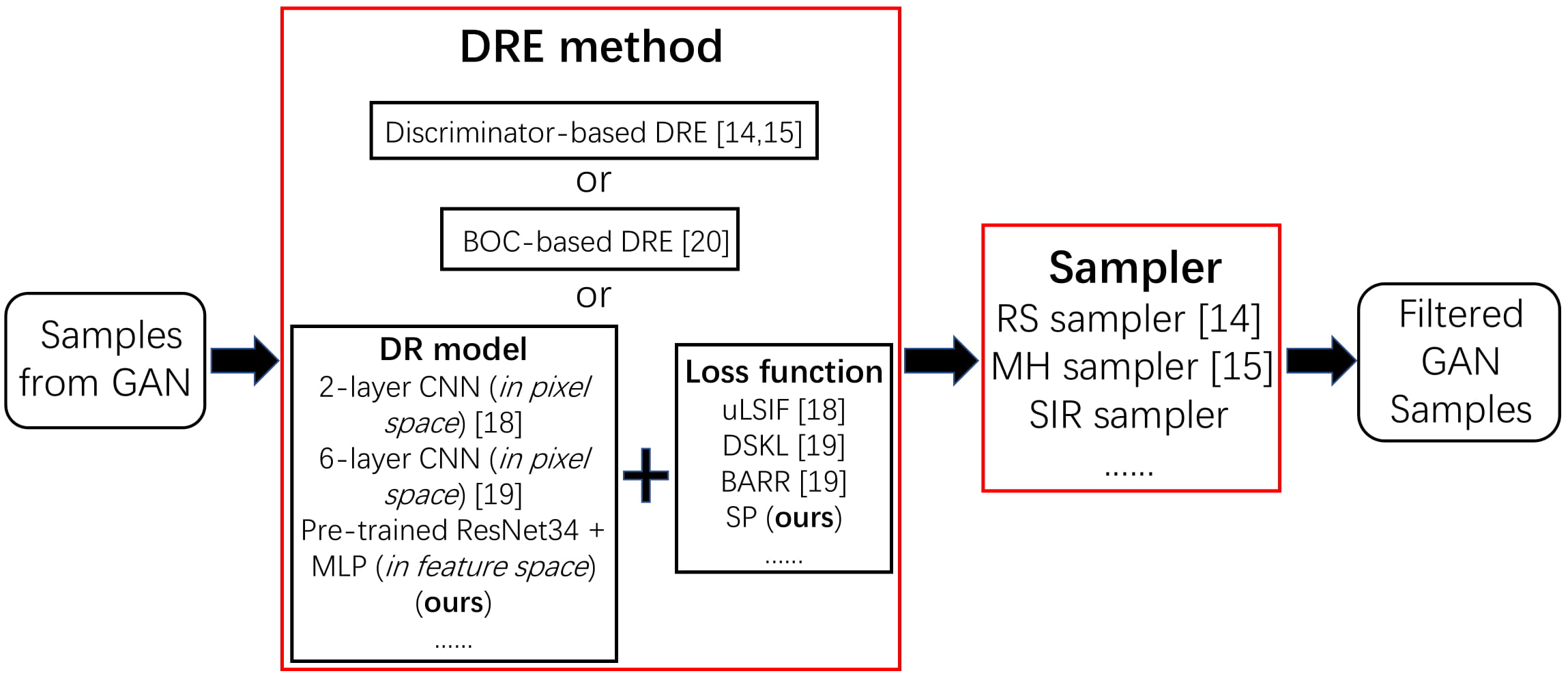}
    	\caption{Workflow of density ratio based subsampling for GANs.}
    	\label{fig:digram_DRE_sampler}
    \end{figure*}
    
    \begin{algorithm}[h]
    	\footnotesize
    	\SetAlgoLined
    	\KwData{a generator $G$, a trained CNN $\phi(\bm{x})$ \eqref{eq:pre_trained_CNN}, a trained MLP $\hat{\psi}(\bm{y};\bm{\beta})$ in \eqref{eq:DRE_F_SP}}
    	\KwResult{$images=\{N\text{ filtered images from }G\}$}
    	Generate $N^\prime$ fake images from $G$\;
    	Estimate the density ratios of these $N^\prime$ fake images by evaluating $\hat{\psi}(\phi(\bm{x});\bm{\beta})$\;
    	$M\leftarrow\max\{N^\prime \text{ estimated density ratios}\}$\;
    	$images\leftarrow \emptyset$\;
    	\While{$|images|<N$}{
    		$\bm{x}\leftarrow \text{get a fake image from }G$\;
    		$ratio\leftarrow \hat{\psi}(\phi(\bm{x});\bm{\beta})$\;
    		$M\leftarrow\text{max}\{M,ratio\}$\;
    		$p\leftarrow ratio/M$ (based on Eq.\eqref{eq:RS_accept_prob})\;
    		$u\leftarrow \text{Uniform}(0,1)$\;
    		\If{$u\leq p$}{
    			$\text{Append}(\bm{x}, images)$\;
    		}
    	}
    	\caption{DRE-F-SP+RS}
    	\label{alg:DRE-F-SP+RS}
    \end{algorithm}
    
    \begin{algorithm}[h]
    	\footnotesize
    	\SetAlgoLined
    	\KwData{a generator $G$, a trained CNN $\phi(\bm{x})$ \eqref{eq:pre_trained_CNN}, a trained MLP $\hat{\psi}(\bm{y};\bm{\beta})$ in \eqref{eq:DRE_F_SP}, real images}
    	\KwResult{$images=\{N\text{ filtered images from }G\}$}
    	$images\leftarrow \emptyset$\;
    	\While{$|images|<N$}{
    		$\bm{x}\leftarrow\text{a real image}$\;
    		\For{$i=1$ \KwTo $K$}{
    			$\bm{x}^\prime\leftarrow \text{get a fake image from }G$\;
    			$u\leftarrow \text{Uniform}(0,1)$\;
    			$p=\min\left(1,\frac{\hat{\psi}(\phi(\bm{x}^\prime);\bm{\beta})}{\hat{\psi}(\phi(\bm{x});\bm{\beta})}\right)$ (based on Eq.\eqref{eq:MH_accept_prob})\;
    			\If{$u\leq p$}{
    				$\bm{x}\leftarrow \bm{x}^\prime$\;
    			}
    		}
    		\If{$\bm{x}$ is not a real image}{
    			$\text{Append}(\bm{x}, images)$\;
    		}
    	}
    	\caption{DRE-F-SP+MH}
    	\label{alg:DRE-F-SP+MH}
    \end{algorithm}
    
    \begin{algorithm}[h]
    	\footnotesize
    	\SetAlgoLined
    	\KwData{a generator $G$, a trained CNN $\phi(\bm{x})$ \eqref{eq:pre_trained_CNN}, a trained MLP $\hat{\psi}(\bm{y};\bm{\beta})$ in \eqref{eq:DRE_F_SP}}
    	\KwResult{$images=\{N\text{ filtered images from }G\}$}
    	Generate a pool of $N_p$ samples $\{\bm{x}^g_1,\cdots,\bm{x}^g_{N_p}\}$ from $G$\;
    	Compute $N_p$ normalized importance weights $\{w_1,\cdots,w_{N_p}\}$ for these fake samples via
    	$$w_i=\frac{\hat{\psi}(\phi(\bm{x}_i^g);\bm{\beta})}{\sum_{i=1}^n\hat{\psi}(\phi(\bm{x}_i^g);\bm{\beta})};$$
    	
    	$images\leftarrow \emptyset$\;
    	\While{$|images|<N$}{
    		Sample an integer $j$ from $\{1,2,\cdots,N_p\}$ where $j$ is drawn with probabability $w_j$\;
    		$\text{Append}(\bm{x}^g_j, images)$\;
    	}
    	\caption{DRE-F-SP+SIR}
    	\label{alg:DRE-F-SP+SIR}
    \end{algorithm}

    \section{Experiment}\label{sec:experiment}
    In this section, our main objective is to justify that DRE-F-SP+RS, DRE-F-SP+MH, and DRE-F-SP+SIR perform better than DRS and MH-GAN for subsampling GANs. Hence, we conduct experiments on a synthetic dataset and a real dataset---CIFAR-10 \cite{krizhevsky2009learning}. We also conduct several ablation studies to empirically demonstrate that the power of three proposed subsampling methods comes from the novel Softplus loss and the scheme of estimating density ratio in the feature space. \changes{Besides the experiments reported here, the results of some extra experiments on the synthetic dataset, CIFAR-10 \cite{krizhevsky2009learning}, MNIST \cite{lecun1998gradient} and CelebA \cite{liu2015faceattributes} are shown in the supplemental material.}

    \subsection{Mixture of 25 2-D Gaussians}\label{sec:exp_simulation}
    We first test the performance of our proposed subsampling methods on synthetic data generated from a mixture of 25 two-dimensional Gaussians (the 25 mixture components have equal weights). This mixture model is used as a toy example in \cite{azadi2018discriminator, turner2018metropolis} and is very popular in the GAN literature.
    
    {\setlength{\parindent}{0cm}\textbf{Experimental setup of the main study:}} The means of these 25 Gaussians are arranged on a 2-D grid $\bm{\mu}\in\{-2,-1,0,1,2\}\times\{-2,-1,0,1,2\}$ and the common covariance matrix is set to $\sigma\bm{I}_{2\times 2}$, where $\sigma=0.05$. From this mixture model, we generate 50,000 training samples, 50,000 validation samples and 10,000 test samples. 
    
    Following \cite{azadi2018discriminator, turner2018metropolis}, we train a GAN model with the standard loss \eqref{eq:standard_gan_loss} on the training set. Both the generator and discriminator in this GAN consist of four fully connected layers with ReLU activation functions, and all hidden layers have size 100. The last layer of the discriminator is a Sigmoid function, and the noise $\bm{z}\in\mathbbm{R}^2$ fed into the generator is drawn from a 2-D Gaussian with mean 0 and standard deviation 1. We deliberately train the generator and discriminator for only 50 epochs to prevent them from reaching optimality, so density ratio estimation in terms of Eq.\eqref{eq:relation_DR_Disc} is not reliable.
    
    When implementing DRS, we follow the setting in \cite{azadi2018discriminator} and set $\gamma$ dynamically for each batch of fake samples drawn from the GAN to the $95$th percentile of $F(\bm{x})$ in Eq.\eqref{eq:DRS_F_hat} for each $\bm{x}$ in this batch. We also keep training the discriminator on the validation set for another 20 epochs to further improve performance of the DRS. When implementing MH-GAN \cite{turner2018metropolis}, we calibrate the trained discriminator on the validation set with logistic regression and set the MCMC iteration $K$ to 100 (more iterations do not show significant improvement).
    
    In our proposed sampling method, at the density ratio estimation stage, we use a 5-layer MLP as the density ratio model $\hat{r}(\bm{x};\bm{\alpha})$ in Eq.\eqref{eq:SP_loss_emp} to directly map a sample to its density ratio without a pre-trained CNN since our synthetic data are not images; its architecture is shown in Supp. \ref{appendix:Sim_nets}. The 5-layer MLP is trained with our proposed penalized SP loss \eqref{eq:penalized_SP}. To select the optimal $\lambda$, we generate a grid of values between 0 and 0.1 and select the one which minimizes the KS test statistic on the validation set (shown in Table \ref{tab:sim_parameter_selection} of the supplemental material). To show the superiority of our proposed SP loss, we also train the 5-layer MLP with the uLSIF \cite{nam2015direct}, DSKL, and BARR \cite{khan2019deep} losses. Following the setting in \cite{khan2019deep}, the $\lambda$ in BARR is set to 10. At the sampling stage, all three samplers---RS, MH, and SIR---are considered. The number of burn-in samples $N^\prime$ for RS in Alg.\ref{alg:DRE-F-SP+RS} is 50,000. The MCMC iterations $K$ for MH in Alg.\ref{alg:DRE-F-SP+MH} is set to 100. The pool size $N_p$ for SIR in Alg.\ref{alg:DRE-F-SP+SIR} is set to 20,000.
    
    We subsample 10,000 fake samples from the trained GAN with each method, and the quality of these fake samples is evaluated. We repeat the whole experiment (i.e., data generation, GAN training, MLP training, subsampling) three times and report in Table \ref{tab:results_simulation_main} the average quality of 10,000 fake samples from each subsampling method over the three repetitions. 
    
    {\setlength{\parindent}{0cm}\textbf{Experimental setup of an ablation study:}} To evaluate the effectiveness of the SP loss, we conduct an ablation study by training the 5-layer MLP in DRE-F-SP with other losses: uLSIF \cite{nam2015direct}, DSKL \cite{khan2019deep}, and BARR \cite{khan2019deep}. We subsample 10,000 fake samples under different losses and three samplers and evaluate the quality of these samples. Similar to the main study, we repeat the whole setting three times and report in Table \ref{tab:results_simulation_different_loss} the average quality of 10,000 fake samples under each loss and each sampler.
    
    {\setlength{\parindent}{0cm}\textbf{Evaluation metrics:}} To measure performance, following \cite{azadi2018discriminator, turner2018metropolis}, we assign each fake sample to its closest mixture component. A fake sample is defined as ``high-quality" if its Euclidean distance to the mean of its mixture component is smaller than $4\sigma=0.2$. Also, we define that a mode (i.e., a mixture component) is recovered if at least one ``high-quality" fake sample is assigned to it. For each sampling method in the main study and the ablation study, we report in Tables \ref{tab:results_simulation_main} and \ref{tab:results_simulation_different_loss} the average percentage of high-quality samples and the average percentage of recovered modes.
    
    {\setlength{\parindent}{0cm}\textbf{Quantitative results:}}
    From Table \ref{tab:results_simulation_main}, we can see that three proposed sampling methods almost perfectly correct the sampling bias of the imperfect generator and significantly outperform DRS and MH-GAN without trading off mode coverage for quality. Table \ref{tab:results_simulation_different_loss} shows that the power of three proposed sampling methods comes from the novel SP loss. 
    
    \begin{table}[htbp]
    	\centering
    	\scriptsize
    	\caption{Average quality of 10,000 fake synthetic samples from different subsampling methods over three repetitions. Higher \% high-quality samples and higher \% recovered modes are better. Each setting is repeated three times, and we report the averaged \% high-quality samples and averaged \% recovered modes. The optimal $\lambda^*$ in each round is shown in Table \ref{tab:sim_parameter_selection}.}
    	\begin{tabular}{cccc}
    		\toprule
    		& No Subsampling   & DRS \cite{azadi2018discriminator}  & MH-GAN \cite{turner2018metropolis} \\
    		\midrule
    		\% High Quality & $69.8\pm 15.4$ & $96.3\pm 1.3$ & $89.7\pm 5.3$ \\
    		\% Rec. Modes & $100.0\pm 0.0$ & $100.0\pm 0.0$ & $100.0\pm 0.0$ \\
    		\bottomrule
    		& DRE-F-SP+RS & DRE-F-SP+MH & DRE-F-SP+SIR \\
    		\midrule
    		\% High Quality     & $\bm{99.1\pm 0.5}$ & $\bm{99.2\pm 0.5}$ & $\bm{99.2\pm 0.4}$ \\
    		\% Rec. Modes  & $100.0\pm 0.0$ & $100.0\pm 0.0$ & $100.0\pm 0.0$ \\
    		\bottomrule
    	\end{tabular}%
    	\label{tab:results_simulation_main}%
    \end{table}%
    
    \begin{table}[htbp]
    	\centering
    	\scriptsize
    	\caption{Ablation study on synthetic data. We train the 5-layer MLP with different loss functions. Each setting is repeated three times, and we report the averaged \% high-quality samples and averaged \% recovered modes.}
    	\begin{tabular}{ccccc}
    		\toprule
    		& \multicolumn{4}{c}{RS} \\
    		\cline{2-5}
    		& uLSIF \cite{nam2015direct} & DSKL \cite{khan2019deep}  & BARR \cite{khan2019deep}  & SP \\
    		\midrule
    		\% High Quality & $89.7\pm 2.4$ & $66.4\pm 7.6$ & $71.9\pm 14.9$ & $\bm{99.1\pm 0.5}$ \\
    		\% Rec. Modes & $100.0\pm 0.0$ & $42.7\pm 5.0$ & $100.0\pm 0.0$ & $100.0\pm 0.0$ \\
    		\midrule
    		& \multicolumn{4}{c}{MH} \\
    		\cline{2-5}
    		& uLSIF \cite{nam2015direct} & DSKL \cite{khan2019deep}  & BARR \cite{khan2019deep}  & SP \\
    		\midrule
    		\% High Quality & $89.6\pm 2.6$ & $66.3\pm 7.3$ & $72.1\pm 15.0$ & $\bm{99.2\pm 0.5}$ \\
    		\% Rec. Modes & $98.7\pm 1.9$ & $38.7\pm 5.0$ & $100.0\pm 0.0$ & $100.0\pm 0.0$ \\
    		\midrule
    		& \multicolumn{4}{c}{SIR} \\
    		\cline{2-5}
    		& uLSIF \cite{nam2015direct} & DSKL \cite{khan2019deep}  & BARR \cite{khan2019deep}  & SP \\
    		\midrule
    		\% High Quality & $89.5\pm 2.3$ & $66.2\pm 7.8$ & $72.1\pm 15.1$ & $\bm{99.2\pm 0.4}$ \\
    		\% Rec. Modes & $98.7\pm 1.9$ & $40.0\pm 5.7$ & $100.0\pm 0.0$ & $100.0\pm 0.0$ \\
    		\bottomrule
    	\end{tabular}%
    	\label{tab:results_simulation_different_loss}%
    \end{table}%
    
    {\setlength{\parindent}{0cm}\textbf{Visual results:}} We visualize the first-round results of the main study in Fig. \ref{fig:sim_visual_results}. In Fig. \ref{fig:sim_visual_gan}, we can see that many samples directly drawn from the generator locate between two neighboring modes. Fig. \ref{fig:sim_visual_DRS} and \ref{fig:sim_visual_MHGAN} show that DRS and MH-GAN can remove some ``bad-quality" points, but many between-modes points still exist. Fig. \ref{fig:sim_visual_SP_RS} to \ref{fig:sim_visual_SP_SIR} show that fake samples from our proposed methods are close to their assigned mixture components where between-modes samples only account for a small portion. 
    
    \begin{figure*}[h]
    	\begin{minipage}[b]{.25\textwidth}
    		\raisebox{+0.5\height}{\subfloat[][Real Samples]{\includegraphics[width=0.9\textwidth, height=4cm]{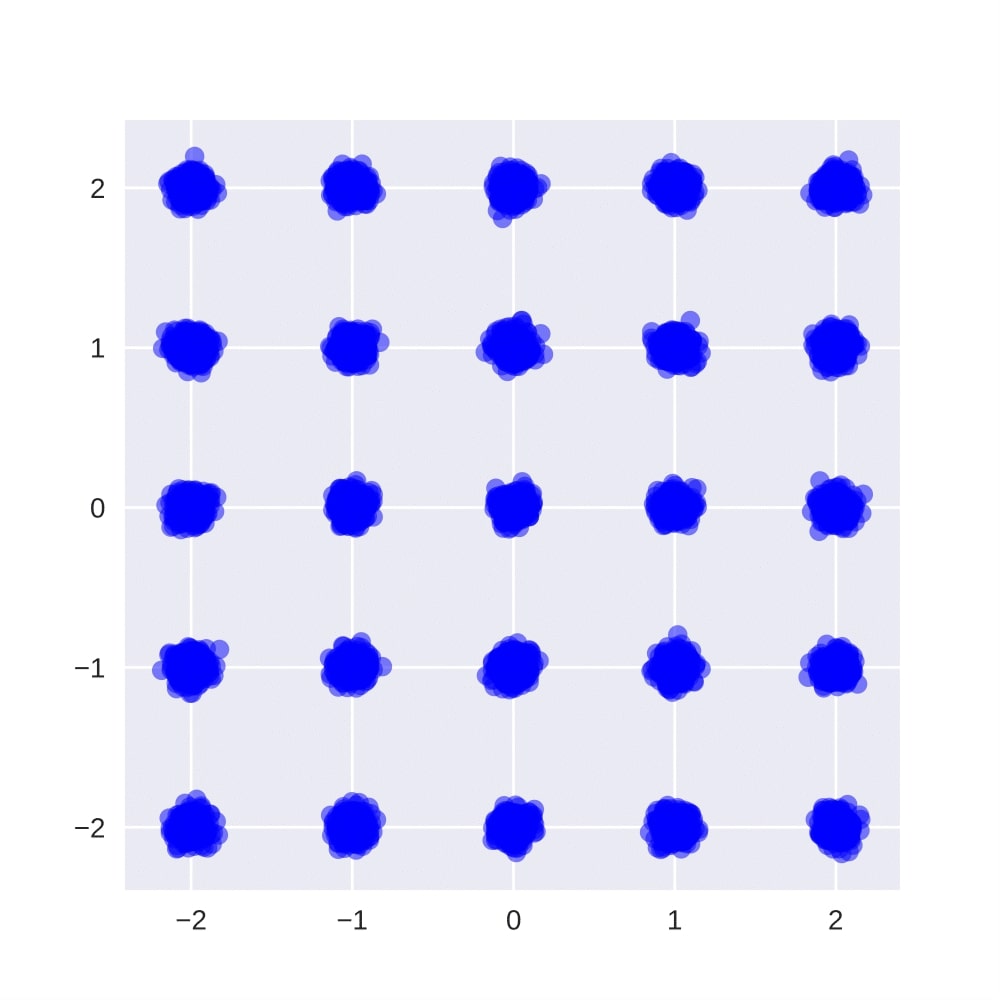}\label{fig:sim_visual_data}}}
    	\end{minipage}
    	\begin{minipage}[b]{.75\linewidth}
    		\subfloat[][No subsampling]{\includegraphics[width=0.3\textwidth, height=4cm]{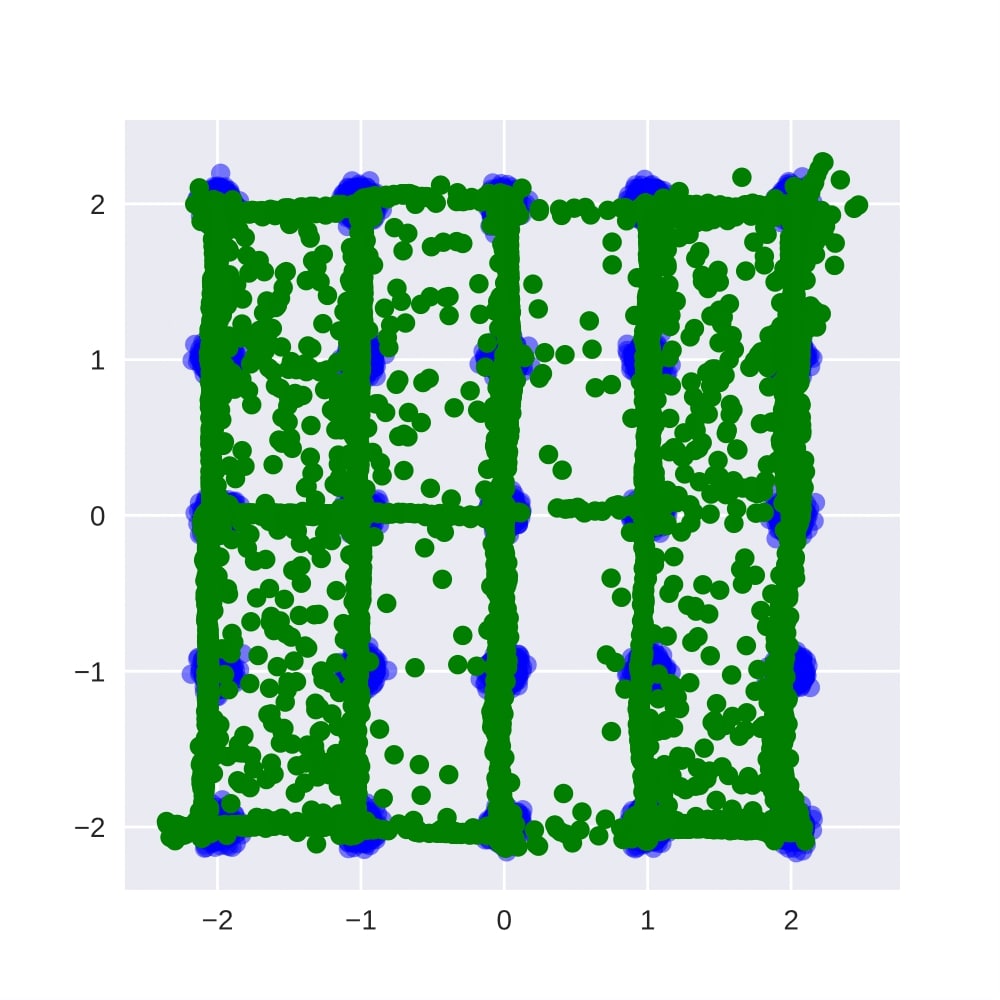}\label{fig:sim_visual_gan}}
    		\subfloat[][DRS \cite{azadi2018discriminator}]{\includegraphics[width=0.3\textwidth, height=4cm]{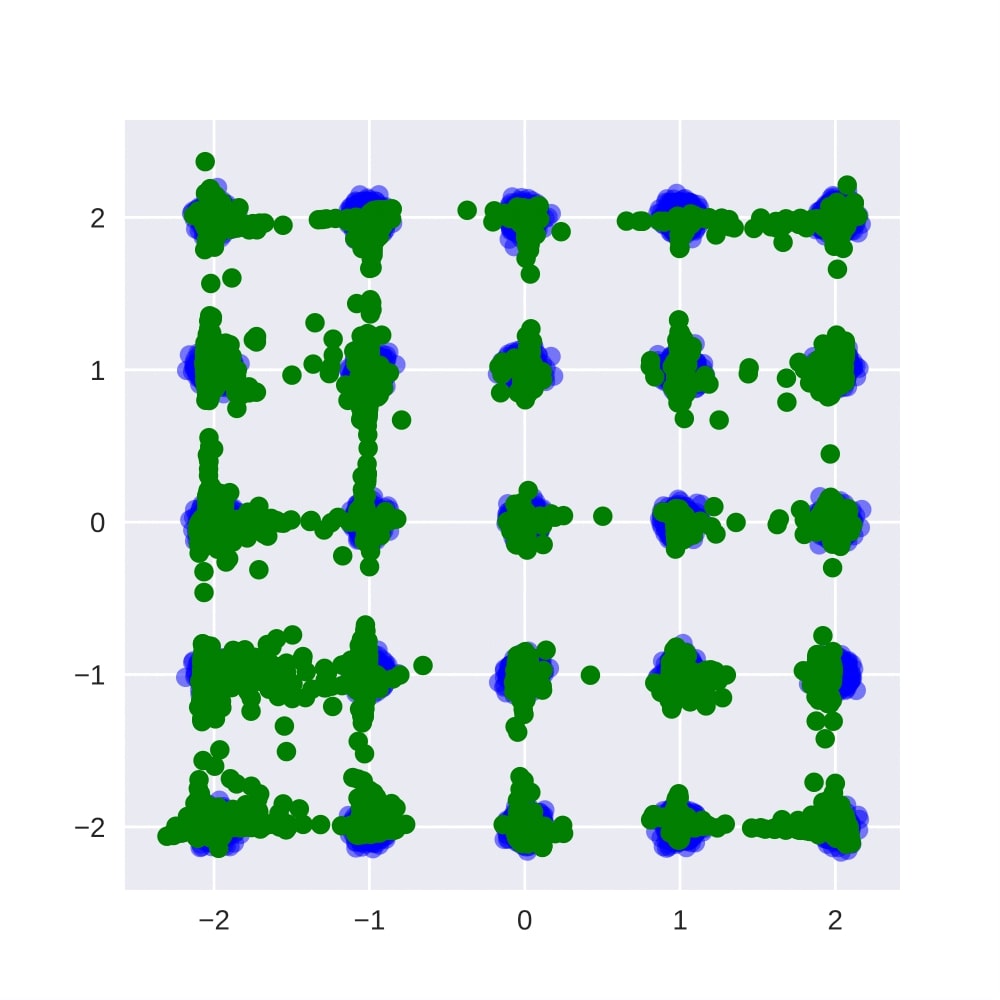}\label{fig:sim_visual_DRS}}
    		\subfloat[][MH-GAN \cite{turner2018metropolis}]{\includegraphics[width=0.3\textwidth, height=4cm]{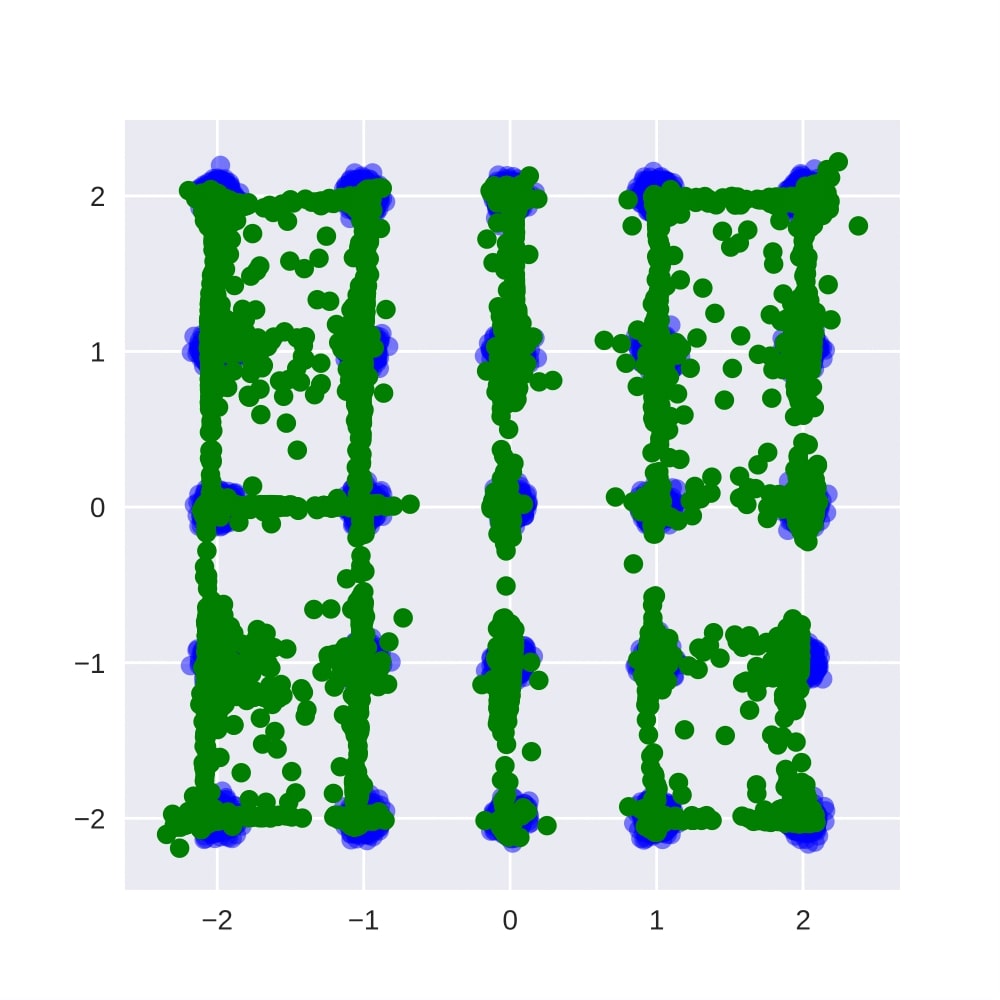}\label{fig:sim_visual_MHGAN}}\\
    		\subfloat[][DRE-F-SP+RS]{\includegraphics[width=0.3\textwidth, height=4cm]{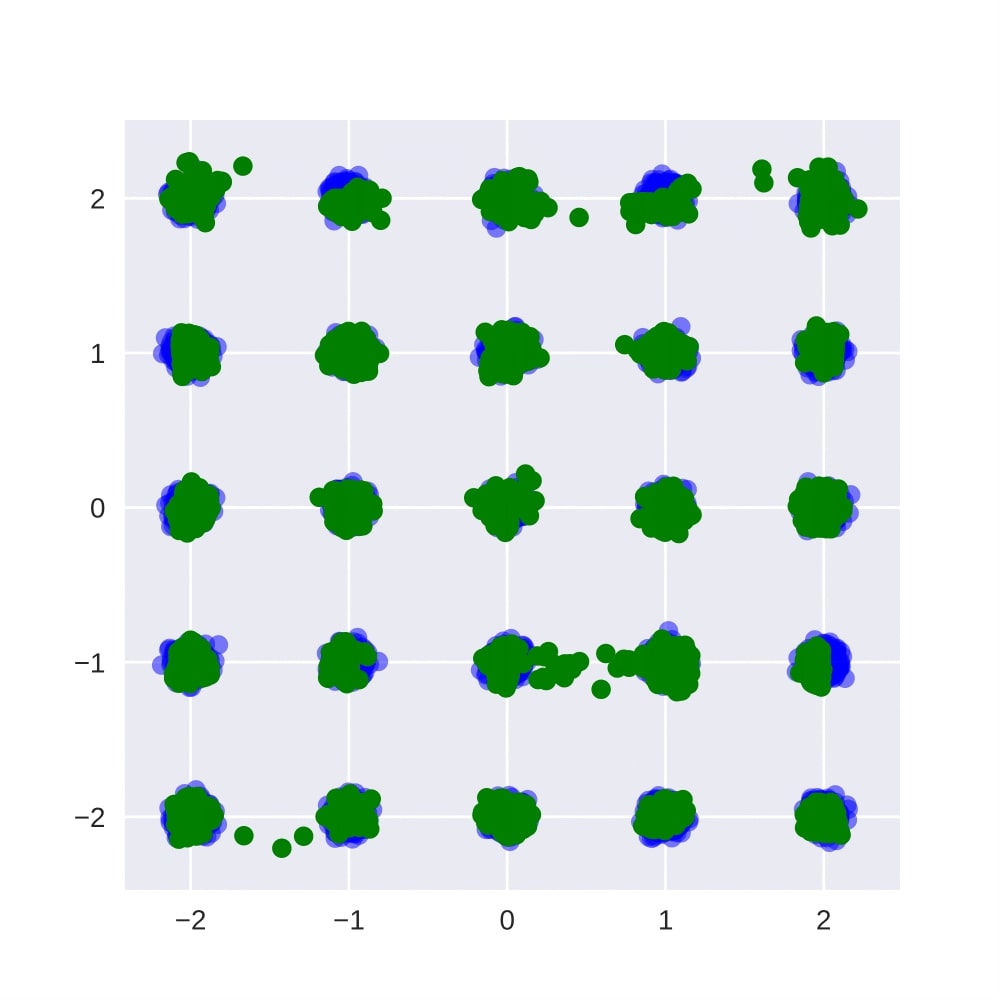}\label{fig:sim_visual_SP_RS}}
    		\subfloat[][DRE-F-SP+MH]{\includegraphics[width=0.3\textwidth, height=4cm]{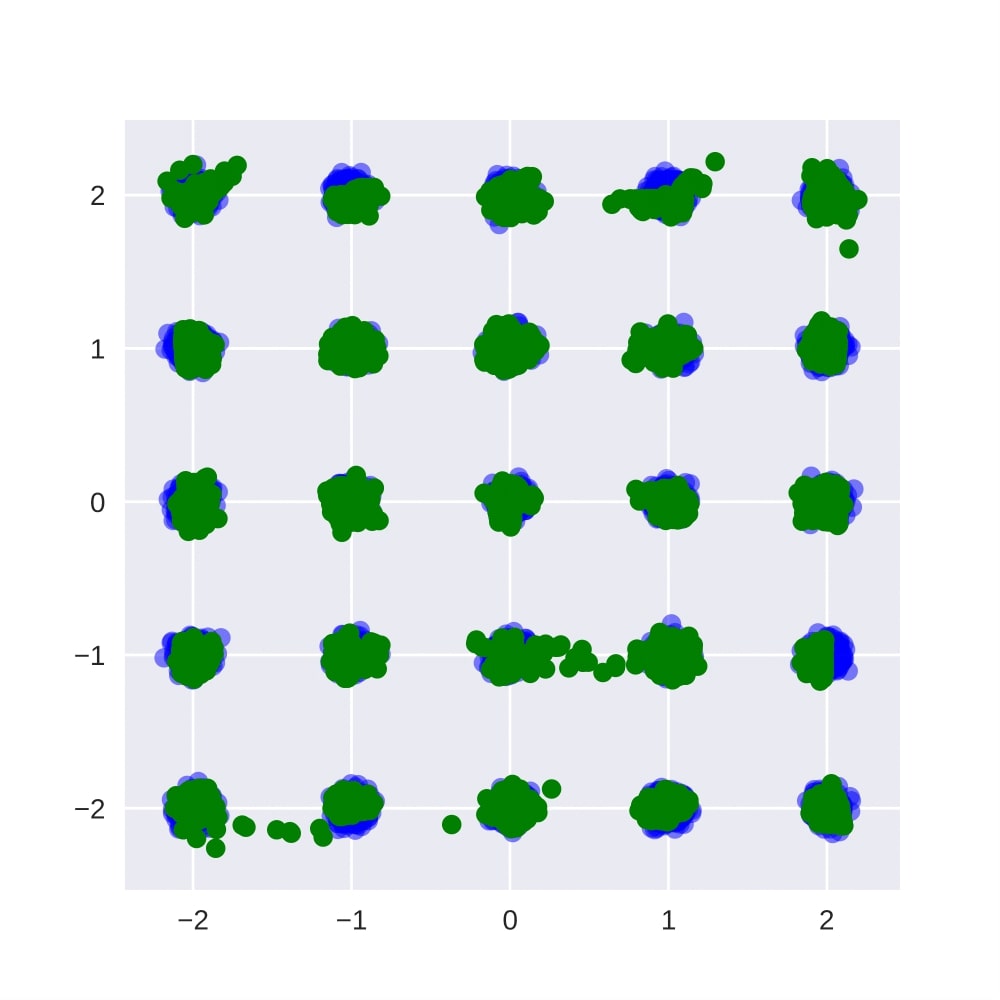}\label{fig:sim_visual_SP_MH}}
    		\subfloat[][DRE-F-SP+SIR]{\includegraphics[width=0.3\textwidth, height=4cm]{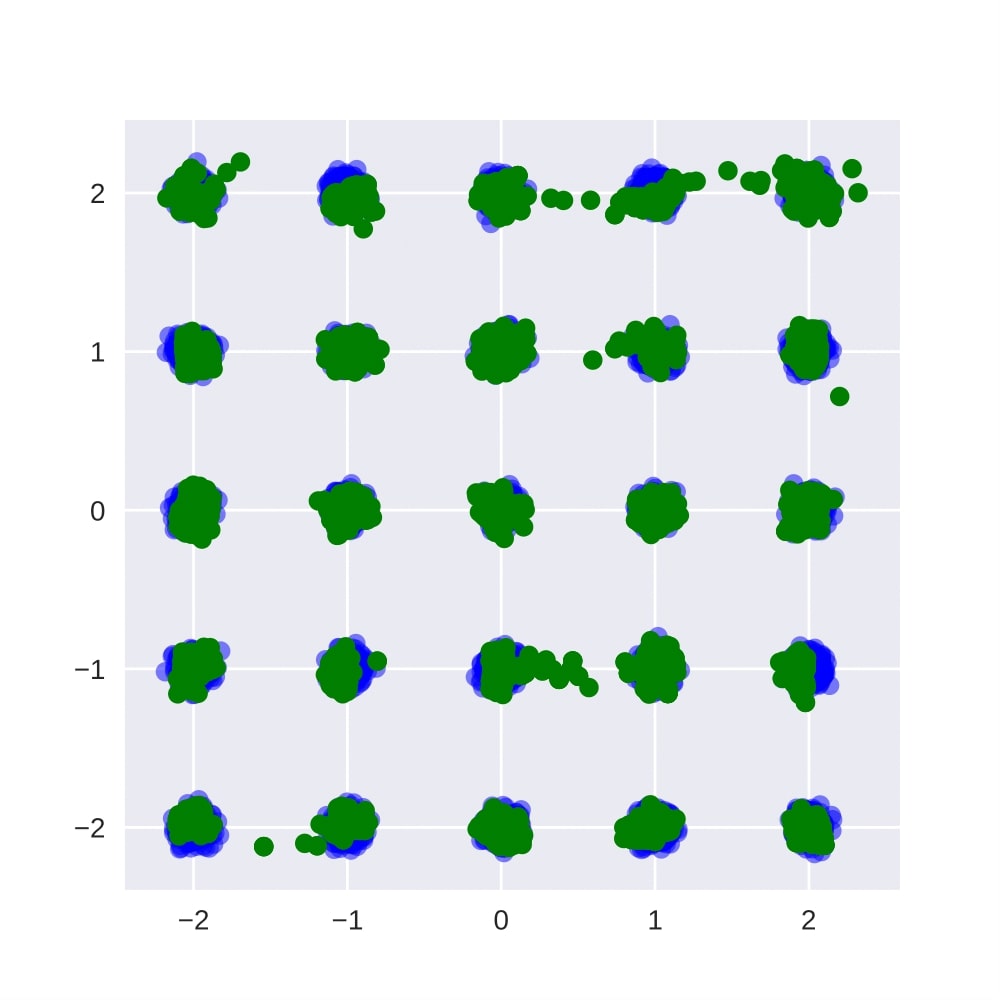}\label{fig:sim_visual_SP_SIR}}
    	\end{minipage}
    	\caption{Visual results of the 25 2-D Gaussians example. Each setting is repeated for three times and we visualize the results of the first round here. In each figure, blue dots denote 10,000 real samples in the test set and green dots denote 10,000 fake samples from each method. The GAN model is trained for only 50 epochs so the discriminator and generator do not reach their optimality. Thus many samples generated by this GAN shown in Fig. \ref{fig:sim_visual_gan} are between-modes. DRS and MH-GAN are effective but we still observe many between-modes samples in Fig. \ref{fig:sim_visual_DRS} and \ref{fig:sim_visual_MHGAN}. On the other hand, our proposed methods can nearly correct the bias in the generator and almost all generated samples in Fig. \ref{fig:sim_visual_SP_RS} to \ref{fig:sim_visual_SP_SIR} are ``high quality".}
    	\label{fig:sim_visual_results}
    \end{figure*}

    {\setlength{\parindent}{0cm}\changes{\textbf{The superiority of the Softplus loss:}}}
    \changes{Besides the main study and the ablation study, we conduct an extra experiment to empirically show why the SP loss performs better than other loss functions such as the uLSIF loss. In this study, the 5-layer MLP as a density ratio model is trained for 5000 epochs with the penalized uLSIF loss \eqref{eq:optim_uLSIF_penlaty} and our penalized SP loss \eqref{eq:penalized_SP} respectively when $\lambda=0.05$. The initial learning rate is $10^{-3}$ and decayed every 1000 epochs with factor $0.1$ (see Supp. \ref{appendix:convergence_uLSIF_SP} for details). From Fig. \ref{fig:convergence_uLSIF} and \ref{fig:convergence_SP}, we can see the SP loss converges after 1000 epochs while the \textit{uLSIF loss does not stop decreasing until the 3000th epoch and then starts fluctuating over a large range}. The uLSIF loss stops decreasing because of a too small learning rate (Fig. \ref{fig:convergence_uLSIF_not_decay_lr} of the supplemental material shows the training curve of the 5-layer MLP under the uLSIF loss with a constant learning rate $10^{-5}$ where we can observe a constantly descending trend).  We draw 10,000 fake samples from the trained GAN and evaluate the 5-layer MLP on the high and low quality fake samples separately. Fig. \ref{fig:dr_HQ_vs_epoch} and \ref{fig:dr_LQ_vs_epoch} show the average density ratios on high/low quality samples and the percentage high quality samples versus epoch. We estimate $p_g$ by a Gaussian mixture model \cite{reynolds2009gaussian} and $p_r$ is known so we can get the true density ratio function which is used to compute the ground truth. From Fig. \ref{fig:dr_HQ_vs_epoch}, when using the SP loss, the average density ratio of high quality samples is slightly above the ground truth and does not decrease over epochs. \textit{This implies the SP loss does not overfit the training data and the penalty term takes effect.} Note that the SP loss may overfit training data if $\lambda=0$; see Supp.\ref{appendix:convergence_uLSIF_SP} for details. In contrast, when using uLSIF loss, the average density ratio of high quality samples decreases after around 900 epochs and is always below the ground truth implying that the \textit{uLSIF loss overfits the training data and the penalty term does not effectively control its unboundedness}. Fig. \ref{fig:dr_LQ_vs_epoch} shows that the uLSIF loss tends to overestimate the density ratios of low quality samples while the SP loss performs optimally. The underestimation and overestimation of the uLSIF loss results in a small difference between high and low quality samples from the density ratio perspective and makes it difficult for the subsequent sampler to distinguish between high and low quality samples. These findings explain why SP loss outperforms uLSIF loss when subsampling.}
    
    \begin{figure*}[ht]
    	\centering
    	\subfloat[][Penalized uLSIF loss \eqref{eq:optim_uLSIF_penlaty}]{
    		\includegraphics[width=0.48\textwidth, height=5cm]{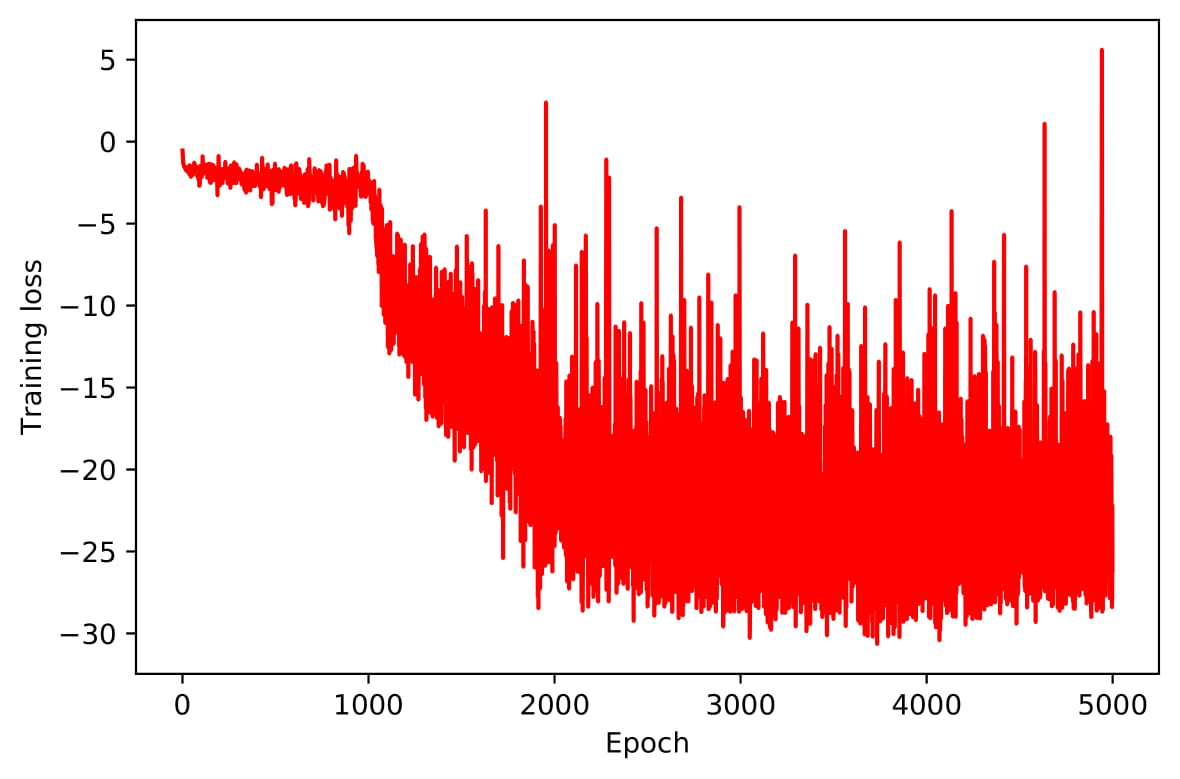}
    		\label{fig:convergence_uLSIF}}
    	\subfloat[][Penalized Softplus loss (ours) \eqref{eq:penalized_SP}]{
    		\includegraphics[width=0.48\textwidth, height=5cm]{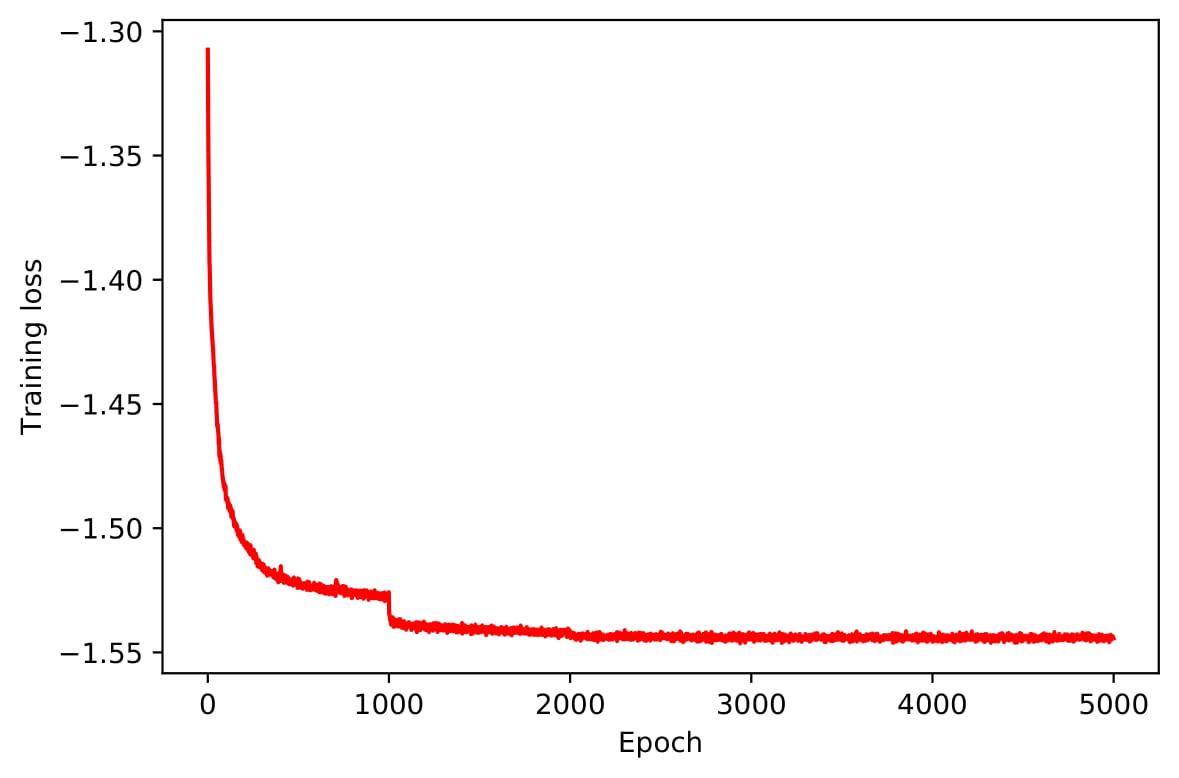}
    		\label{fig:convergence_SP}}
    	\caption{Training curves of a 5-layer MLP under the penalized uLSIF loss \eqref{eq:optim_uLSIF_penlaty} and the penalized SP loss \eqref{eq:penalized_SP} when $\lambda=0.05$ in the 25 2-D Gaussians example.}
    	\label{fig:convergence_uLSIF_SP}
    \end{figure*}
    
    \begin{figure*}[ht]
    	\centering
    	\subfloat[][Average density ratio (solid lines) of \textit{high quality} fake samples and \% high quality samples (dotted lines) versus epoch of DRE training.]{
    		\includegraphics[width=0.48\textwidth, height=5.5cm]{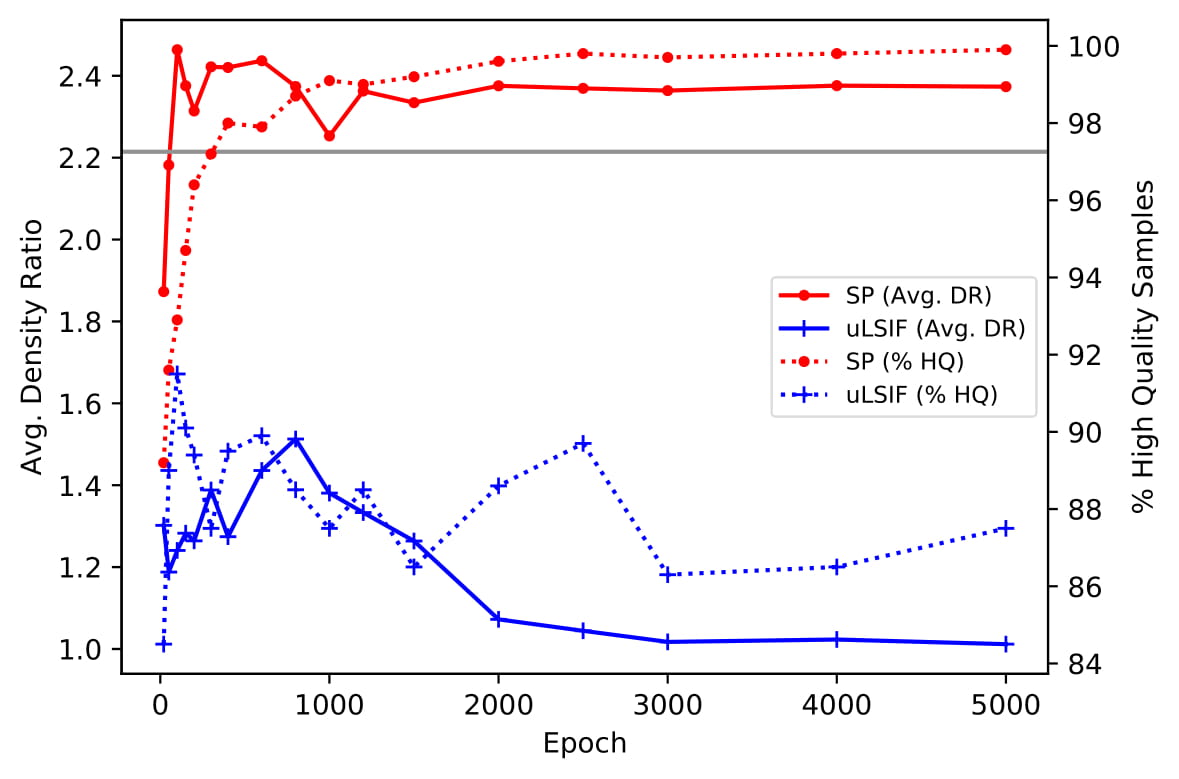}
    		\label{fig:dr_HQ_vs_epoch}}
    	\subfloat[][Average density ratio (solid lines) of \textit{low quality} fake samples and \% high quality samples (dotted lines) versus epoch of DRE training.]{
    		\includegraphics[width=0.48\textwidth, height=5.5cm]{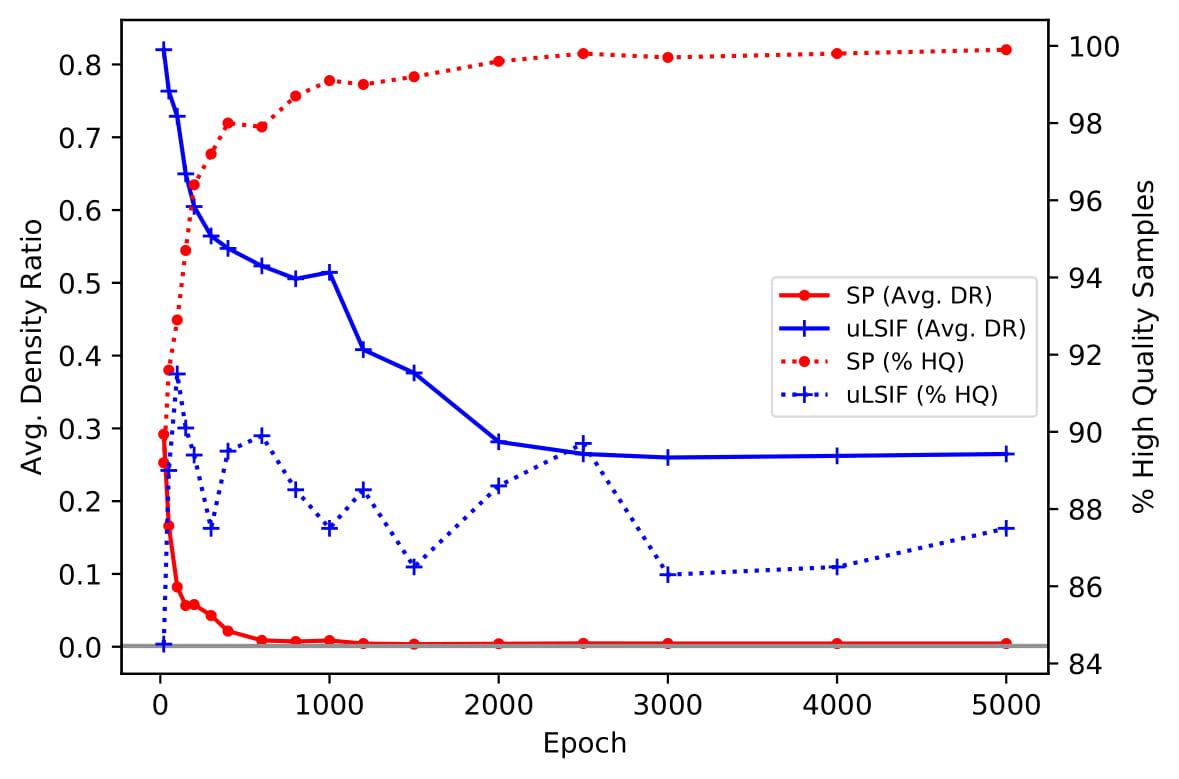}
    		\label{fig:dr_LQ_vs_epoch}}
    	\caption{The density ratio estimation and subsampling performance of a 5-layer MLP trained under the penalized uLSIF loss \eqref{eq:optim_uLSIF_penlaty} and the penalized SP loss \eqref{eq:penalized_SP} when $\lambda=0.05$ in the 25 2-D Gaussians example. The grey lines stand for the ground truth average density ratios.}
    	\label{fig:dre_subsampling_uLSIF_SP}
    \end{figure*}

    \subsection{CIFAR-10 Dataset}\label{sec:exp_cifar10}
    In this section, our main study is to empirically show the superiority of our approach to DRS \cite{azadi2018discriminator} and MH-GAN \cite{turner2018metropolis} in subsampling DCGAN \cite{radford2015unsupervised}, WGAN-GP \cite{gulrajani2017improved}, and MMD-GAN \cite{li2017mmd}  trained on the CIFAR-10 \cite{krizhevsky2009learning} dataset. We also conduct two extra ablation studies to investigate the reason behind the efficacy of our approach. 
    
    {\setlength{\parindent}{0cm}\textbf{Experimental setup of the main study:}} The CIFAR-10 dataset consists of 60,000 $32\times 32$ RGB images which are classified into 10 classes. The dataset is split into a training set of 50,000 images with 5000 per class and a validation set of 10,000 images with 1000 per class. GANs are trained with network architectures and training setups shown in Supp. \ref{appendix:cifar10_nets}. 
    
    We use DRE-F-SP+RS, DRE-F-SP+MH, and DRE-F-SP+SIR to subsample 50,000 fake images from a trained GAN. At density ratio estimation stage, we train a ResNet-34 \cite{he2016deep} on the training set with a modified architecture shown in Supp. \ref{appendix:cifar10_nets} where we incorporate an extra fully connected layer to output a feature map with dimension $(32\times 32\times 3)\times 1=3072\times 1$. A 5-layer MLP is used as the density ratio model $\hat{\psi}(\bm{y};\bm{\beta})$ in Eq.\eqref{eq:DRE_F_SP} to map the extracted features of an image to its density ratio, which is trained with the penalized SP loss \eqref{eq:penalized_SP} on the training set and fake images from the trained GAN. Detailed training setups of the ResNet-34 and the 5-layer MLP are described in Supp. \ref{appendix:cifar10_training_setups}. We conduct hyperparameter selection on a grid of values from 0 to 0.1 on the validation set (shown in Table \ref{tab:cifar10_hyperparameter_selection} of the supplemental material). At the sampling stage, the number of burn-in samples $N^\prime$ for RS in Alg.\ref{alg:DRE-F-SP+RS} is set to 50,000; the MCMC iterations $K$ for MH in Alg.\ref{alg:DRE-F-SP+MH} is set to 640; the pool size $N_p$ for SIR in Alg.\ref{alg:DRE-F-SP+SIR} is set to 100,000.
    
    We consider three competitors: no subsampling, DRS \cite{azadi2018discriminator} and MH-GAN \cite{turner2018metropolis}. We use each subsampling method to draw 50,000 fake images from a trained GAN. No subsampling refers to directly sampling from a generator. When implementing DRS, following the setting of \cite{azadi2018discriminator} on ImageNet dataset, we set $\gamma$ dynamically for each batch of fake samples drawn from the GAN to the $80$th percentile of the $F(\bm{x})$ in Eq.\eqref{eq:DRS_F_hat} for each $\bm{x}$ in this batch. Continuing to train the discriminator on the validation set does not improve the performance of DRS, so we do not conduct ``keep training". Since the discriminator of WGAN-GP outputs a class score instead of a probability, we apply the calibration technique in MH-GAN \cite{turner2018metropolis} to calibrate the trained discriminator on the validation set with logistic regression to let it output class probabilities. When implementing MH-GAN, following \cite{turner2018metropolis}, the MCMC iteration $K$ is set to 640. Note that, as we mentioned in Section \ref{sec:DRS_and_MH-GAN}, DRS and MH-GAN cannot be applied to MMD-GAN.
    
    In the main study, we subsample 50,000 fake images with each subsampling method from each GAN three times. The average quality of 50,000 fake images of each method over three repetitions is reported in Table \ref{tab:results_cifar10_main}. Note that, in real data analysis, we only repeat subsampling three times, but train each GAN and each density ratio model only once. 
    
    {\setlength{\parindent}{0cm}\textbf{Experimental setup of two ablation studies:}} The first ablation study aims at justifying the effectiveness of our proposed DRE-F-SP in subsampling three types of GANs. We consider four other density ratio estimation methods for images in the comparison: DRE-P-uLSIF \cite{nam2015direct}, DRE-P-DSKL \cite{khan2019deep}, DRE-P-BARR \cite{khan2019deep} and BOC \cite{grover2019bias}. The architectures of the 2-layer CNN for DRE-P-uLSIF and the 6-layer CNN for DRE-P-DSKL and DRE-P-BARR are shown in Supp. \ref{appendix:cifar10_nets}. When implementing BOC, we train a CNN as the Bayes optimal classifier with the architecture proposed in \cite{grover2019bias} and shown in Supp. \ref{appendix:cifar10_nets} on 10,000 hold-out validation images and 10,000 fake images. We attach a RS sampler to these DRE methods and conduct the same three repetitions of the main study. We report in Table \ref{tab:results_cifar10_DRE_compare} the average quality of 50,000 fake images for different DRE methods over three repetitions. 
    
    The second ablation study focuses on researching the effect of different loss functions on the final subsampling performance. We replace the SP loss in DRE-F-SP with other loss functions---uLSIF \cite{nam2015direct}, DSKL \cite{khan2019deep} and BARR \cite{khan2019deep}---while using the same RS sampler and the same 5-layer MLP. The average quality of 50,000 fake images for each loss over three repetitions is shown in Table \ref{tab:results_cifar10_loss_compare}.
    
    {\setlength{\parindent}{0cm}\textbf{Evaluation metrics:}} We evaluate the quality of fake images from different subsampling methods by \textit{Inception Score} (IS) \cite{salimans2016improved} and \textit{Fr\'echet Inception Distance} (FID) \cite{heusel2017gans}. They are two popular evaluation metrics for GANs; see Supp. \ref{appendix:IS_FID} for more details. Larger IS and smaller FID are better.
    
    {\setlength{\parindent}{0cm}\textbf{Quantitative results:}} Table \ref{tab:results_cifar10_main} shows the results of the main study and demonstrates our approaches significantly outperform other existing subsampling methods and can also dramatically improve MMD-GAN, where DRS and MH-GAN are not applicable. 
    
    Table \ref{tab:results_cifar10_DRE_compare} shows the results of  ablation study 1. Four existing DRE methods for images are applied in this case, but they are incapable of improving any GAN model, let alone outperforming DRE-F-SP. This ablation study demonstrates that the effectiveness of the three proposed subsampling methods results from our proposed density ratio estimation method---DRE-F-SP. 
    
    Table \ref{tab:results_cifar10_loss_compare} shows the results of the ablation study 2 and demonstrates the novel SP loss plays a crucial role in the success of the density ratio estimation in the feature space.
    
    {\setlength{\parindent}{0cm}\textbf{Visual results:}} We also show in Fig. \ref{fig:cifar_visual_results_dcgan} to \ref{fig:cifar_visual_results_mmdgan} of Supp. \ref{appendix:cifar10_visual_results} some example images from each subsampling method in the main study.

    \begin{table}[htbp]
    	\centering
    	\footnotesize
    	\caption{Average quality of 50,000 fake CIFAR-10 images from different subsampling methods over three repetitions. We draw 50,000 fake images by each method on which we compute the IS and FID. We repeat this sampling three times and report the average IS and FID. Higher IS and lower FID are better. A grid search is conducted for DRE-F-SP to select the hyperparameter, and the results under the optimal $\lambda^*$ are shown in this table. We include the IS and FID of 50,000 training data and 10,000 test data as a reference.}
    	\begin{tabular}{lll}
    		\toprule
    		Method & IS (mean$\pm$std)    & FID (mean$\pm$std) \\
    		\midrule
    		\textbf{- Real Data -} &       &  \\
    		50,000 Training Data & 9.984 & --- \\
    		10,000 Test Data & 9.462 & 0.134 \\
    		\midrule
    		\textbf{- DCGAN -} &       &  \\
    		No Subsampling & $6.261\pm 0.003$ & $3.006\pm 0.011$ \\
    		DRS \cite{azadi2018discriminator}   & $6.385\pm 0.004$ & $2.930\pm 0.008$ \\
    		MH-GAN \cite{turner2018metropolis} & $6.300\pm 0.010$ & $2.982\pm 0.009$ \\
    		DRE-F-SP+RS ($\lambda^*=0$) & $\bm{8.597\pm 0.011}$ & $\bm{1.664\pm 0.007}$ \\
    		DRE-F-SP+MH ($\lambda^*=0$) & $\bm{8.588\pm 0.007}$ & $\bm{1.669\pm 0.004}$ \\
    		DRE-F-SP+SIR ($\lambda^*=0$) & $\bm{8.572\pm 0.021}$ & $\bm{1.685\pm 0.027}$ \\
    		\midrule
    		\textbf{- WGAN-GP -} &       &  \\
    		No Subsampling & $6.445\pm 0.015$ & $2.944\pm 0.004$ \\
    		DRS \cite{azadi2018discriminator}   & $6.427\pm 0.012$ & $2.947\pm 0.013$ \\
    		MH-GAN \cite{turner2018metropolis} & $6.428\pm 0.021$ & $2.948\pm 0.014$ \\
    		DRE-F-SP+RS ($\lambda^*=0.005$) & $\bm{8.625\pm 0.013}$ & $\bm{1.774\pm 0.011}$ \\
    		DRE-F-SP+MH ($\lambda^*=0.005$) & $\bm{8.606\pm 0.013}$ & $\bm{1.796\pm 0.014}$ \\
    		DRE-F-SP+SIR ($\lambda^*=0.005$) & $\bm{8.605\pm 0.043}$ & $\bm{1.826\pm 0.030}$ \\
    		\midrule
    		\textbf{- MMD-GAN -} &       &  \\
    		No Subsampling & $5.508\pm 0.016$ & $3.682\pm 0.007$ \\
    		DRS \cite{azadi2018discriminator}   & ---     & --- \\
    		MH-GAN \cite{turner2018metropolis} & ---     & --- \\
    		DRE-F-SP+RS  ($\lambda^*=0.006$) & $\bm{7.800\pm 0.012}$ & $\bm{2.471\pm 0.017}$ \\
    		DRE-F-SP+MH ($\lambda^*=0.006$) & $\bm{7.782\pm 0.008}$ & $\bm{2.469\pm 0.007}$ \\
    		DRE-F-SP+SIR ($\lambda^*=0.006$) & $\bm{7.740\pm 0.017}$ & $\bm{2.525\pm 0.045}$ \\
    		\bottomrule
    	\end{tabular}%
    	\label{tab:results_cifar10_main}%
    \end{table}%
    
    \begin{table}[htbp]
    	\centering
    	\footnotesize
    	\caption{Ablation study 1 on CIFAR-10. The average quality of 50,000 fake CIFAR-10 images from subsampling methods with different DRE methods but the same RS sampler over three repetitions.}
    	\begin{tabular}{lll}
    		\toprule
    		Method & IS (mean$\pm$std)    & FID (mean$\pm$std) \\
    		\midrule
    		\textbf{- DCGAN -} & & \\
    		DRE-P-uLSIF \cite{nam2015direct} & $6.340\pm 0.005$ & $2.773\pm 0.004$ \\
    		DRE-P-DSKL \cite{khan2019deep} & $5.584\pm 0.015$ & $3.986\pm 0.008$ \\
    		DRE-P-BARR \cite{khan2019deep} & $6.191\pm 0.002$ & $3.156\pm 0.013$ \\
    		BOC \cite{grover2019bias}   & $6.259\pm 0.005$ & $3.003\pm 0.003$ \\
    		DRE-F-SP & $\bm{8.597\pm 0.011}$ & $\bm{1.664\pm 0.007}$ \\
    		\midrule
    		\textbf{- WGAN-GP -} & & \\
    		DRE-P-uLSIF \cite{nam2015direct} & $6.418\pm 0.008$ & $2.897\pm 0.007$ \\
    		DRE-P-DSKL \cite{khan2019deep} & $6.274\pm 0.010$ & $2.998\pm 0.003$ \\
    		DRE-P-BARR \cite{khan2019deep} & $6.427\pm 0.006$ & $2.945\pm 0.005$ \\
    		BOC \cite{grover2019bias}   & $6.431\pm 0.022$ & $2.953\pm 0.007$ \\
    		DRE-F-SP & $\bm{8.625\pm 0.013}$ & $\bm{1.774\pm 0.011}$ \\
    		\midrule
    		\textbf{- MMD-GAN -} & & \\
    		DRE-P-uLSIF \cite{nam2015direct} & $5.427\pm 0.007$ & $3.776\pm 0.004$ \\
    		DRE-P-DSKL \cite{khan2019deep} & $5.473\pm 0.002$ & $3.668\pm 0.008$ \\
    		DRE-P-BARR \cite{khan2019deep} & $5.465\pm 0.008$ & $3.733\pm 0.000$ \\
    		BOC \cite{grover2019bias}   & $5.384\pm 0.013$ & $3.884\pm 0.006$ \\
    		DRE-F-SP & $\bm{7.800\pm 0.012}$ & $\bm{2.471\pm 0.017}$ \\
    		\bottomrule
    	\end{tabular}%
    	\label{tab:results_cifar10_DRE_compare}%
    \end{table}%
    
    \begin{table}[htbp]
    	\centering
    	\footnotesize
    	\caption{Ablation study 2 on CIFAR-10. The average quality of 50,000 fake CIFAR-10 images from subsampling methods with different loss functions but the same DR model and RS sampler over three repetitions.}
    	\begin{tabular}{lll}
    		\toprule
    		Method & IS (mean$\pm$std)    & FID (mean$\pm$std) \\
    		\midrule
    		\textbf{- DCGAN -} &       &  \\
    		uLSIF \cite{nam2015direct} & $8.036\pm 0.007$ & $2.194\pm 0.018$ \\
    		DSKL \cite{khan2019deep}  & $6.619\pm 0.010$ & $2.736\pm 0.004$ \\
    		BARR \cite{khan2019deep}  & $6.910\pm 0.005$ & $2.582\pm 0.006$ \\
    		SP    & $\bm{8.597\pm 0.011}$ & $\bm{1.664\pm 0.007}$ \\
    		\midrule
    		\textbf{- WGAN-GP -} &       &  \\
    		uLSIF \cite{nam2015direct} & $8.435\pm 0.006$ & $1.943\pm 0.020$ \\
    		DSKL \cite{khan2019deep}  & $5.941\pm 0.003$ & $3.623\pm 0.005$ \\
    		BARR \cite{khan2019deep}  & $6.966\pm 0.009$ & $2.528\pm 0.010$ \\
    		SP    & $\bm{8.625\pm 0.013}$ & $\bm{1.774\pm 0.011}$ \\
    		\midrule
    		\textbf{- MMD-GAN -} &       &  \\
    		uLSIF \cite{nam2015direct} & $7.760\pm 0.017$ & $2.503\pm 0.017$ \\
    		DSKL \cite{khan2019deep}  & $5.590\pm 0.005$ & $3.765\pm 0.003$ \\
    		BARR \cite{khan2019deep}  & $5.763\pm 0.008$ & $3.488\pm 0.005$ \\
    		SP    & $\bm{7.800\pm 0.012}$ & $\bm{2.471\pm 0.017}$ \\
    		\bottomrule
    	\end{tabular}%
    	\label{tab:results_cifar10_loss_compare}%
    \end{table}%
    
    \section{Conclusion}\label{sec:discussion}
    We propose a novel subsampling framework (including DRE-F-SP+RS, DRE-F-SP+MH, and DRE-F-SP+SIR) for GANs to replace DRS \cite{azadi2018discriminator} and MH-GAN \cite{turner2018metropolis}. In this framework, a novel SP loss function is proposed for density ratio estimation, and its rate of convergence is determined theoretically with respect to training size. Based on the SP loss, we further propose to do density ratio estimation in the feature space learned by a specially designed ResNet-34. We demonstrate the efficiency of the overall framework on a 25 2-D Gaussians example and the CIFAR-10 dataset. Experimental results show that our proposed framework can dramatically improve different types of GANs and substantially outperform DRS and MH-GAN. Our approach can also improve GANs (e.g., MMD-GAN), where DRS and MH-GAN are not applicable.


    \bibliographystyle{IEEEtran}
    \bibliography{BIB_Importance_Reampling_GANs}

\begin{thebibliography}{10}
\providecommand{\url}[1]{#1}
\csname url@samestyle\endcsname
\providecommand{\newblock}{\relax}
\providecommand{\bibinfo}[2]{#2}
\providecommand{\BIBentrySTDinterwordspacing}{\spaceskip=0pt\relax}
\providecommand{\BIBentryALTinterwordstretchfactor}{4}
\providecommand{\BIBentryALTinterwordspacing}{\spaceskip=\fontdimen2\font plus
\BIBentryALTinterwordstretchfactor\fontdimen3\font minus
  \fontdimen4\font\relax}
\providecommand{\BIBforeignlanguage}[2]{{%
\expandafter\ifx\csname l@#1\endcsname\relax
\typeout{** WARNING: IEEEtran.bst: No hyphenation pattern has been}%
\typeout{** loaded for the language `#1'. Using the pattern for}%
\typeout{** the default language instead.}%
\else
\language=\csname l@#1\endcsname
\fi
#2}}
\providecommand{\BIBdecl}{\relax}
\BIBdecl

\bibitem{goodfellow2014generative}
I.~Goodfellow, J.~Pouget-Abadie, M.~Mirza, B.~Xu, D.~Warde-Farley, S.~Ozair,
  A.~Courville, and Y.~Bengio, ``Generative adversarial nets,'' in
  \emph{Advances in Neural Information Processing Systems 27}, 2014, pp.
  2672--2680.

\bibitem{wang2018perceptual}
C.~{Wang}, C.~{Xu}, C.~{Wang}, and D.~{Tao}, ``Perceptual adversarial networks
  for image-to-image transformation,'' \emph{IEEE Transactions on Image
  Processing}, vol.~27, no.~8, pp. 4066--4079, 2018.

\bibitem{wang2018deeply}
Q.~{Wang}, H.~{Fan}, L.~{Zhu}, and Y.~{Tang}, ``Deeply supervised face
  completion with multi-context generative adversarial network,'' \emph{IEEE
  Signal Processing Letters}, vol.~26, no.~3, pp. 400--404, 2019.

\bibitem{hsu2018sigan}
C.-C. Hsu, C.-W. Lin, W.-T. Su, and G.~Cheung, ``{SiGAN}: Siamese generative
  adversarial network for identity-preserving face hallucination,'' \emph{IEEE
  Transactions on Image Processing}, vol.~28, no.~12, pp. 6225--6236, 2019.

\bibitem{quan2018compressed}
T.~M. Quan, T.~Nguyen-Duc, and W.-K. Jeong, ``Compressed sensing {MRI}
  reconstruction using a generative adversarial network with a cyclic loss,''
  \emph{IEEE Transactions on Medical Imaging}, vol.~37, no.~6, pp. 1488--1497,
  2018.

\bibitem{gao2019universal}
Y.~Gao, Y.~Liu, Y.~Wang, Z.~Shi, and J.~Yu, ``A universal intensity
  standardization method based on a many-to-one weak-paired cycle generative
  adversarial network for magnetic resonance images,'' \emph{IEEE Transactions
  on Medical Imaging}, vol.~38, no.~9, pp. 2059--2069, 2019.

\bibitem{wei2019facial}
J.~Wei, G.~Lu, H.~Liu, and J.~Yan, ``Facial image inpainting with deep
  generative model and patch search using region weight,'' \emph{IEEE Access},
  vol.~7, pp. 67\,456--67\,468, 2019.

\bibitem{brock2018large}
A.~Brock, J.~Donahue, and K.~Simonyan, ``Large scale {GAN} training for high
  fidelity natural image synthesis,'' in \emph{International Conference on
  Learning Representations}, 2019.

\bibitem{miyato2018spectral}
T.~Miyato, T.~Kataoka, M.~Koyama, and Y.~Yoshida, ``Spectral normalization for
  generative adversarial networks,'' in \emph{International Conference on
  Learning Representations}, 2018.

\bibitem{zhang2019self}
H.~Zhang, I.~Goodfellow, D.~Metaxas, and A.~Odena, ``Self-attention generative
  adversarial networks,'' in \emph{International Conference on Machine
  Learning}, 2019, pp. 7354--7363.

\bibitem{arjovsky2017wasserstein}
M.~Arjovsky, S.~Chintala, and L.~Bottou, ``{W}asserstein generative adversarial
  networks,'' in \emph{Proceedings of the 34th International Conference on
  Machine Learning}, vol.~70, 2017, pp. 214--223.

\bibitem{gulrajani2017improved}
I.~Gulrajani, F.~Ahmed, M.~Arjovsky, V.~Dumoulin, and A.~C. Courville,
  ``Improved training of {Wasserstein GANs},'' in \emph{Advances in Neural
  Information Processing Systems 30}, 2017, pp. 5767--5777.

\bibitem{villani2008optimal}
C.~Villani, \emph{Optimal transport: old and new}, 2008, vol. 338.

\bibitem{li2017mmd}
C.-L. Li, W.-C. Chang, Y.~Cheng, Y.~Yang, and B.~Poczos, ``{MMD GAN}: Towards
  deeper understanding of moment matching network,'' in \emph{Advances in
  Neural Information Processing Systems 30}, 2017, pp. 2203--2213.

\bibitem{gretton2012kernel}
A.~Gretton, K.~M. Borgwardt, M.~J. Rasch, B.~Sch{\"o}lkopf, and A.~Smola, ``A
  kernel two-sample test,'' \emph{Journal of Machine Learning Research},
  vol.~13, no. Mar, pp. 723--773, 2012.

\bibitem{azadi2018discriminator}
S.~Azadi, C.~Olsson, T.~Darrell, I.~Goodfellow, and A.~Odena, ``Discriminator
  rejection sampling,'' in \emph{International Conference on Learning
  Representations}, 2019.

\bibitem{turner2018metropolis}
R.~Turner, J.~Hung, E.~Frank, Y.~Saatchi, and J.~Yosinski,
  ``{M}etropolis-{H}astings generative adversarial networks,'' in
  \emph{Proceedings of the 36th International Conference on Machine Learning},
  vol.~97, 2019, pp. 6345--6353.

\bibitem{nam2015direct}
H.~Nam and M.~Sugiyama, ``Direct density ratio estimation with convolutional
  neural networks with application in outlier detection,'' \emph{IEICE
  Transactions on Information and Systems}, vol.~98, no.~5, pp. 1073--1079,
  2015.

\bibitem{khan2019deep}
H.~Khan, L.~Marcuse, and B.~Yener, ``Deep density ratio estimation for change
  point detection,'' \emph{arXiv preprint arXiv:1905.09876}, 2019.

\bibitem{grover2019bias}
A.~Grover, J.~Song, A.~Agarwal, K.~Tran, A.~Kapoor, E.~Horvitz, and S.~Ermon,
  ``Bias correction of learned generative models using likelihood-free
  importance weighting,'' \emph{arXiv preprint arXiv:1906.09531}, 2019.

\bibitem{robert2010introducing}
C.~P. Robert, G.~Casella, and G.~Casella, \emph{Introducing {Monte Carlo}
  methods with {R}}, 2010, vol.~18.

\bibitem{bolic2005resampling}
M.~Bolic, P.~M. Djuric, and S.~Hong, ``Resampling algorithms and architectures
  for distributed particle filters,'' \emph{IEEE Transactions on Signal
  Processing}, vol.~53, no.~7, pp. 2442--2450, 2005.

\bibitem{bregman1967relaxation}
L.~Bregman, ``The relaxation method of finding the common point of convex sets
  and its application to the solution of problems in convex programming,''
  \emph{USSR Computational Mathematics and Mathematical Physics}, vol.~7,
  no.~3, pp. 200 -- 217, 1967.

\bibitem{varshney2011bayes}
K.~R. Varshney, ``Bayes risk error is a {Bregman} divergence,'' \emph{IEEE
  Transactions on Signal Processing}, vol.~59, no.~9, pp. 4470--4472, 2011.

\bibitem{sugiyama2012density}
M.~Sugiyama, T.~Suzuki, and T.~Kanamori, \emph{Density ratio estimation in
  machine learning}, 1st~ed., 2012.

\bibitem{Chakravarti1967KStest}
L.~Chakravarti and Roy, \emph{Handbook of methods of applied statistics. Volume
  I: Techniques of Computation Descriptive Methods, and Statistical
  Inference.}, 1967.

\bibitem{van2014renyi}
T.~Van~Erven and P.~Harremos, ``R{\'e}nyi divergence and {Kullback-Leibler}
  divergence,'' \emph{IEEE Transactions on Information Theory}, vol.~60, no.~7,
  pp. 3797--3820, 2014.

\bibitem{foundation_ml}
M.~Mohri, A.~Rostamizadeh, and A.~Talwalkar, \emph{Foundations of machine
  learning}, 2012.

\bibitem{lafferty2010}
\BIBentryALTinterwordspacing
J.~Lafferty, H.~Liu, and L.~Wasserman, ``Concentration of measure.'' [Online].
  Available: \url{http://www.stat.cmu.edu/~larry/=sml/Concentration.pdf}
\BIBentrySTDinterwordspacing

\bibitem{uehara2016generative}
M.~Uehara, I.~Sato, M.~Suzuki, K.~Nakayama, and Y.~Matsuo, ``Generative
  adversarial nets from a density ratio estimation perspective,'' \emph{arXiv
  preprint arXiv:1610.02920}, 2016.

\bibitem{noteULLN}
\BIBentryALTinterwordspacing
X.~Shi, ``Some useful asymptotic theory.'' [Online]. Available:
  \url{https://www.ssc.wisc.edu/~xshi/econ715/Lecture_2_some_asymptotic_theorems.pdf}
\BIBentrySTDinterwordspacing

\bibitem{jennrich1969asymptotic}
R.~I. Jennrich, ``Asymptotic properties of non-linear least squares
  estimators,'' \emph{The Annals of Mathematical Statistics}, vol.~40, no.~2,
  pp. 633--643, 1969.

\bibitem{he2016deep}
K.~He, X.~Zhang, S.~Ren, and J.~Sun, ``Deep residual learning for image
  recognition,'' in \emph{Proceedings of the IEEE conference on computer vision
  and pattern recognition}, 2016, pp. 770--778.

\bibitem{krizhevsky2009learning}
A.~Krizhevsky, G.~Hinton \emph{et~al.}, ``Learning multiple layers of features
  from tiny images,'' Citeseer, Tech. Rep., 2009.

\bibitem{lecun1998gradient}
Y.~LeCun, L.~Bottou, Y.~Bengio, and P.~Haffner, ``Gradient-based learning
  applied to document recognition,'' \emph{Proceedings of the IEEE}, vol.~86,
  no.~11, pp. 2278--2324, 1998.

\bibitem{liu2015faceattributes}
Z.~Liu, P.~Luo, X.~Wang, and X.~Tang, ``Deep learning face attributes in the
  wild,'' in \emph{Proceedings of International Conference on Computer Vision
  (ICCV)}, December 2015.

\bibitem{reynolds2009gaussian}
D.~A. Reynolds, ``Gaussian mixture models,'' \emph{Encyclopedia of biometrics},
  vol. 741, 2009.

\bibitem{radford2015unsupervised}
A.~Radford, L.~Metz, and S.~Chintala, ``Unsupervised representation learning
  with deep convolutional generative adversarial networks,'' \emph{arXiv
  preprint arXiv:1511.06434}, 2015.

\bibitem{salimans2016improved}
T.~Salimans, I.~Goodfellow, W.~Zaremba, V.~Cheung, A.~Radford, X.~Chen, and
  X.~Chen, ``Improved techniques for training {GANs},'' in \emph{Advances in
  Neural Information Processing Systems 29}, 2016, pp. 2234--2242.

\bibitem{heusel2017gans}
M.~Heusel, H.~Ramsauer, T.~Unterthiner, B.~Nessler, and S.~Hochreiter, ``{GANs}
  trained by a two time-scale update rule converge to a local {Nash}
  equilibrium,'' in \emph{Advances in Neural Information Processing Systems},
  2017, pp. 6626--6637.

\bibitem{wu2018group}
Y.~Wu and K.~He, ``Group normalization,'' in \emph{Proceedings of the European
  Conference on Computer Vision (ECCV)}, 2018, pp. 3--19.

\bibitem{ioffe2015batch}
S.~Ioffe and C.~Szegedy, ``Batch normalization: accelerating deep network
  training by reducing internal covariate shift,'' in \emph{International
  Conference on Machine Learning}, 2015, pp. 448--456.

\bibitem{pmlr-v15-coates11a}
A.~Coates, A.~Ng, and H.~Lee, ``An analysis of single-layer networks in
  unsupervised feature learning,'' in \emph{Proceedings of the Fourteenth
  International Conference on Artificial Intelligence and Statistics}, ser.
  Proceedings of Machine Learning Research, G.~Gordon, D.~Dunson, and
  M.~Dudík, Eds., vol.~15, 2011, pp. 215--223.

\bibitem{kingma2014adam}
D.~P. Kingma and J.~Ba, ``Adam: a method for stochastic optimization,''
  \emph{arXiv preprint arXiv:1412.6980}, 2014.

\bibitem{szegedy2016rethinking}
C.~Szegedy, V.~Vanhoucke, S.~Ioffe, J.~Shlens, and Z.~Wojna, ``Rethinking the
  inception architecture for computer vision,'' in \emph{Proceedings of the
  IEEE conference on computer vision and pattern recognition}, 2016, pp.
  2818--2826.

\bibitem{imagenet_cvpr09}
J.~Deng, W.~Dong, R.~Socher, L.-J. Li, K.~Li, and L.~Fei-Fei, ``{ImageNet: A
  large-scale hierarchical image database},'' in \emph{CVPR09}, 2009.

\end{thebibliography}
    
    \newpage

    \begin{appendices}
    	\renewcommand{\thesection}{S.\Roman{section}} 
    	\renewcommand{\thesubsection}{\thesection.\alph{subsection}}
    	\renewcommand\thefigure{\thesection.\arabic{figure}}
    	\renewcommand\thetable{\thesection.\arabic{table}}
    	\renewcommand{\theequation}{S.\arabic{equation}}
    	
    	\section{More Details of Simulation in Section \ref{sec:exp_simulation}}\label{appendix:Sim}
    	\setcounter{figure}{0}
    	\setcounter{table}{0}
    	\subsection{Network Architectures}\label{appendix:Sim_nets}
    	\begin{table}[h]%
    		\centering
    		\caption{Network architectures for the generator and discriminator in the simulation. ``fc" denotes a fully-connected layer.}%
    		\subfloat[][Generator]{\begin{tabular}{c}
    				\toprule
    				$z\in \mathbbm{R}^{128}\sim N(0,I)$ \\ \hline
    				fc$\rightarrow 100$; ReLU \\ \hline
    				fc$\rightarrow 100$; ReLU \\ \hline
    				fc$\rightarrow 100$; ReLU \\ \hline
    				fc$\rightarrow 2$ \\ 
    				\bottomrule
    		\end{tabular}}\\
    		\subfloat[][Discriminator]{\begin{tabular}{c}
    				\toprule
    				A sample $x\in\mathbbm{R}^{2}$ \\ \hline
    				fc$\rightarrow 100$; ReLU \\ \hline
    				fc$\rightarrow 100$; ReLU \\ \hline
    				fc$\rightarrow 100$; ReLU \\ \hline
    				fc$\rightarrow 1$; Sigmoid \\ 
    				\bottomrule
    		\end{tabular}}
    		\label{tab:sim_gan_architecture}%
    	\end{table}
    	
    	\begin{table}[h]
    		\centering
    		\caption{Architecture of the 5-layer MLP for DRE in the simulation. We use group normalization (GN) \cite{wu2018group} in each hidden layer instead of batch normalization \cite{ioffe2015batch} because we find batch normalization performs quite differently in the training stage and evaluation stage.}
    		\begin{tabular}{c}
    			\toprule
    			A sample $x\in\mathbbm{R}^2$\\ \hline
    			fc$\rightarrow 2048$; GN (4 groups); ReLU; dropout ($p=0.2$)\\ \hline
    			fc$\rightarrow 1024$; GN (4 groups); ReLU; dropout($p=0.2$)\\ \hline
    			fc$\rightarrow 512$; GN (4 groups); ReLU; dropout($p=0.2$)\\ \hline
    			fc$\rightarrow 256$; GN (4 groups); ReLU; dropout($p=0.2$)\\ \hline
    			fc$\rightarrow 128$; GN (4 groups); ReLU; dropout($p=0.2$)\\ \hline
    			fc$\rightarrow 1$; ReLU\\ 
    			\bottomrule
    		\end{tabular}
    		\label{tab:sim_mlp_architecture}
    	\end{table}

    	\subsection{Training Setups}\label{appendix:Sim_training_setups}
    	The GAN model is trained for 50 epochs with the Adam \cite{pmlr-v15-coates11a} optimizer, a constant learning rate $10^{-3}$ and batch size 512. DR models are trained with the setups in Table \ref{tab:sim_setups}, and the hyperparameter selections are in Table \ref{tab:sim_parameter_selection}.
    	
    	\begin{table}[h]
    		\caption{Setups for training the 5-layer MLP under different loss functions in the simulation.}
    		\begin{center}
    			\begin{tabular}{ccccc}
    				\toprule
    				Loss & uLISF & DSKL & BARR & SP \\ 
    				\midrule
    				Optimizer & Adam \cite{kingma2014adam} & Adam \cite{kingma2014adam} & Adam \cite{kingma2014adam} & Adam \cite{kingma2014adam} \\ \hline
    				Constant LR & 1e-5 & 1e-5 & 1e-5 & 1e-3 \\ \hline
    				Epochs & 400 & 400 & 400 & 400 \\ \hline
    				Batch Size & 512 & 512 & 512 & 512 \\ 
    				\bottomrule
    			\end{tabular}
    		\end{center}
    		\label{tab:sim_setups}
    	\end{table}
    	
    	\begin{table}[h]
    		\centering
    		\caption{Hyperparameter selection in Simulation. Two-Sample Kolmogorov-Smirnov test statistic is shown for each $\lambda$ at each round.}
    		\begin{tabular}{cccccc}
    			\toprule
    			& \multicolumn{5}{c}{Round 1} \\
    			\cline{2-6}
    			$\lambda$ & 0     & 0.005 & 0.01  & 0.05  & 0.1 \\
    			\midrule
    			KS Stat. & 0.00464 & 0.00388 & \textbf{0.00386} & 0.00496 & 0.0046 \\
    			\bottomrule
    			& \multicolumn{5}{c}{Round 2} \\
    			\cline{2-6}
    			$\lambda$ & 0     & 0.005 & 0.01  & 0.05  & 0.1 \\
    			\midrule
    			KS Stat. & 0.00896 & \textbf{0.00838} & 0.01044 & 0.0096 & 0.01068 \\
    			\bottomrule
    			& \multicolumn{5}{c}{Round 3} \\
    			\cline{2-6}
    			$\lambda$ & 0     & 0.005 & 0.01  & 0.05  & 0.1 \\
    			\midrule
    			KS Stat. & 0.00494 & 0.00508 & \textbf{0.00398} & 0.008 & 0.00588 \\
    			\bottomrule
    		\end{tabular}%
    		\label{tab:sim_parameter_selection}%
    	\end{table}%
    	
    	\section{How does the loss function affect density ratio estimation and the subsequent subsampling?}\label{appendix:convergence_uLSIF_SP}
    	\setcounter{figure}{0}
    	\setcounter{table}{0}
    	
    	\subsection{Experimental Setups}\label{appendix:convergence_uLSIF_SP_setups}
    	
    	We repeat the following steps 5 times and report the average result. For each one, we do the following:
    	
    	\begin{enumerate}
    		\item Train the GAN for 50 epochs. Train the density ratio model (MLP-5) with the SP loss and uLSIF loss respectively for 5000 epochs. We use the Adam \cite{pmlr-v15-coates11a} optimizer. The initial learning rate for both loss functions is set to $10^{-3}$ and decayed every 1000 epochs. Set 17 checkpoints during training from the 20th epoch to the 5000th epoch. The hyperparameter $\lambda$ is set to 0 or 0.05. 
    		\item Use GMM to model $p_g$. The optimal number of components is selected by minimizing BIC.  Since we can draw a large number of fake samples from $p_g$, the estimation of $p_g$ by GMM should be very accurate. In this experiment, we use 100,000 fake samples to fit the GMM. The true distribution $p_r$ is a mixture of 25 Gaussian with parameters known to us. Therefore, we can know the ground truth density ratio function $r=p_r/p_g$. 
    		\item At each checkpoint, evaluate the ground truth density ratio function $r$ and two estimated density ratio functions $\hat{r}$ (trained with SP and uLSIF respectively) on all 10000 fake samples, only high-quality fake samples, and only low-quality fake samples. Report the average density ratio of these fake samples at each checkpoint.
    		\item At each checkpoint, we use DRE-SP+SIR to draw 10,000 fake samples and report the percentage of high quality samples.
    	\end{enumerate}

    	\subsection{Results for $\lambda=0.05$}\label{appendix:convergence_uLSIF_SP_lambda0.05}
    	In Fig. \ref{fig:convergence_uLSIF_SP} of the main article, we show the results when $\lambda=0.05$. Note that, in Figure \ref{fig:convergence_uLSIF} and \ref{fig:convergence_SP} we only report the training loss of the first round of 5 repetitions; however, in Figure \ref{fig:dr_HQ_vs_epoch} and \ref{fig:dr_LQ_vs_epoch}, the average density ratios of high/low quality samples and the percentage of high quality samples are averaged over 5 repetitions.
    	
    	In Fig. \ref{fig:dr_all_versus_epochDRE_lambda0.05}, we also show the average density ratio over all 10,000 fake samples and the percentage of high quality samples (averaged over 5 repetitions). We can see, when using the SP loss, the average density ratio over all 10,000 fake samples is close to the ground truth. However, when using the uLSIF loss, the average density ratio keeps decreasing (a sign of overfitting).
    	
    	In Fig. \ref{fig:convergence_uLSIF}, the training loss of uLSIF stops decreasing after 3000 epochs because of a too small learning rate which seems to conflict with our argument in Section \ref{sec:DRE_images}. In Fig. \ref{fig:convergence_uLSIF_not_decay_lr}, however, we show the training curve for a constant learning rate $10^{-5}$ and we can see a constantly decreasing trend.
    	
    	\begin{figure}[h]
    		\centering
    		\includegraphics[width=0.48\textwidth, height=5.5cm]{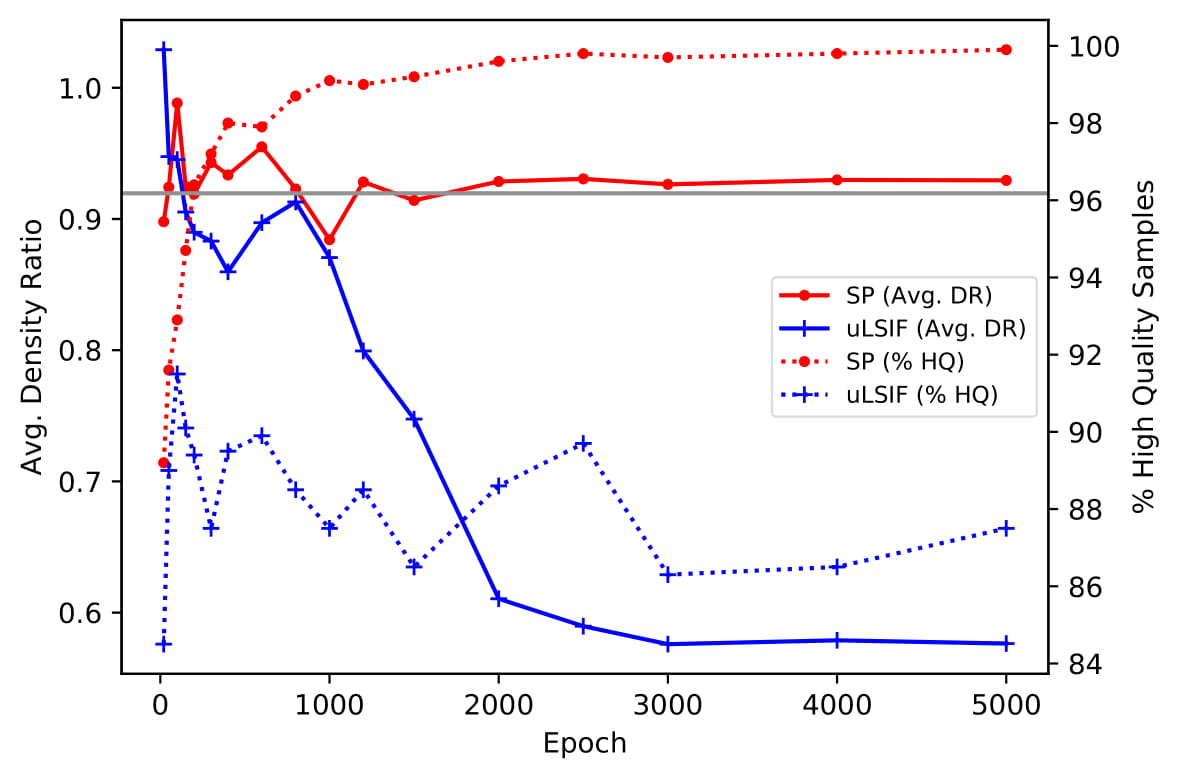}
    		\caption{Average density ratio of 10,000 fake samples and percentage of high quality samples versus the epoch of DRE training when $\lambda=0.05$. The grey horizontal line stands for the ground truth average density ratio.}
    		\label{fig:dr_all_versus_epochDRE_lambda0.05}
    	\end{figure}
    	
    	\begin{figure}[h]
    		\centering
    		\includegraphics[width=0.48\textwidth, height=5.5cm]{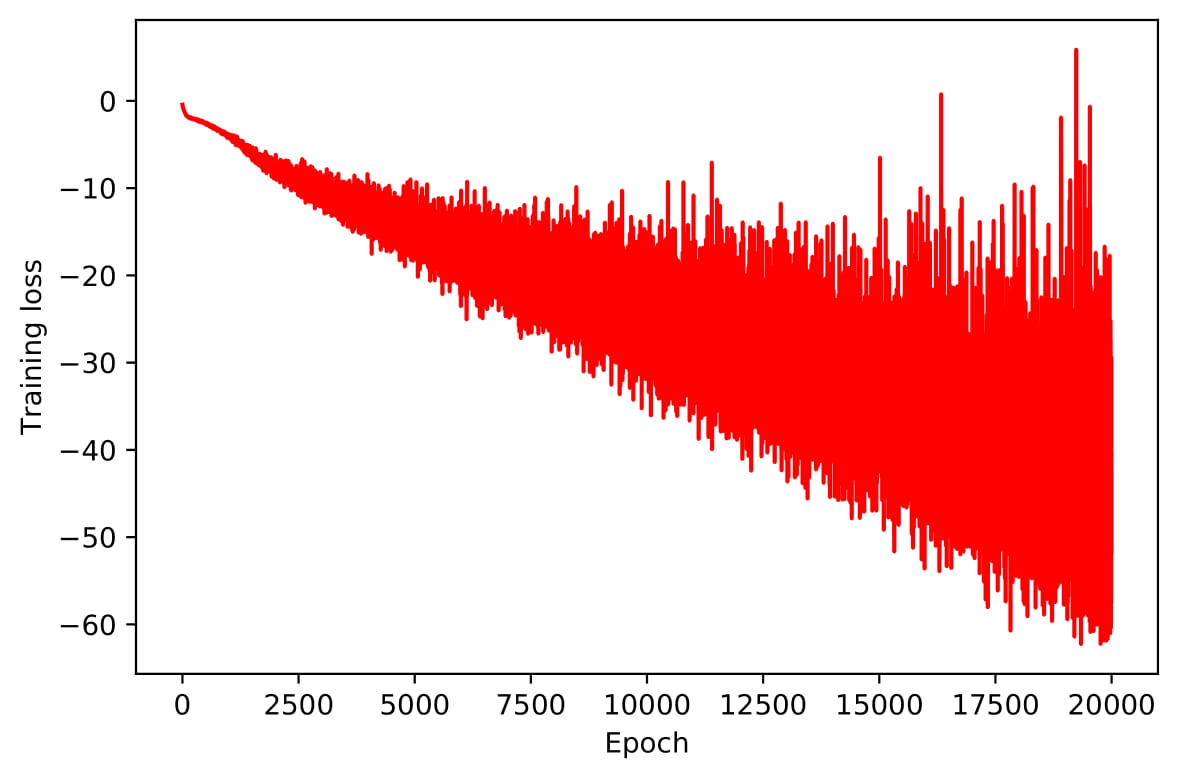}
    		\caption{The training curve of a 5-layer MLP under the penalized uLSIF loss \eqref{eq:optim_uLSIF_penlaty} with $\lambda=0.05$ and a constant learning rate $10^{-5}$.}
    		\label{fig:convergence_uLSIF_not_decay_lr}
    	\end{figure}
    	
    	\subsection{Results for $\lambda=0$}\label{appendix:convergence_uLSIF_SP_lambda0}
    	Besides $\lambda=0.05$, we also set $\lambda$ to 0 and show the corresponding results in Figure \ref{fig:dr_all_versus_epochDRE_lambda0} and \ref{fig:dr_HQorLQ_versus_epochDRE_lambda0}. From Fig. \ref{fig:dr_all_versus_epochDRE_lambda0}, we can see both loss functions suffer from overfitting. However, the overfitting problem for the uLSIF loss is more obvious than the SP loss, because Fig. \ref{fig:dr_HQorLQ_versus_epochDRE_lambda0} shows that the difference between high quality and low quality samples in terms of the average density ratios is bigger when using SP loss. Therefore, even with the overfitting problem, the SP loss still leads to a much better subsampling performance than the uLSIF loss does.
    	
    	\begin{figure}[h]
    		\centering
    		\includegraphics[width=0.48\textwidth, height=5.5cm]{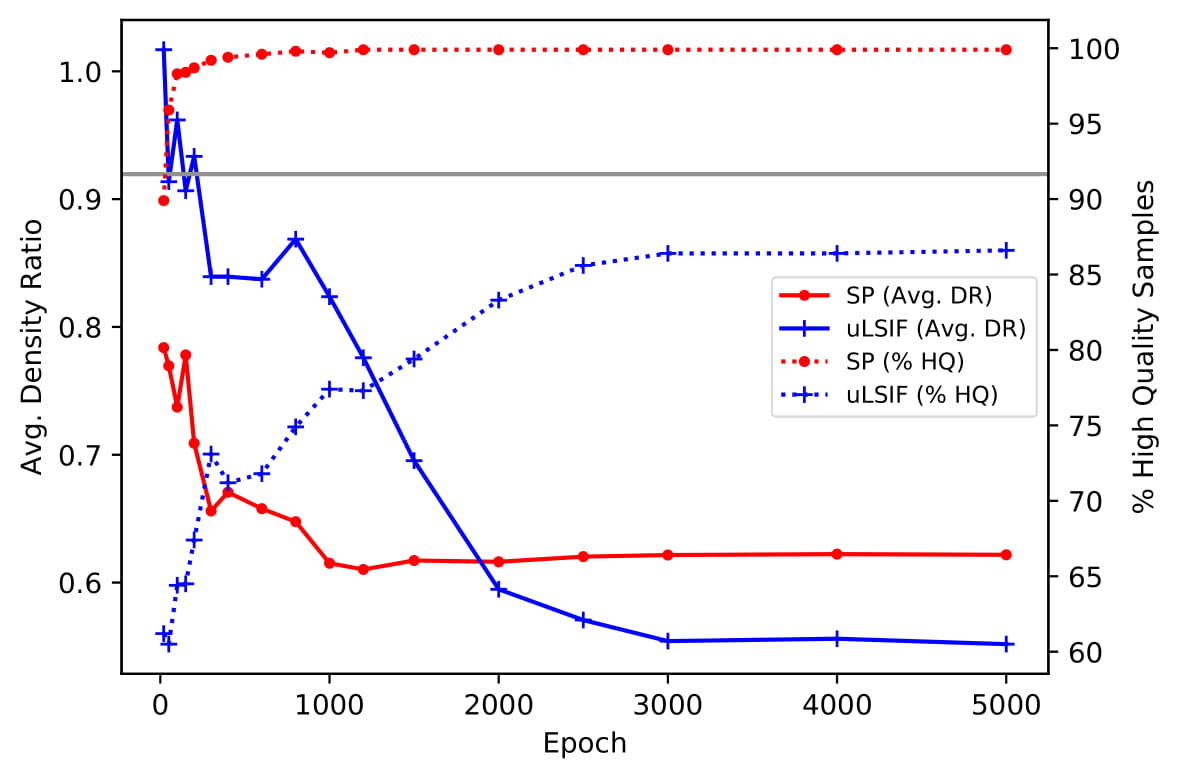}
    		\caption{Average density ratio of 10,000 fake samples and percentage of high quality samples versus the epoch of DRE training when $\lambda=0$. The grey horizontal line stands for the ground truth average density ratio.}
    		\label{fig:dr_all_versus_epochDRE_lambda0}
    	\end{figure}
    	
    	\begin{figure}[h]
    		\centering
    		\subfloat[][Average density ratio of high quality fake samples versus the epoch of DRE training when $\lambda=0$.]{
    			\includegraphics[width=0.48\textwidth, height=5.5cm]{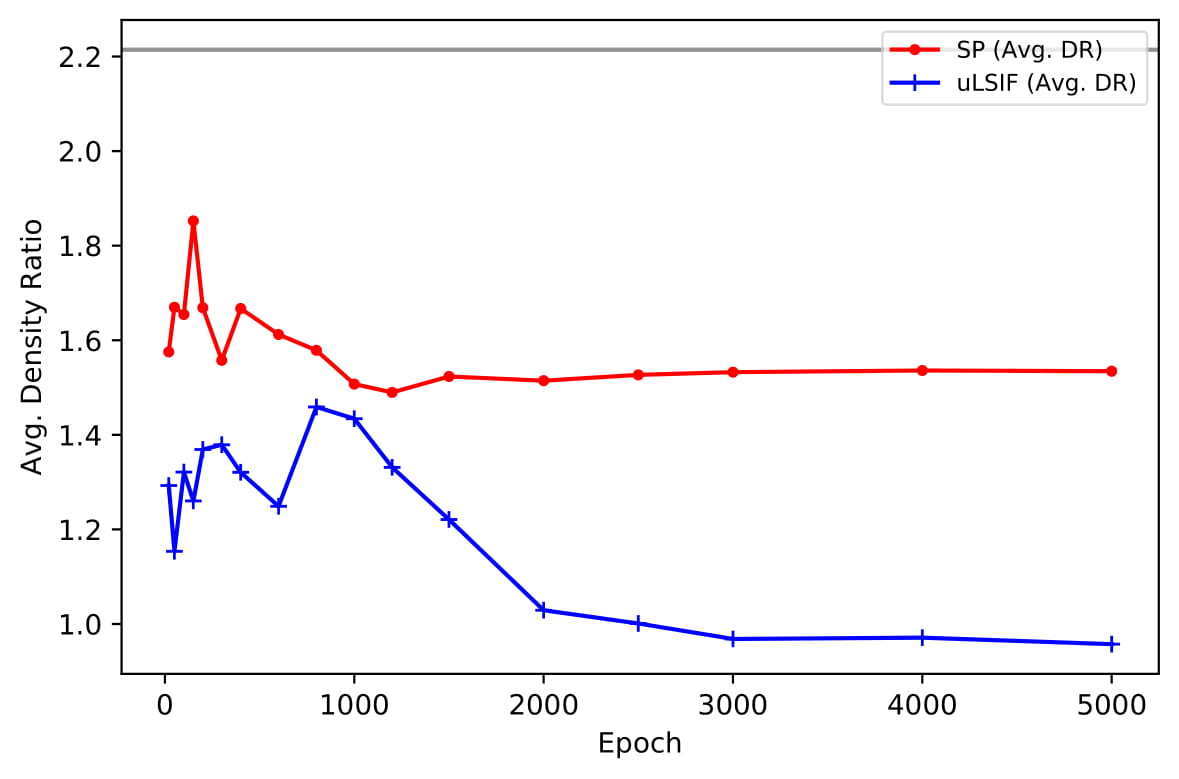}
    			\label{fig:dr_HQ_versus_epochDRE_lambda0}}\\
    		\subfloat[][Average density ratio of low quality fake samples versus the epoch of DRE training when $\lambda=0$.]{
    			\includegraphics[width=0.48\textwidth, height=5.5cm]{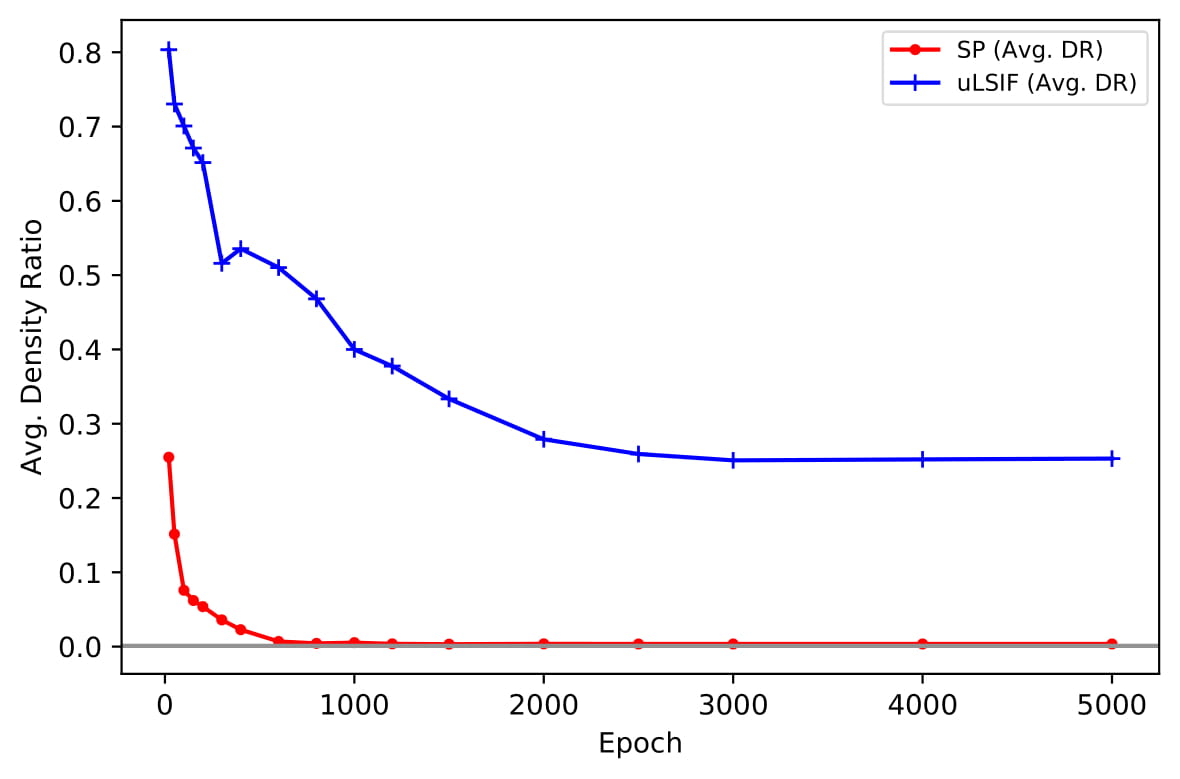}
    			\label{fig:dr_LQ_versus_epochDRE_lambda0}}
    		\caption{Average density ratio of high/low quality fake samples versus the epoch of DRE training when $\lambda=0$. The grey horizontal line stands for the ground truth.}
    		\label{fig:dr_HQorLQ_versus_epochDRE_lambda0}
    	\end{figure}

    	\section{Real Datasets} \label{appendix:real_data}
    	\setcounter{figure}{0}
    	\setcounter{table}{0}
    	\subsection{Evaluation Metrics}\label{appendix:IS_FID}
    	\textit{Inception Score} (IS) \cite{salimans2016improved} is a popular evaluation metric for GANs which is defined as follows:
    	\begin{equation}
    	\label{eq:IS}
    	IS = \exp\{\mathbbm{E}_{\bm{x}}[\mathbbm{KL}(p(y|\bm{x})\|p(y))]\}=\exp\{H(y)-\mathbbm{E}_{\bm{x}}[H(y|\bm{x})] \},
    	\end{equation}
    	where $p(y|\bm{x})$ is the conditional label distribution for an image $\bm{x}$, $p(y)$ is the marginal label distribution, $H(y)$ is the entropy of $y$, and $H(y|\bm{x})$ is the entropy of $y$ conditioning on $\bm{x}$. The distribution $p(y|\bm{x})$ is often modeled by a pre-trained CNN, say Inception-V3 \cite{szegedy2016rethinking}, and $p(y)\approx (1/N)\sum_{i=1}^Np(y|\bm{x}_i)$. IS evaluates the quality of a set of fake images from two perspectives: high classifiability and diversity with respect to class labels. We assume high quality images are more classifiable so we favor smaller $\mathbbm{E}_{\bm{x}}[H(y|\bm{x})]$. On the other hand, high diversity means the GAN model can generate images from all potential classes instead of a few classes so we expect high entropy in the class labels (predicted by the pre-trained Inception-V3) of those fake images (i.e., larger $H(y)$). Therefore, the larger the IS is, the better quality the fake images have. 
    	
    	\textit{Fr\'echet Inception Distance} (FID) \cite{heusel2017gans} is another popular evaluation metric for GAN models. The FID is defined on a feature space learned by a pre-trained CNN and we assume a feature $\bm{y}$ extracted by this pre-trained CNN follows a multivariate normal distribution with mean $\bm{\mu}$ and covariance $\bm{\Sigma}$. In other words, we assume $\bm{y}^r\sim\mathcal{N}(\bm{\mu}_r, \bm{\Sigma}_r)$ and 
    	$\bm{y}^g\sim\mathcal{N}(\bm{\mu}_{g}, \bm{\Sigma}_{g})$, where $\bm{y}^r$ and $\bm{y}^g$ are extracted features of real and fake images respectively. This assumption looks very strong, but empirical studies show that FID is consistent with human judgments and works more robustly than IS does \cite{heusel2017gans}. FID is defined as follows
    	\begin{equation}
    	\label{eq:FID}
    	FID = \|\bm{\mu}_r-\bm{\mu}_{g}\|^2_2 +\text{Tr}(\bm{\Sigma}_r+\bm{\Sigma}_g-2(\bm{\Sigma}_r\bm{\Sigma}_g)^{\frac{1}{2}}),
    	\end{equation}
    	where $\bm{\mu}_r$, $\bm{\mu}_g$, $\bm{\Sigma}_g$ and $\bm{\Sigma}_r$ can be estimated from samples. Note that FID is computed based on both real images and fake images while IS is only computed on fake images.
    	
    	\subsection{CIFAR-10} \label{appendix:cifar10}
    	\subsubsection{Network Architectures}\label{appendix:cifar10_nets}
    	We implement DCGAN and WGAN-GP with generator and discriminator shown in Table \ref{tab:cifar10_gan_architecture}. For MMD-GAN, we directly use codes in \url{https://github.com/OctoberChang/MMD-GAN}. Please see \cite{li2017mmd} for more details about MMD-GAN.
    	
    	\begin{table}[h]%
    		\centering
    		\caption{Network architectures for the generator and discriminator of DCGAN and WGAN-GP in the experiment on CIFAR-10. The slopes of all LeakyReLU are set to 0.2. We denote stride and padding by s and p respectively.}%
    		\subfloat[][Generator]{\begin{tabular}{c}
    				\toprule
    				$z\in \mathbbm{R}^{128}\sim N(0,I)$ \\ \hline
    				fc$\rightarrow 4\times 4\times 512$ \\ \hline
    				deconv, $4\times 4$, $\text{s}=2$, $\text{p}=1$, $256$; BN; ReLU \\ \hline
    				deconv, $4\times 4$, $\text{s}=2$, $\text{p}=1$, $128$; BN; ReLU \\ \hline
    				deconv, $4\times 4$, $\text{s}=2$, $\text{p}=1$, $64$; BN; ReLU \\ \hline
    				conv, $3\times 3$, $\text{s}=1$, $\text{p}=1$, $3$; Tanh \\
    				\bottomrule
    		\end{tabular}}\\
    		\subfloat[][Discriminator]{\begin{tabular}{c}
    				\toprule
    				RGB image $x\in\mathbbm{R}^{3\times 32\times 32}$ \\ \hline
    				conv, $3\times 3$, $\text{s}=1$, $\text{p}=1$, 64; LeakyReLU \\ \hline
    				conv, $4\times 4$, $\text{s}=2$, $\text{p}=1$, 64; LeakyReLU \\ \hline
    				conv, $3\times 3$, $\text{s}=1$, $\text{p}=1$, 128; LeakyReLU \\ \hline
    				conv, $4\times 4$, $\text{s}=2$, $\text{p}=1$, 128; LeakyReLU \\ \hline
    				conv, $3\times 3$, $\text{s}=1$, $\text{p}=1$, 256; LeakyReLU \\ \hline
    				conv, $4\times 4$, $\text{s}=2$, $\text{p}=1$, 256; LeakyReLU \\ \hline
    				conv, $3\times 3$, $\text{s}=1$, $\text{p}=1$, 512; LeakyReLU \\ \hline
    				fc$\rightarrow 1$\\\midrule
    				sigmoid (for DCGAN only)\\
    				\bottomrule
    		\end{tabular}}
    		\label{tab:cifar10_gan_architecture}%
    	\end{table}

    	\begin{figure}[h]
    		\centering
    		\includegraphics[width=0.2\textwidth, height=3in]{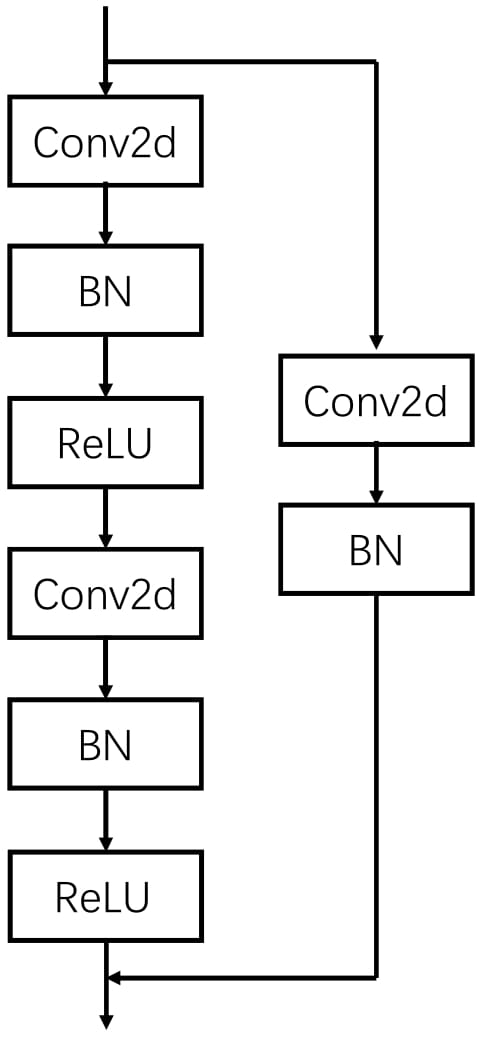}
    		\caption{A residual block (ResBlock) in ResNet-34 for CIFAR-10. Please check our codes for detailed setups of each layer.}
    		\label{fig:resblock}
    	\end{figure}
    	
    	\begin{table}[h]
    		\centering
    		\caption{Architecture of the ResNet-34 for feature extraction for CIFAR-10. ``down" refers to down sampling. We add an extra fully-connected layer fc1 to output features for density ratio estimation in feature space. $\times n$ represents $n$ consecutive such blocks.}
    		\begin{tabular}{c}
    			\toprule
    			RGB image $x\in\mathbbm{R}^{3\times 32\times 32}$ \\
    			\hline
    			conv, 3x3, stride=1, p=1, 64; BN; ReLU \\
    			\hline
    			\{ResBlock, 64\} $\times 3$ \\
    			\hline
    			ResBlock, down, 128 \\
    			\hline
    			\{ResBlock, 128\}$ \times 3$ \\
    			\hline
    			ResBlock, down, 256 \\
    			\hline
    			\{ResBlock, 256\} $\times 5$  \\
    			\hline
    			ResBlock, down, 512 \\
    			\hline
    			\{ResBlock, 512\} $\times 2$ \\
    			\hline
    			Avg. pooling, $4\times 4$, $\text{s}=4$ \\
    			\hline
    			fc1$\rightarrow 32\times32\times3=3072$ \\
    			\hline
    			fc2$\rightarrow 10$ \\
    			\bottomrule
    		\end{tabular}%
    		\label{tab:cifar10_resnet34}%
    	\end{table}%
    	
    	\begin{table}[h]
    		\centering
    		\caption{5-layer MLP for DRE in feature space for CIFAR-10.}
    		\begin{tabular}{c}
    			\toprule
    			extracted feature $\bm{y}\in \mathbbm{R}^{3072}$ \\
    			\hline
    			fc$\rightarrow 2048$, GN (4 groups), ReLU, Dropout($p=0.4$) \\\hline
    			fc$\rightarrow 1024$, GN (4 groups), ReLU, Dropout($p=0.4$) \\\hline
    			fc$\rightarrow 512$, GN (4 groups), ReLU, Dropout($p=0.4$) \\\hline
    			fc$\rightarrow 256$, GN (4 groups), ReLU, Dropout($p=0.4$) \\\hline
    			fc$\rightarrow 128$, GN (4 groups), ReLU, Dropout($p=0.4$) \\\hline
    			fc$\rightarrow 1$, ReLU \\
    			\bottomrule
    		\end{tabular}%
    		\label{tab:cifar10_MLP5}%
    	\end{table}%
    	
    	\begin{table}[h]
    		\centering
    		\caption{2-layer CNN for DRE in pixel space \cite{nam2015direct} for CIFAR-10}
    		\begin{tabular}{c}
    			\toprule
    			RGB image $\bm{x}\in \mathbbm{R}^{3\times 32\times 32}$ \\\hline
    			conv, $9\times 9$, $s=1$, 6 \\\hline
    			Avg. pooling, $2\times 2$, $s=2$ \\\hline
    			Sigmoid \\\hline
    			conv, $9\times 9$, $s=1$, 12 \\\hline
    			Avg. pooling, $2\times 2$, $s=2$ \\\hline
    			Sigmoid \\\hline
    			fc$\rightarrow 1$; ReLU \\
    			\bottomrule
    		\end{tabular}%
    		\label{tab:cifar10_2layerCNN}%
    	\end{table}%
    	
    	\begin{table}[h]
    		\centering
    		\caption{6-layer CNN for DRE in pixel space \cite{khan2019deep} for CIFAR-10}
    		\begin{tabular}{c}
    			\toprule
    			RGB image $\bm{x}\in \mathbbm{R}^{3\times 32\times 32}$ \\\hline
    			conv, $3\times 3$, $s=1$, 60 \\\hline
    			Max pooling, $2\times 2$, $s=1$ \\\hline
    			conv, $3\times 3$, $s=1$, 50 \\\hline
    			conv, $3\times 3$, $s=1$, 40 \\\hline
    			Max pooling, $2\times 2$, $s=1$ \\\hline
    			conv, $3\times 3$, $s=1$, 20 \\\hline
    			Max pooling, $2\times 2$, $s=1$ \\\hline
    			conv, $2\times 2$, $s=1$, 10 \\\hline
    			conv, $2\times 2$, $s=1$, 5 \\\hline
    			fc$\rightarrow 250$, ReLU, Dropout($p=0.25$) \\\hline
    			fc$\rightarrow 1$; ReLU \\
    			\bottomrule
    		\end{tabular}%
    		\label{tab:cifar10_6layerCNN}%
    	\end{table}%

    	\begin{table}[H]
    		\centering
    		\caption{Binary classifier for the BOC-based DRE \cite{grover2019bias} for CIFAR-10}
    		\begin{tabular}{c}
    			\toprule
    			RGB image $\bm{x}\in \mathbbm{R}^{3\times 32\times 32}$ \\\hline
    			conv, $3\times 3$, $s=1$, $p=1$, 64; ReLU; BN \\\hline
    			Max pooling, $2\times 2$, $s=2$ \\\hline
    			conv, $3\times 3$, $s=1$, $p=1$, 64; ReLU; BN \\\hline
    			Max pooling, $2\times 2$, $s=2$ \\\hline
    			conv, $3\times 3$, $s=1$, $p=1$, 64; ReLU; BN \\\hline
    			Max pooling, $2\times 2$, $s=2$ \\\hline
    			conv, $3\times 3$, $s=1$, $p=1$, 64; ReLU; BN \\\hline
    			Max pooling, $2\times 2$, $s=2$ \\\hline
    			fc$\rightarrow 1$; Sigmoid \\
    			\bottomrule
    		\end{tabular}%
    		\label{tab:cifar10_BOC}%
    	\end{table}%

    	\subsubsection{Training Setups}\label{appendix:cifar10_training_setups}
    	In the CIFAR-10 setting, three types of GAN are trained on the 50,000 training images with setups in Table \ref{tab:cifar10_gan_training_setups}. The modified ResNet-34 for feature extraction is trained on the training set for 200 epochs with the SGD optimizer, initial learning rate 0.1 (decay at epoch 100 and 150 with factor 0.1), weight decay $10^{-4}$, and batch size 256. The training setups for different DRE methods are shown in Table \ref{tab:cifar10_dre_setups}.

    	\begin{table}[h] 
    		\centering
    		\caption{Training setups for three types of GAN on CIFAR-10.}
    		\begin{tabular}{cccc}
    			\toprule
    			& DCGAN & WGAN-GP & MMD-GAN \\
    			\midrule
    			Optimizer & Adam  & Adam  & Adam \\
    			Constant learning rate & 2E-04 & 2E-04 & 5E-05 \\
    			Epochs & 500   & 2000  & 4000 \\
    			Batch Size & 256   & 256   & 256 \\
    			\bottomrule
    		\end{tabular}%
    		\label{tab:cifar10_gan_training_setups}%
    	\end{table}%
    	
    	\begin{table}[h]
    		\scriptsize
    		\caption{Training setups for DRE methods on CIFAR-10. Corresponding subsampling results are shown in Table \ref{tab:results_cifar10_main} to \ref{tab:results_cifar10_loss_compare} of the main article.}
    		\begin{center}
    			\begin{tabular}{ccccc}
    				\toprule
    				& DRE-P-uLSIF & DRE-P-DSKL & DRE-P-BARR & BOC \\ \midrule
    				Optimizer & Adam & Adam & Adam & Adam \\ \hline
    				Initia LR & 1E-4 & 1E-4 & 1E-4 & 1E-3 \\ \hline
    				LR Decay & No & No & No & No \\ \hline
    				Epochs & 200 & 200 & 200 & 100 \\ \hline
    				Batch Size & 512 & 512 & 512 & 100 \\ \bottomrule
    				& DRE-F-uLSIF & DRE-F-DSKL & DRE-F-BARR & DRE-F-SP \\ \midrule
    				Optimizer & Adam & Adam & Adam & Adam \\ \hline
    				Initia LR & 1E-5 & 1E-5 & 1E-5 & 1E-4 \\ \hline
    				LR Decay & No & No & No & epoch 100 \\ 
    				&&&&($\times 0.1$)\\\hline
    				Epochs & 200 & 200 & 200 & 200 \\ \hline
    				Batch Size & 512 & 512 & 512 & 512 \\ \bottomrule
    			\end{tabular}
    		\end{center}
    		\label{tab:cifar10_dre_setups}
    	\end{table}
    	
    	\begin{table}[h]
    		\centering
    		\caption{Hyperparameter selection for CIFAR-10. Two-Sample Kolmogorov-Smirnov test statistic is shown for each $\lambda$ and each GAN.}
    		\begin{tabular}{cccccc}
    			\toprule
    			& \multicolumn{5}{c}{DCGAN} \\
    			\cline{2-6}
    			$\lambda$ & 0     & 0.005 & 0.01  & 0.05  & 0.1 \\ \midrule
    			KS Stat. & \textbf{1.380E-01} & 1.390E-01 & 1.386E-01 & 1.389E-01 & 1.394E-01 \\
    			\bottomrule
    			& \multicolumn{5}{c}{WGAN-GP} \\
    			\cline{2-6}
    			$\lambda$ & 0     & 0.005 & 0.01  & 0.05  & 0.1 \\ \midrule
    			KS Stat. & 1.131E-01 & \textbf{1.118E-01} & 1.138E-01 & 1.165E-01 & 1.204E-01 \\
    			\bottomrule
    			& \multicolumn{5}{c}{MMD-GAN} \\
    			\cline{2-6}
    			$\lambda$ & 0     & 0.005 & 0.006 & 0.008 & 0.01 \\ \midrule
    			KS Stat. & 1.220E-01 & 1.231E-01 & \textbf{1.209E-01} & 1.233E-01 & 1.209E-01 \\
    			\bottomrule
    		\end{tabular}%
    		\label{tab:cifar10_hyperparameter_selection}%
    	\end{table}%
    	
    	\subsubsection{Performance Measures} \label{appendix:cifar10_performance_measures}
    	To evaluate the quality of fake images by IS and FID, we train the Inception-V3 on 50,000 CIFAR-10 training images. FID is computed based on the final average pooling features from the pre-trained Inception-V3. We compute the FID between 50,000 fake images and 50,000 training images.

    	\subsubsection{Visual Results} \label{appendix:cifar10_visual_results}
    	We show some example CIFAR-10 images from different subsampling methods in Fig. \ref{fig:cifar_visual_results_dcgan} to \ref{fig:cifar_visual_results_mmdgan}. For each method, we draw images from two classes (car and horse) with 50 images per class (first 5 rows correspond to cars and the rest correspond to horses). The improvement of our methods is more obvious on images from WGAN-GP in Fig. \ref{fig:cifar_visual_results_wgangp} where our methods generate more recognizable cars and horses.

    	\begin{figure*}[h]
    		\centering
    		\subfloat[][No subsampling]{\includegraphics[width=0.45\textwidth, height=7cm]{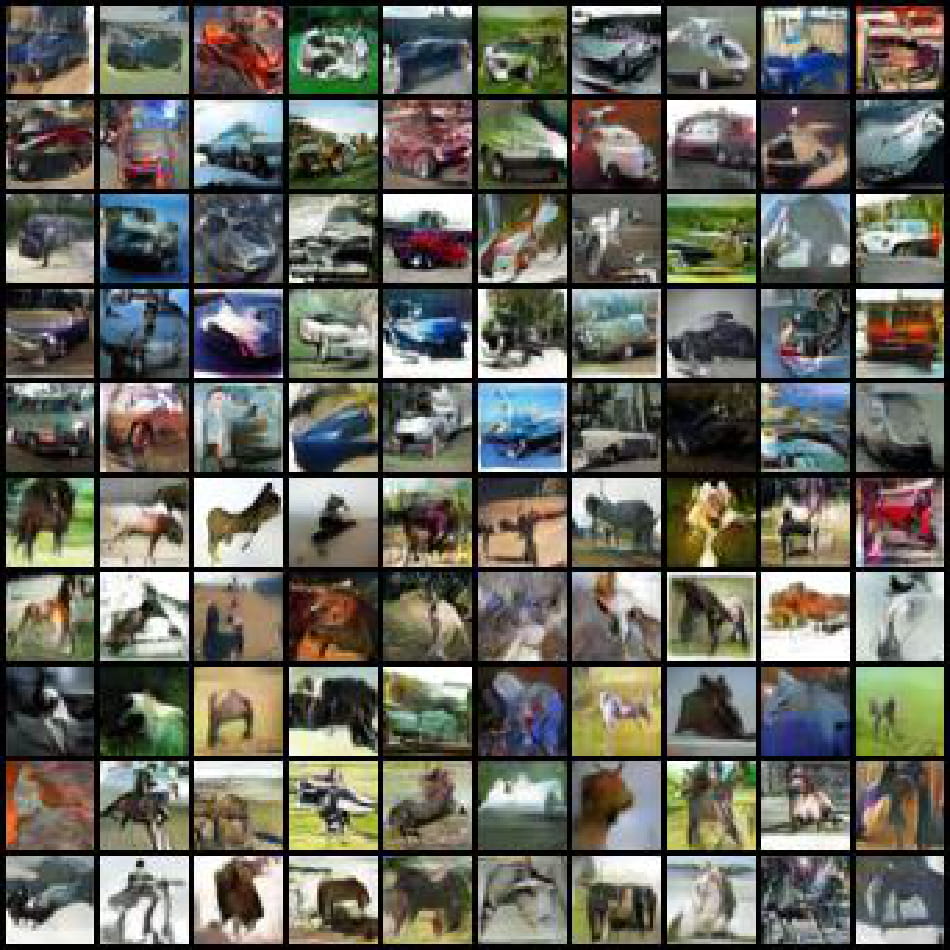}\label{fig:cifar10_visual_dcgan}}\quad
    		\subfloat[][DRS \cite{azadi2018discriminator}]{\includegraphics[width=0.45\textwidth, height=7cm]{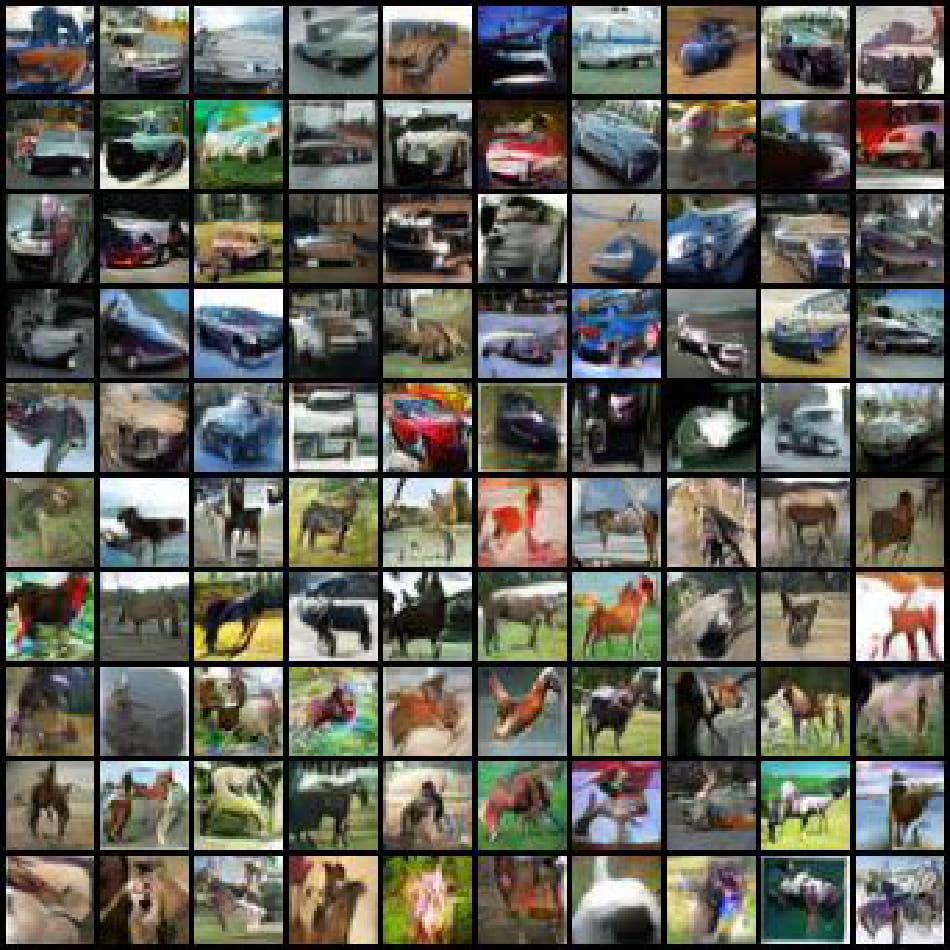}\label{fig:cifar10_visual_dcgan_DRS}}
    		\\
    		\subfloat[][MH-GAN \cite{turner2018metropolis}]{\includegraphics[width=0.45\textwidth, height=7cm]{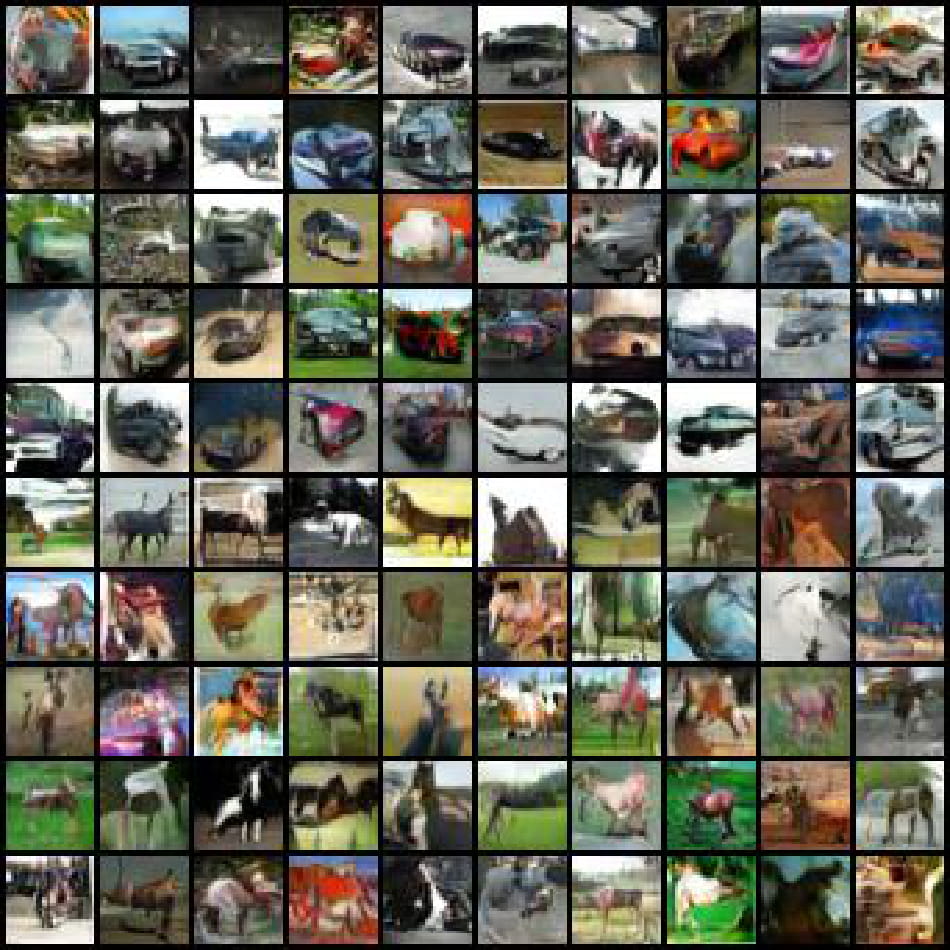}\label{fig:cifar10_visual_dcgan_MHGAN}}\quad
    		\subfloat[][DRE-F-SP+RS]{\includegraphics[width=0.45\textwidth, height=7cm]{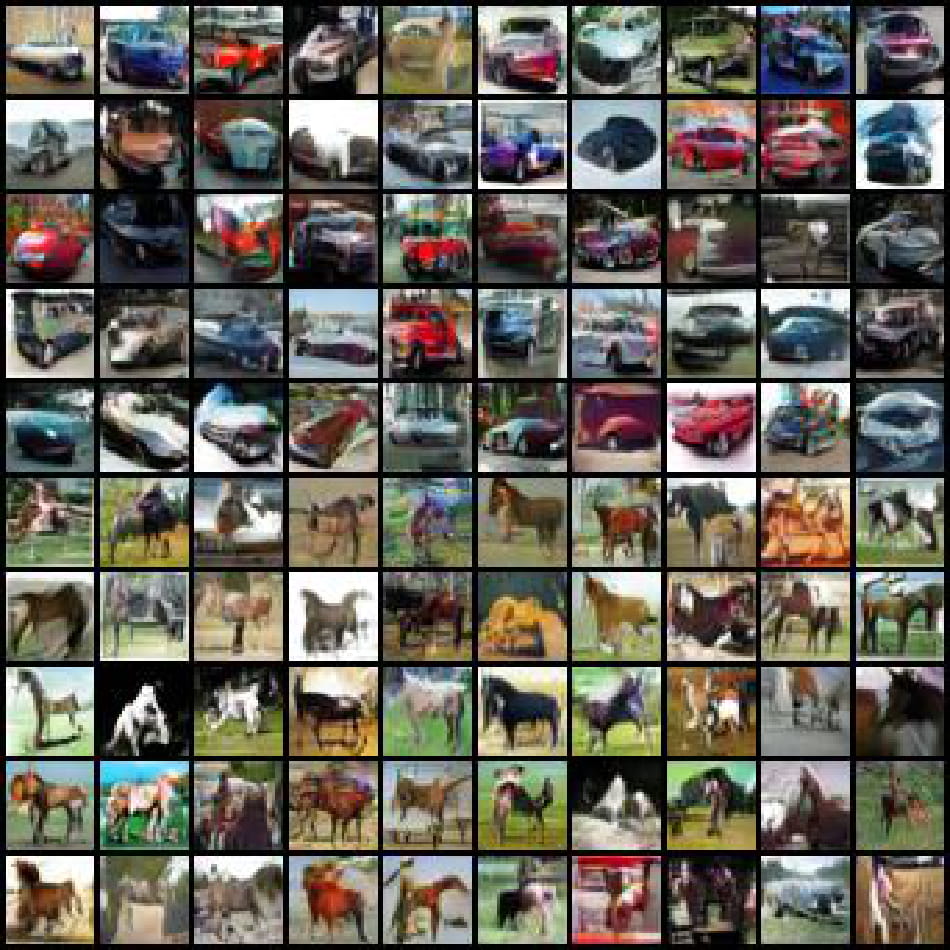}\label{fig:cifar10_visual_dcgan_SP_RS}}
    		\\
    		\subfloat[][DRE-F-SP+MH]{\includegraphics[width=0.45\textwidth, height=7cm]{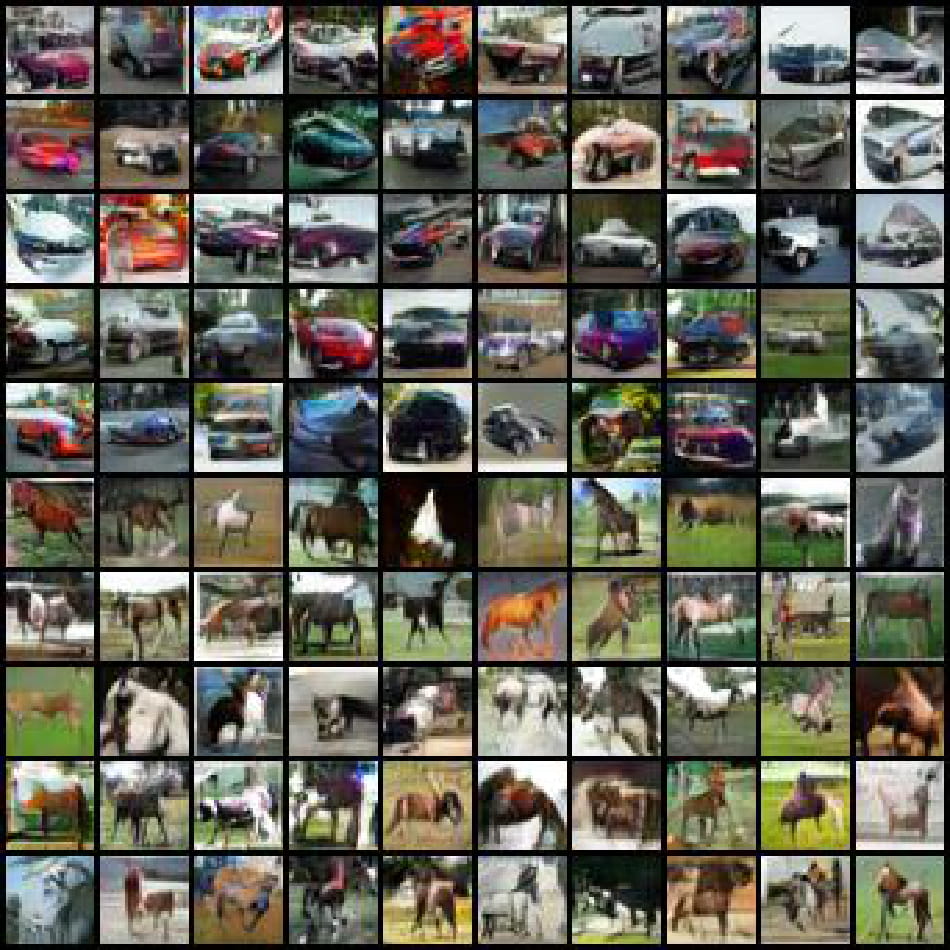}\label{fig:cifar10_visual_dcgan_SP_MH}}\quad
    		\subfloat[][DRE-F-SP+SIR]{\includegraphics[width=0.45\textwidth, height=7cm]{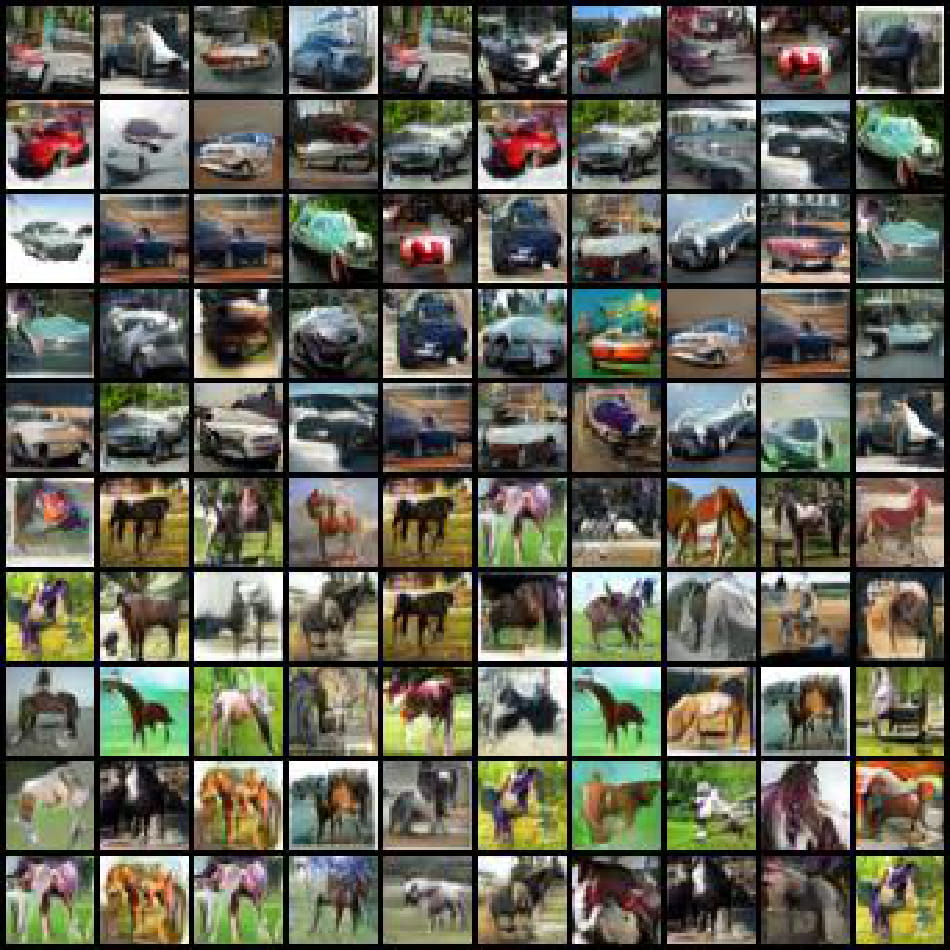}\label{fig:cifar10_visual_dcgan_SP_SIR}}
    		\caption{Fake CIFAR-10 images (car and horse) from DCGAN}.
    		\label{fig:cifar_visual_results_dcgan}
    	\end{figure*}
    	\begin{figure*}[h]
    		\centering
    		\subfloat[][No subsampling]{\includegraphics[width=0.45\textwidth, height=7cm]{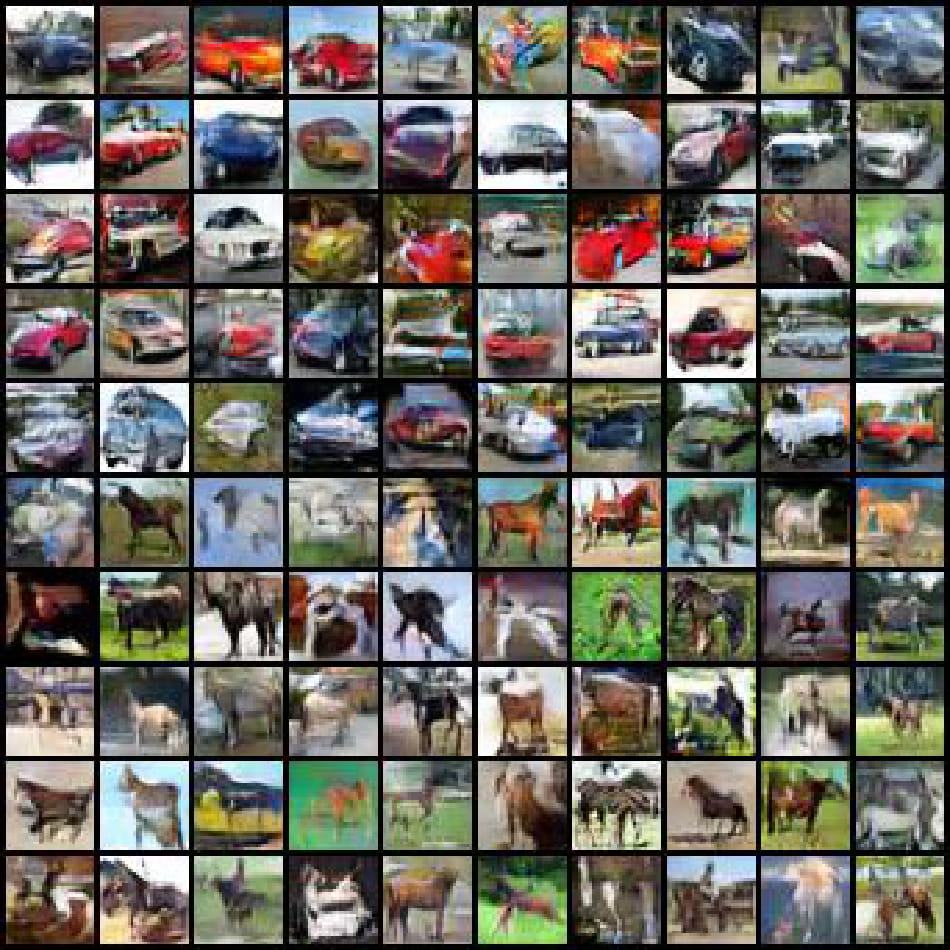}\label{fig:cifar10_visual_wgangp}}\quad
    		\subfloat[][DRS \cite{azadi2018discriminator}]{\includegraphics[width=0.45\textwidth, height=7cm]{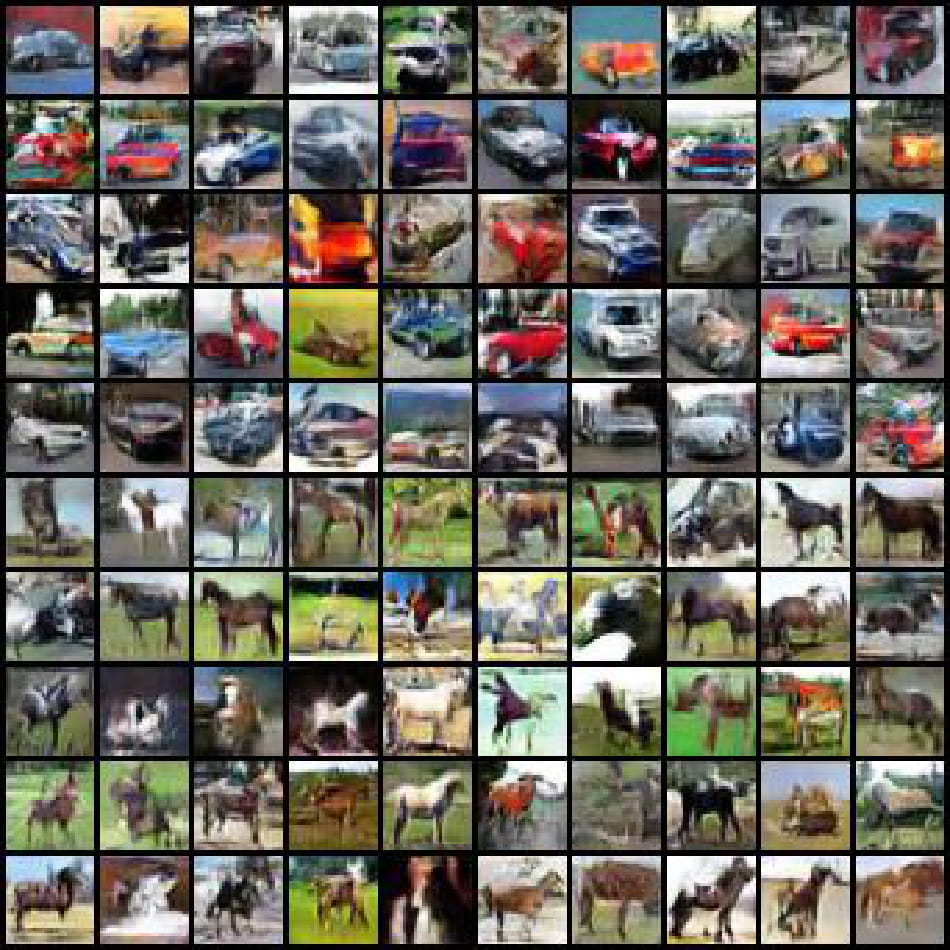}\label{fig:cifar10_visual_wgangp_DRS}}
    		\\
    		\subfloat[][MH-GAN \cite{turner2018metropolis}]{\includegraphics[width=0.45\textwidth, height=7cm]{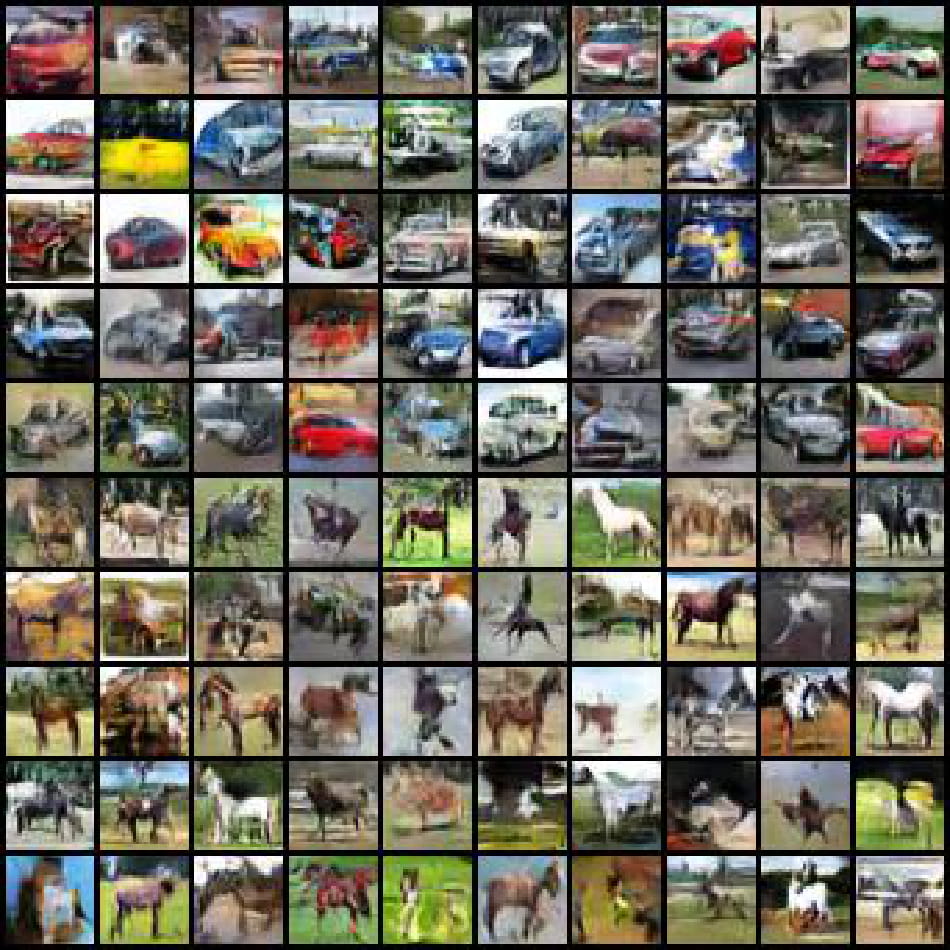}\label{fig:cifar10_visual_wgangp_MHGAN}}\quad
    		\subfloat[][DRE-F-SP+RS]{\includegraphics[width=0.45\textwidth, height=7cm]{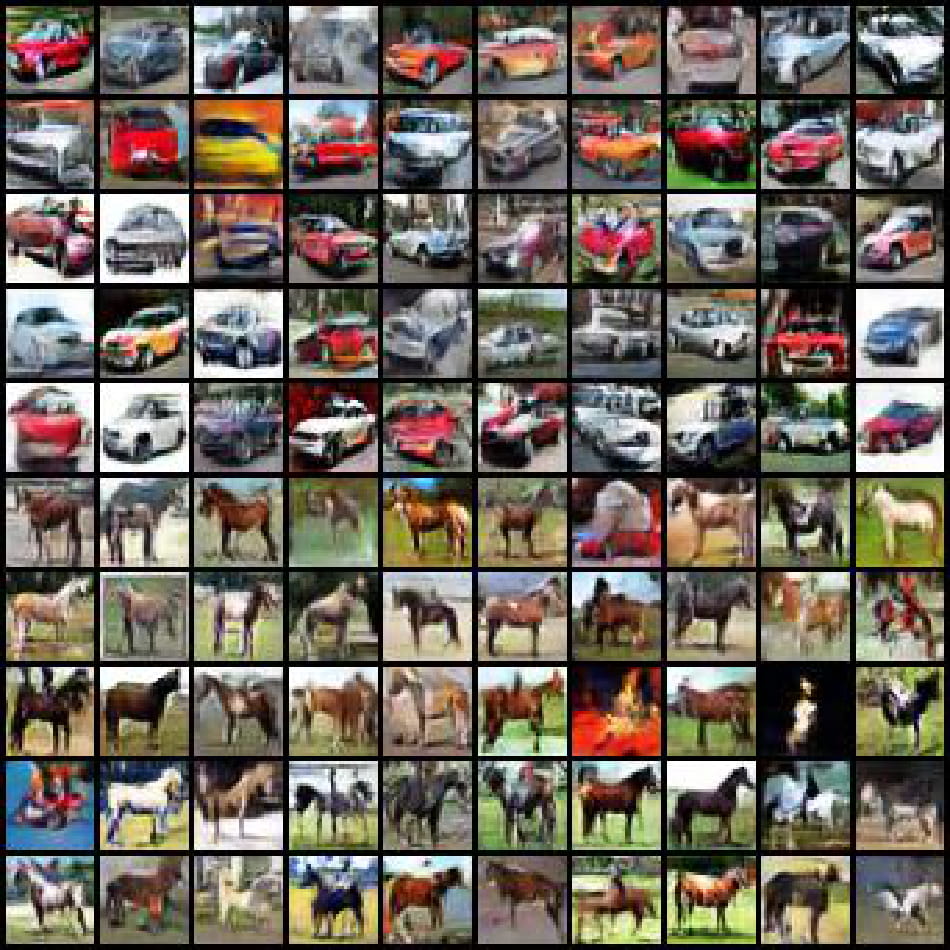}\label{fig:cifar10_visual_wgangp_SP_RS}}
    		\\
    		\subfloat[][DRE-F-SP+MH]{\includegraphics[width=0.45\textwidth, height=7cm]{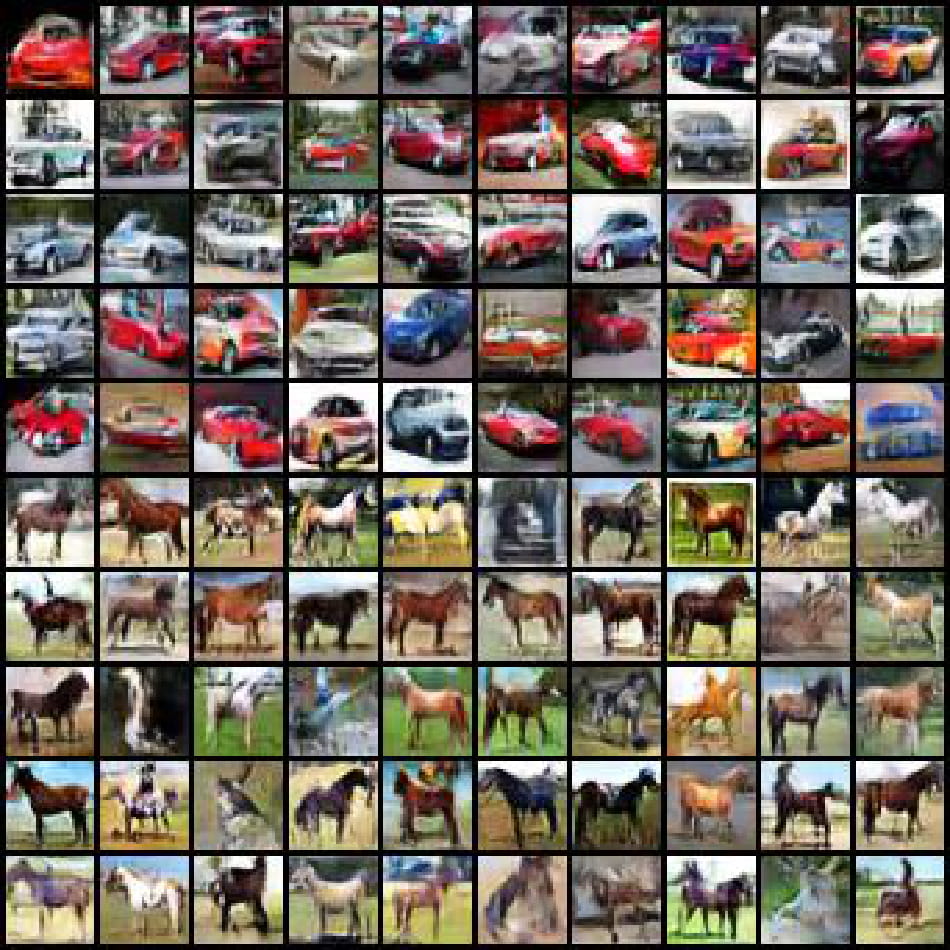}\label{fig:cifar10_visual_wgangp_SP_MH}}\quad
    		\subfloat[][DRE-F-SP+SIR]{\includegraphics[width=0.45\textwidth, height=7cm]{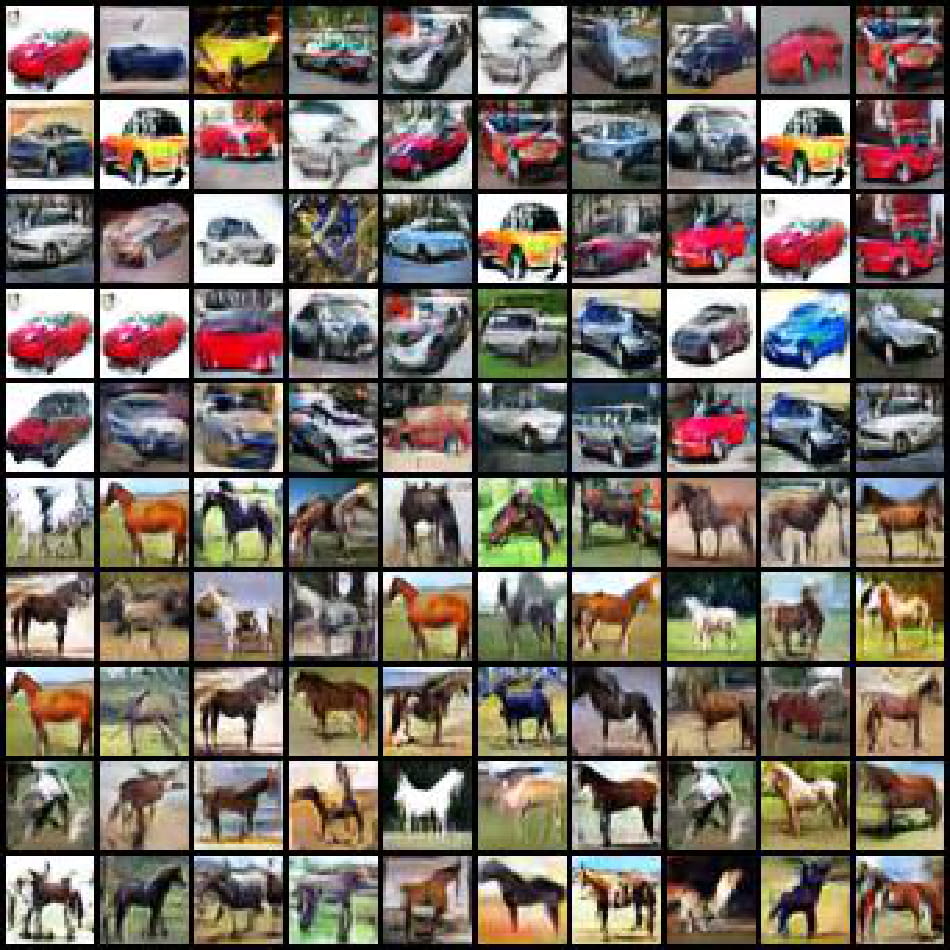}\label{fig:cifar10_visual_wgangp_SP_SIR}}
    		\caption{Fake CIFAR-10 images (car and horse) from WGAN-GP.}
    		\label{fig:cifar_visual_results_wgangp}
    	\end{figure*}
    	\begin{figure*}[h]
    		\centering
    		\subfloat[][No subsampling]{\includegraphics[width=0.45\textwidth, height=7cm]{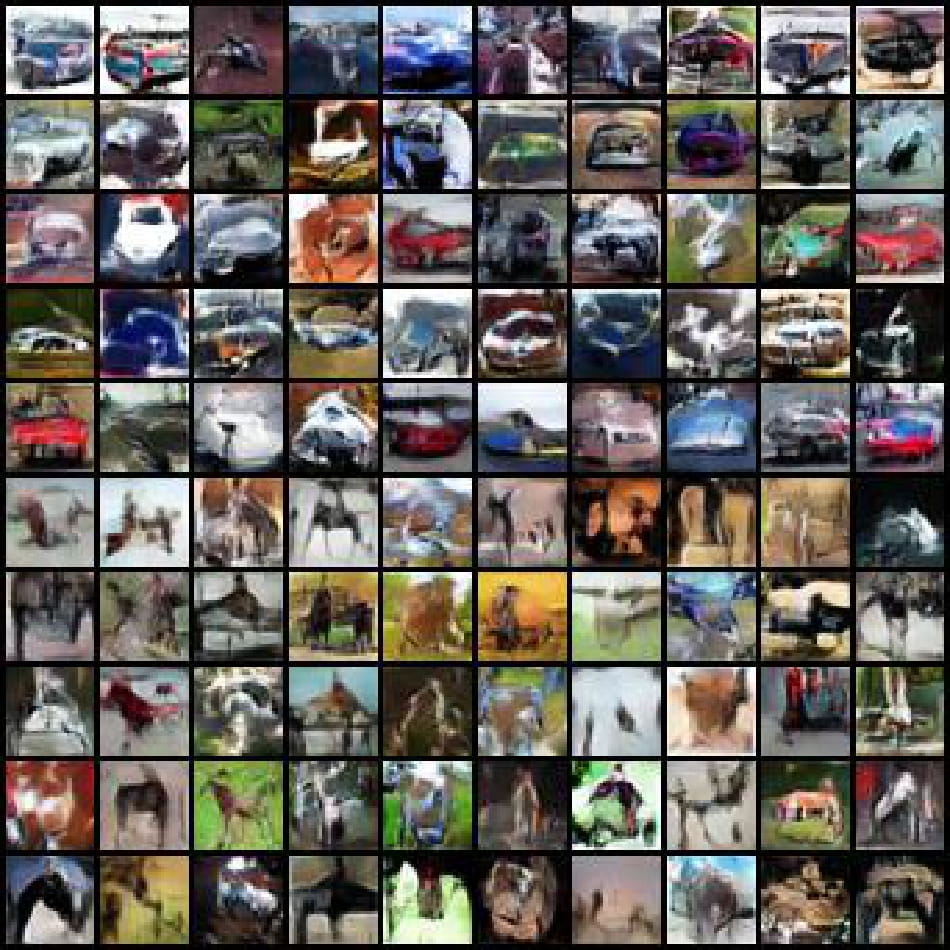}\label{fig:cifar10_visual_mmdgan}}\quad
    		\subfloat[][DRE-F-SP+RS]{\includegraphics[width=0.45\textwidth, height=7cm]{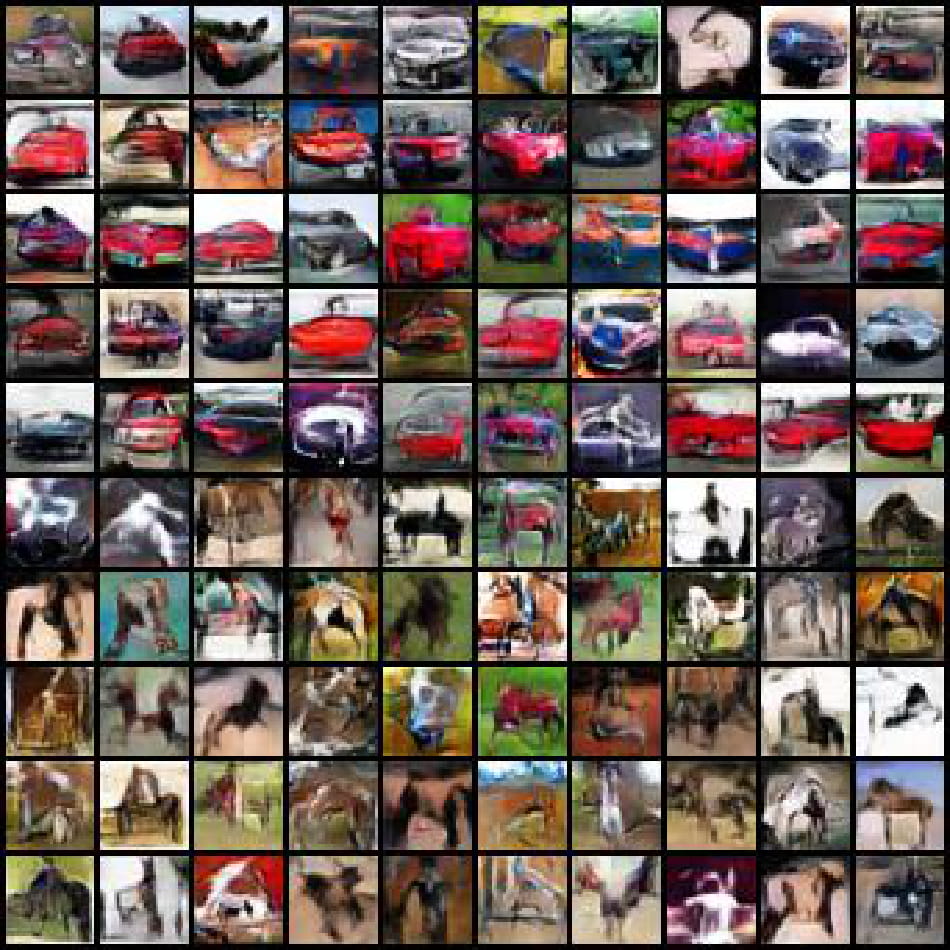}\label{fig:cifar10_visual_mmdgan_SP_RS}}
    		\\
    		\subfloat[][DRE-F-SP+MH]{\includegraphics[width=0.45\textwidth, height=7cm]{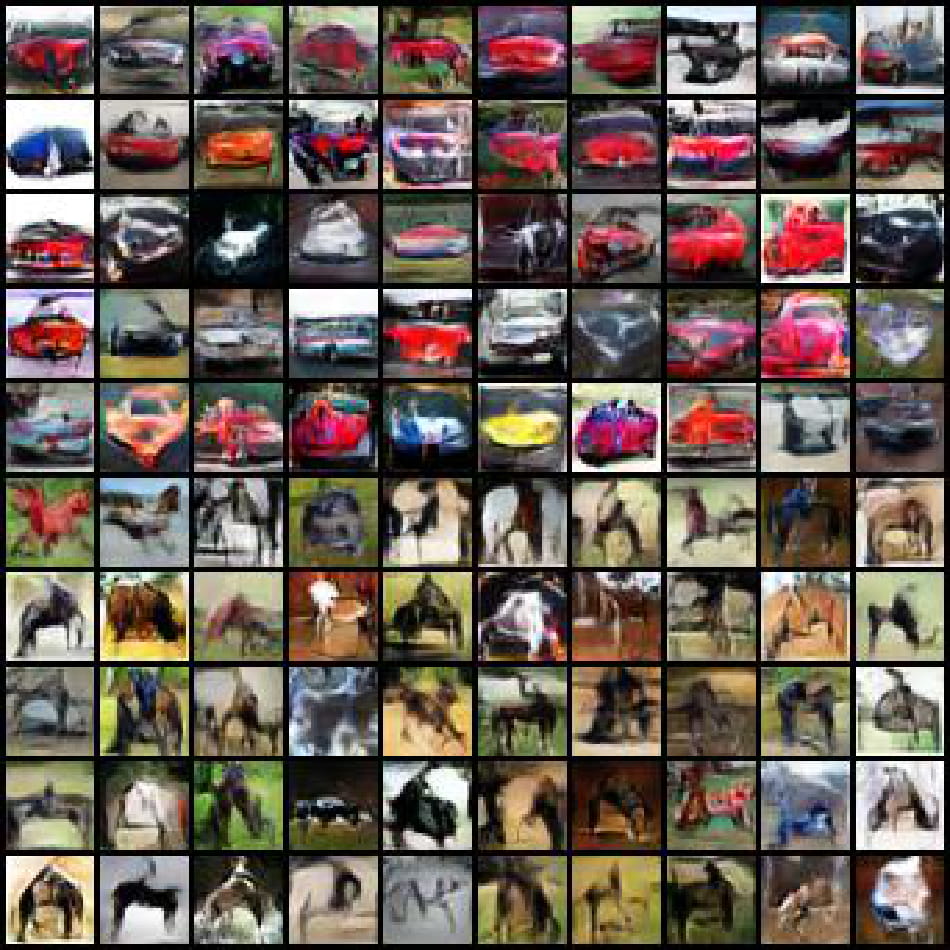}\label{fig:cifar10_visual_mmdgan_SP_MH}}\quad
    		\subfloat[][DRE-F-SP+SIR]{\includegraphics[width=0.45\textwidth, height=7cm]{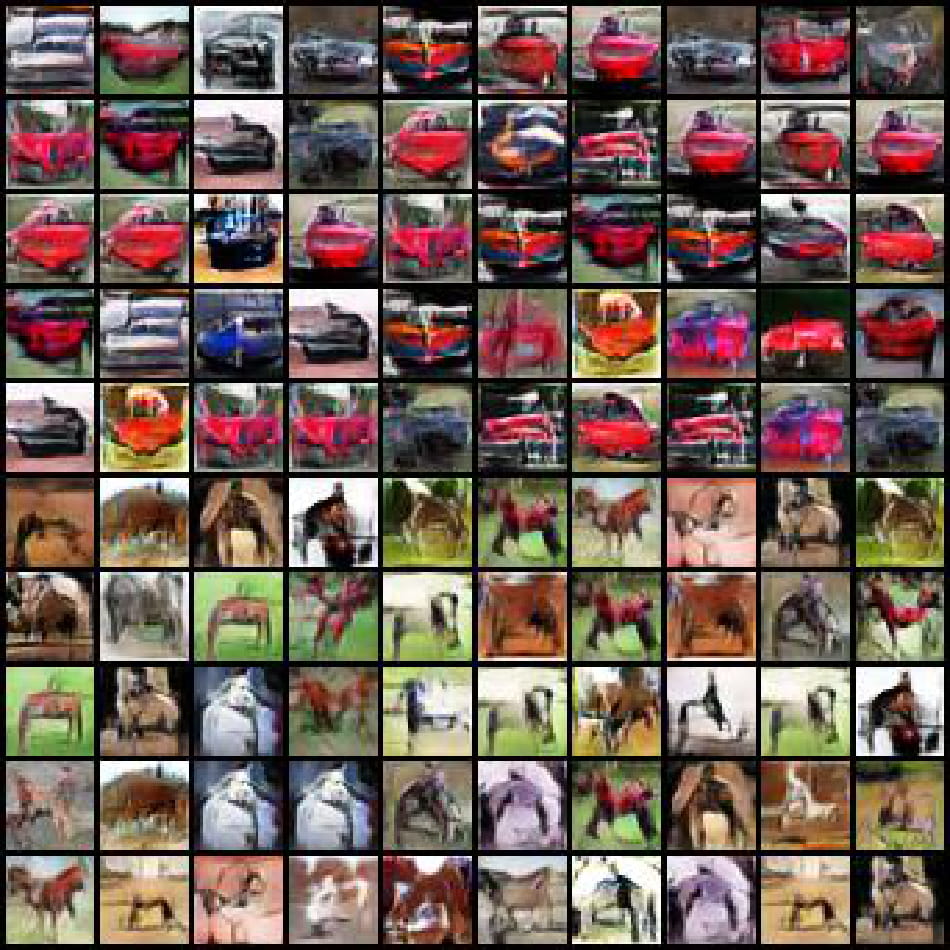}\label{fig:cifar10_visual_mmdgan_SP_SIR}}
    		\caption{Fake CIFAR-10 images (car and horse) from MMD-GAN.}
    		\label{fig:cifar_visual_results_mmdgan}
    	\end{figure*}

    	\subsection{CIFAR-10: An Extra Experiment} \label{appendix:cifar10_extra}
    	\cite{turner2018metropolis} compares the authors' MH-GAN with DRS \cite{azadi2018discriminator} and no subsampling when the DCGAN is trained for only 60 epochs. To show how the epoch of DCGAN training influences the subsampling results, we fix all settings (e.g., $\lambda=0$) but only vary the epoch of DCGAN training and visualize the results in Fig. \ref{fig:FID_and_IS_over_epochs}. In terms of IS and FID, our proposed DRE-F-SP+RS is consistently better than other methods at all epochs. 
    	
    	\begin{figure*}[h]
    		\centering
    		\subfloat[][IS versus the epoch of GAN training]{
    			\includegraphics[width=0.48\textwidth, height=6cm]{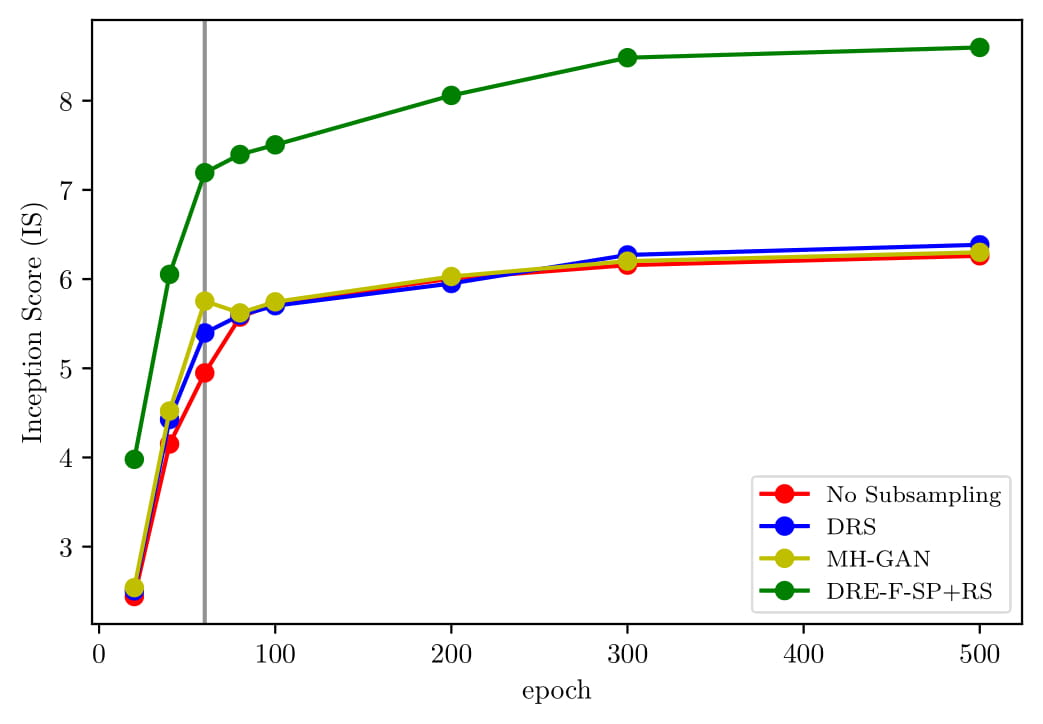}
    			\label{fig:IS_vs_epoch}}
    		\subfloat[][FID versus the epoch of GAN training]{
    			\includegraphics[width=0.48\textwidth, height=6cm]{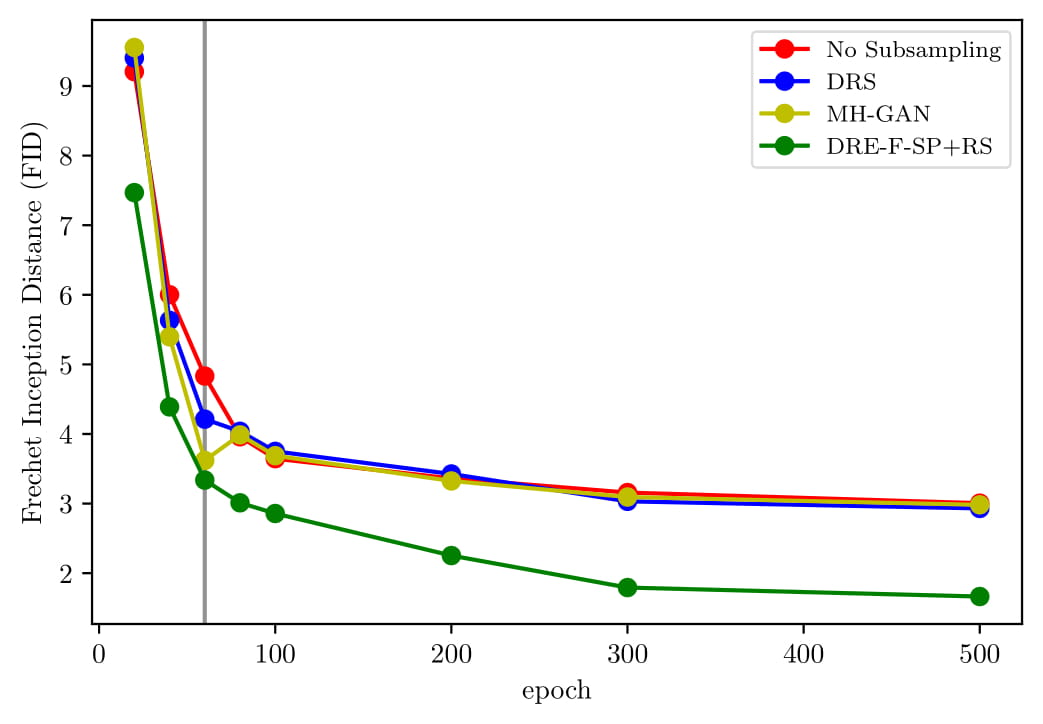}
    			\label{fig:FID_vs_epoch}}
    		\caption{Inception Score (IS) and Frechet Inception Distance (FID) versus the epoch of DCGAN training. The grey vertical line indicates the epoch 60. DRS \cite{azadi2018discriminator} and MH-GAN \cite{turner2018metropolis} perform well relative to the baseline of No Subsampling only at epoch 60 but the DCGAN is not well-trained at that epoch. It is more reasonable to train the DCGAN well first (e.g., 500 epochs) and then apply any subsampling method to further improve the image quality.}
    		\label{fig:FID_and_IS_over_epochs}
    	\end{figure*}

    	\subsection{Reduced MNIST} \label{appendix:reduced_mnist}	
    	
    	\subsubsection{Data}\label{appendix:mnist_data}
    	The MNIST dataset has 70,000 $28\times 28$ gray-scale handwritten digits from 10 classes (10 digits). The dataset is split into a training set of 60,000 images with 6000 per class and a test set of 10,000 images with 1000 per class. Since MNIST is a very simple dataset, to increase the difficulty for both the GAN training and the subsampling, we randomly select 5000 images from the original training set to form a reduced training set.

    	\subsubsection{Network Architectures}\label{appendix:mnist_nets}
    	We implement DCGAN with the generator and the discriminator shown in Table \ref{tab:mnist_gan_architecture}. The ResNet-34 for feature extraction and the MLP-5 for DRE are identical to the CIFAR-10 experiment except a few setups (e.g., the number of input channels of the first convolutional layer of ResNet-34) and please see our codes for more details. 
    	
    	\begin{table}[h]%
    		\centering
    		\footnotesize
    		\caption{Network architectures for the generator and the discriminator of DCGAN in the experiment on the MNIST. The slopes of all LeakyReLU are set to 0.2. We denote stride and padding by s and p respectively.}%
    		\subfloat[][Generator]{\begin{tabular}{c}
    				\toprule
    				$z\in \mathbbm{R}^{128}\sim N(0,I)$ \\ \hline
    				dense$\rightarrow 7\times 7\times 256$ \\ \hline
    				deconv, $4\times 4$, $\text{s}=2$, $\text{p}=1$, $128$; BN; ReLU \\ \hline
    				deconv, $4\times 4$, $\text{s}=2$, $\text{p}=1$, $64$; BN; ReLU \\ \hline
    				conv, $3\times 3$, $\text{s}=1$, $\text{p}=1$, $3$; Tanh \\
    				\bottomrule
    		\end{tabular}}%
    		\qquad
    		\subfloat[][Discriminator]{\begin{tabular}{c}
    				\toprule
    				RGB image $x\in\mathbbm{R}^{1\times 28\times 28}$ \\ \hline
    				conv, $3\times 3$, $\text{s}=1$, $\text{p}=1$, 64; LeakyReLU \\ \hline
    				conv, $4\times 4$, $\text{s}=2$, $\text{p}=1$, 64; LeakyReLU \\ \hline
    				conv, $3\times 3$, $\text{s}=1$, $\text{p}=1$, 128; LeakyReLU \\ \hline
    				conv, $4\times 4$, $\text{s}=2$, $\text{p}=1$, 128; LeakyReLU \\ \hline
    				conv, $3\times 3$, $\text{s}=1$, $\text{p}=1$, 256; LeakyReLU \\ \hline
    				fc$\rightarrow 1$; Sigmoid\\
    				\bottomrule
    		\end{tabular}}
    		\label{tab:mnist_gan_architecture}%
    	\end{table}

    	\subsubsection{Training Setups}\label{appendix:mnist_training_setups}
    	In the MNIST setting, the DCGAN is trained on the reduced training set with the Adam optimizer, constant learning rate $10^{-4}$, batch size 256, and 500 epochs. The modified ResNet-34 for feature extraction is trained on the reduced training set for 200 epochs with the SGD optimizer, initial learning rate 0.01 (decay at epoch 100 with factor 0.1), weight decay $10^{-4}$, and batch size 512. To implement DRE-F-SP, we use the MLP-5 as in the CIFAR-10 setting, and it is trained with the Adam optimizer, initial learning rate $10^{-4}$ (decay at the epoch 400 with factor 0.1), batch size 512, and 500 epochs. The optimal hyperparameter is selected based on Table \ref{tab:mnist_hyperparameter_selection}
    	
    	\begin{table}[h]
    		\centering
    		\footnotesize
    		\caption{Hyperparameter selection for MNIST. Two-Sample Kolmogorov-Smirnov test statistic is shown for each $\lambda$.}
    		\begin{tabular}{ccccc}
    			\toprule
    			& \multicolumn{4}{c}{DCGAN} \\
    			\cline{2-5}
    			$\lambda$ & 0     & 0.001 & 0.01    & 0.1 \\ \midrule
    			KS Stat. & 1.203E-01 & 1.182E-01 & 1.168E-01 & \textbf{1.115E-01} \\
    			\bottomrule
    		\end{tabular}%
    		\label{tab:mnist_hyperparameter_selection}%
    	\end{table}%
    	
    	\subsubsection{Performance Measures} \label{appendix:mnist_performance_measures}
    	To evaluate the quality of fake images by IS and FID, we train the Inception-V3 on all 60,000 MNIST training images (not the reduced training set). FID is computed based on the final average pooling features from the pre-trained Inception-V3. We compute the FID between 50,000 fake images and all 60,000 training images.
    	
    	\subsubsection{Quantitative Results} \label{appendix:mnist_quantitative_results}
    	Table \ref{tab:results_reduced_mnist_main} shows the results of the experiment and demonstrates that our approaches significantly outperform two existing subsampling methods.
    	
    	\begin{table}[h]
    		\centering
    		\footnotesize
    		\caption{Average quality of 50,000 fake MNIST images from different subsampling methods over three repetitions. We draw 50,000 fake images by each method on which we compute the IS and FID. We repeat this sampling for three times and report the average IS and FID. Higher IS and lower FID are better. A grid search is conducted for DRE-F-SP to select the hyperparameter, and the results under the optimal $\lambda^*$ are shown in this table. We include the IS and FID of all 60,000 training data and 10,000 test data as a reference.}
    		\begin{tabular}{lll}
    			\toprule
    			Method & IS (mean$\pm$std)    & FID (mean$\pm$std) \\
    			\midrule
    			\textbf{- Real Data -} &       &  \\
    			All 60,000 Training Data & 9.984 & - \\
    			10,000 Test Data & 9.462 & 0.134 \\
    			\midrule
    			\textbf{- DCGAN -} &       &  \\
    			No Subsampling & $7.791\pm 0.004$ & $1.723\pm 0.004$ \\
    			DRS \cite{azadi2018discriminator} & $8.032\pm 0.007$ & $1.658\pm 0.007$ \\
    			MH-GAN \cite{turner2018metropolis} & $8.461\pm 0.001$ & $1.269\pm 0.007$ \\
    			DRE-F-SP+RS ($\lambda^*=0.1$) & $\bm{9.676\pm 0.005}$ & $\bm{0.374\pm 0.007}$ \\
    			DRE-F-SP+MH ($\lambda^*=0.1$) & $\bm{9.677\pm 0.006}$ & $\bm{0.368\pm 0.008}$ \\
    			DRE-F-SP+SIR ($\lambda^*=0.1$) & $\bm{9.675\pm 0.003}$ & $\bm{0.380\pm 0.010}$ \\
    			\bottomrule
    		\end{tabular}%
    		\label{tab:results_reduced_mnist_main}%
    	\end{table}%
    	
    	\subsubsection{Visual Results} \label{appendix:mnist_visual_results}
    	We show some example MNIST images from different subsampling methods in Fig. \ref{fig:mnist_visual_results_dcgan1}. For each method, we draw images from all 10 classes with 30 images per class. 
    	
    	\begin{figure*}[h]
    		\centering
    		\subfloat[][No subsampling]{\includegraphics[width=0.45\textwidth, height=20cm]{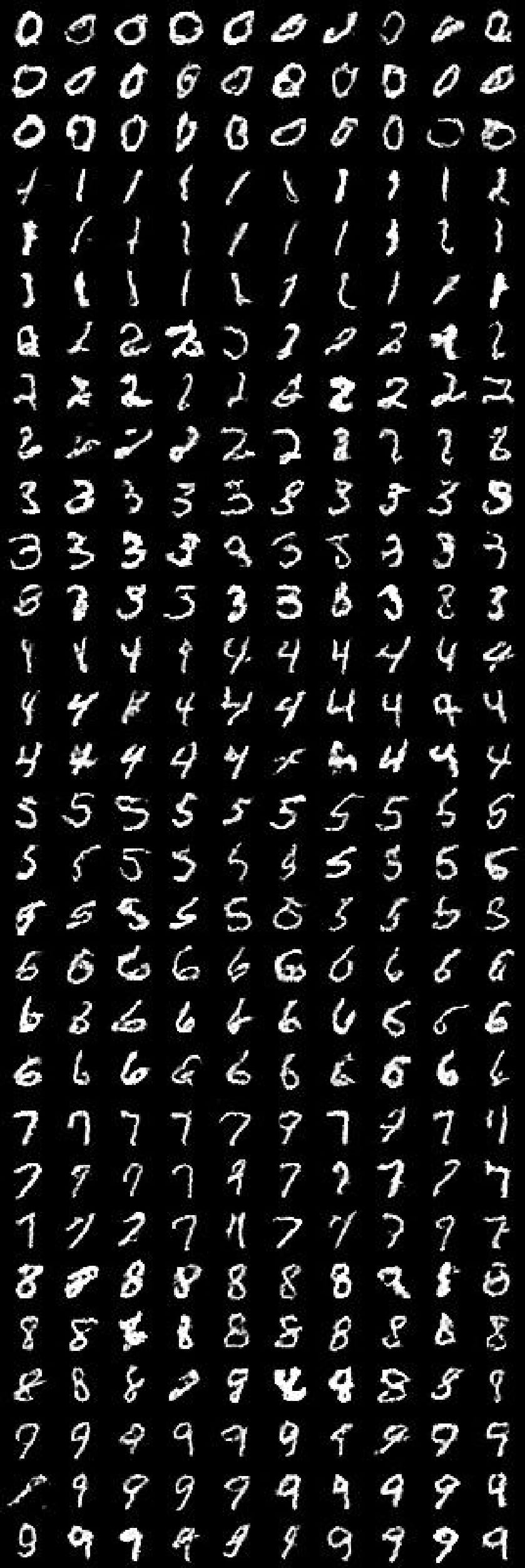}\label{fig:mnist_visual_dcgan}}\quad
    		\subfloat[][DRS \cite{azadi2018discriminator}]{\includegraphics[width=0.45\textwidth, height=20cm]{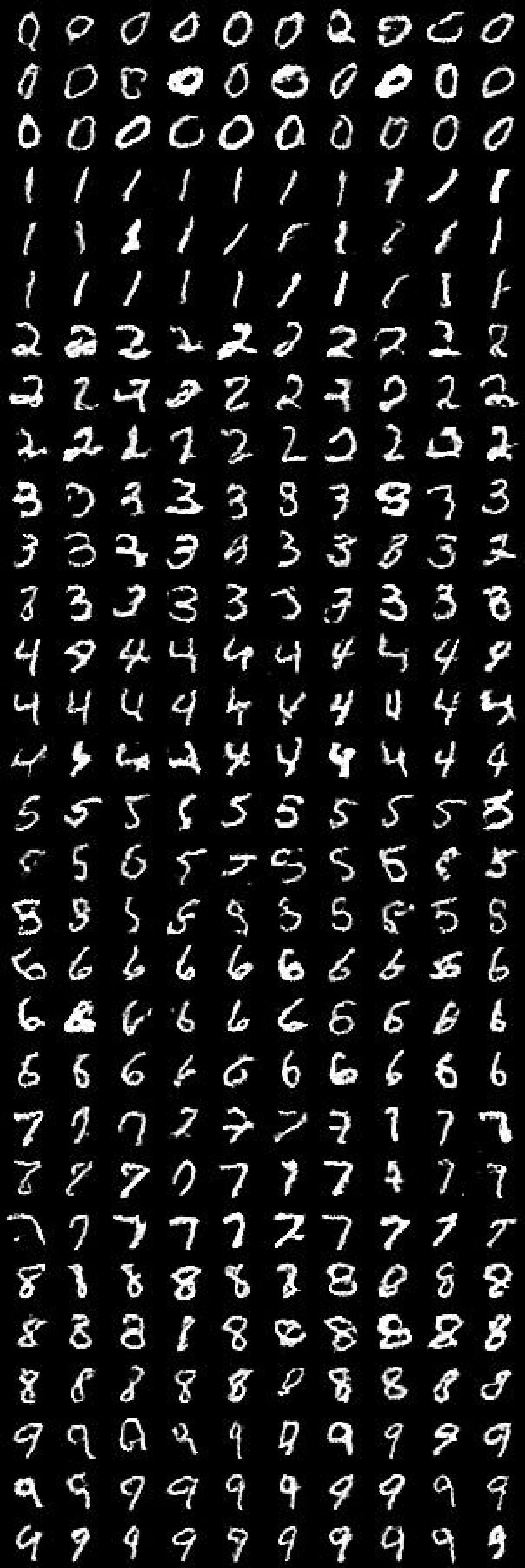}\label{fig:mnist_visual_dcgan_drs}}
    		\caption{Fake MNIST images from DCGAN: Part 1}.
    		\label{fig:mnist_visual_results_dcgan1}
    	\end{figure*}
    	
    	\begin{figure*}[h]
    		\centering
    		\ContinuedFloat
    		
    		\subfloat[][MH-GAN \cite{turner2018metropolis}]{\includegraphics[width=0.45\textwidth, height=20cm]{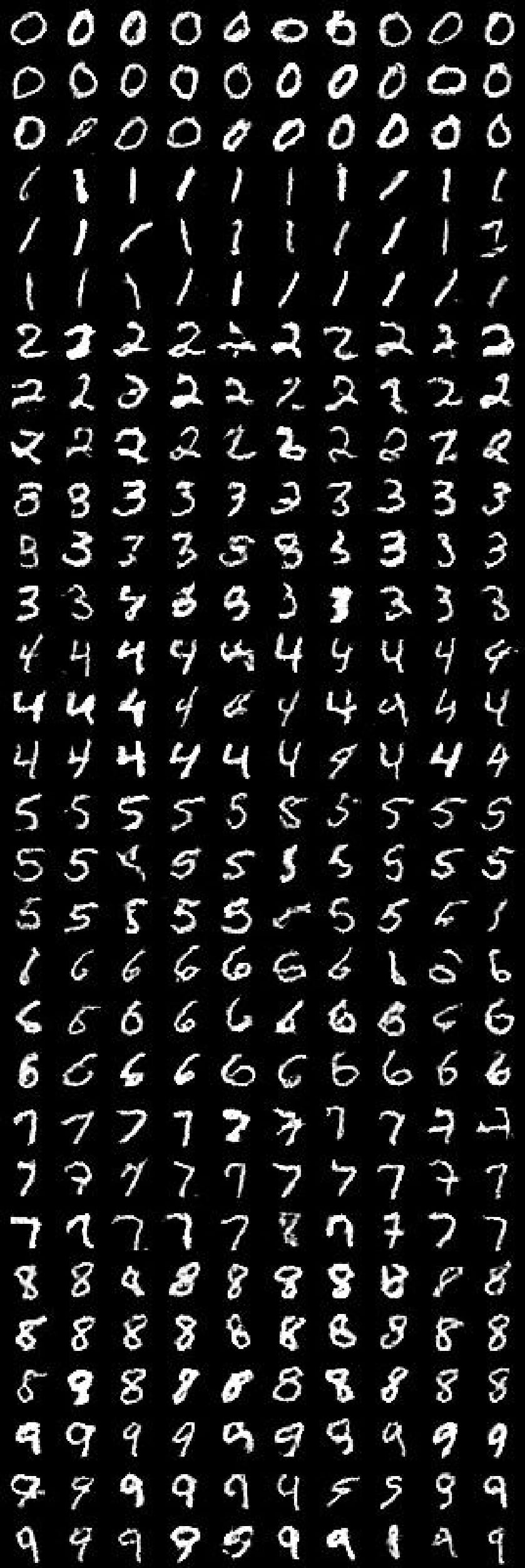}\label{fig:mnist_visual_mhgan}}\quad
    		\subfloat[][DRE-F-SP+RS]{\includegraphics[width=0.45\textwidth, height=20cm]{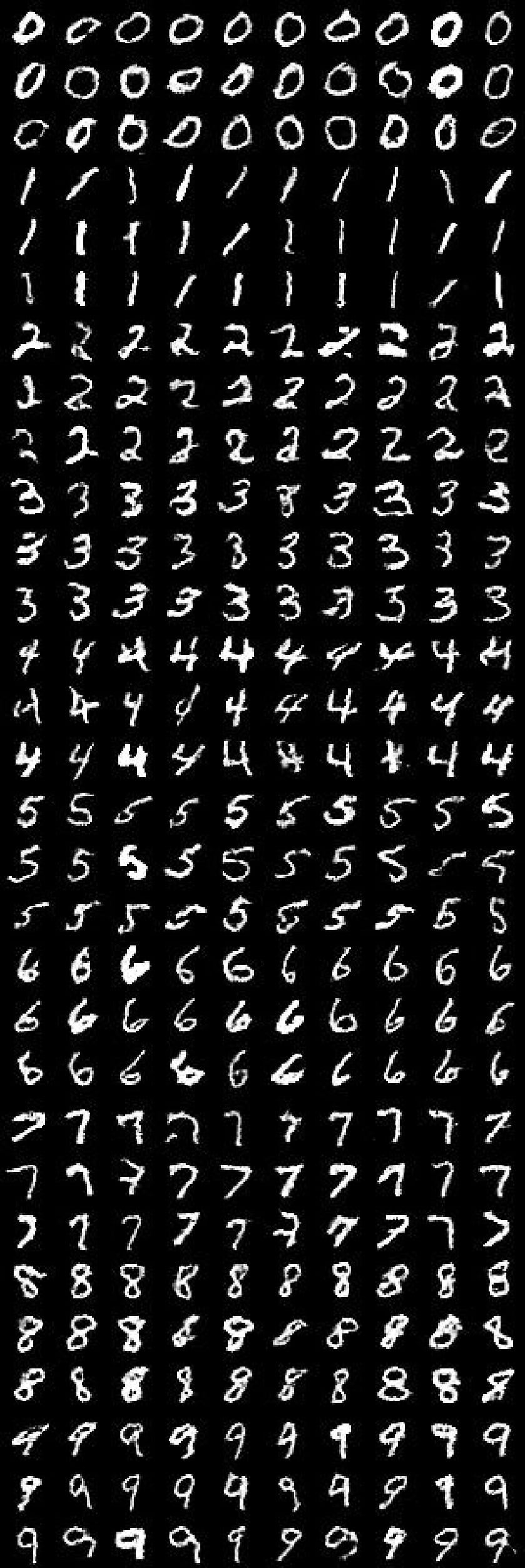}\label{fig:mnist_visual_dcgan_DRE-F-SP+RS}}
    		\caption{Fake MNIST images from DCGAN: Part 2}.
    		\label{fig:mnist_visual_results_dcgan2}
    	\end{figure*}
    	
    	\begin{figure*}[h]
    		\centering
    		\ContinuedFloat
    		
    		\subfloat[][DRE-F-SP+MH]{\includegraphics[width=0.45\textwidth, height=20cm]{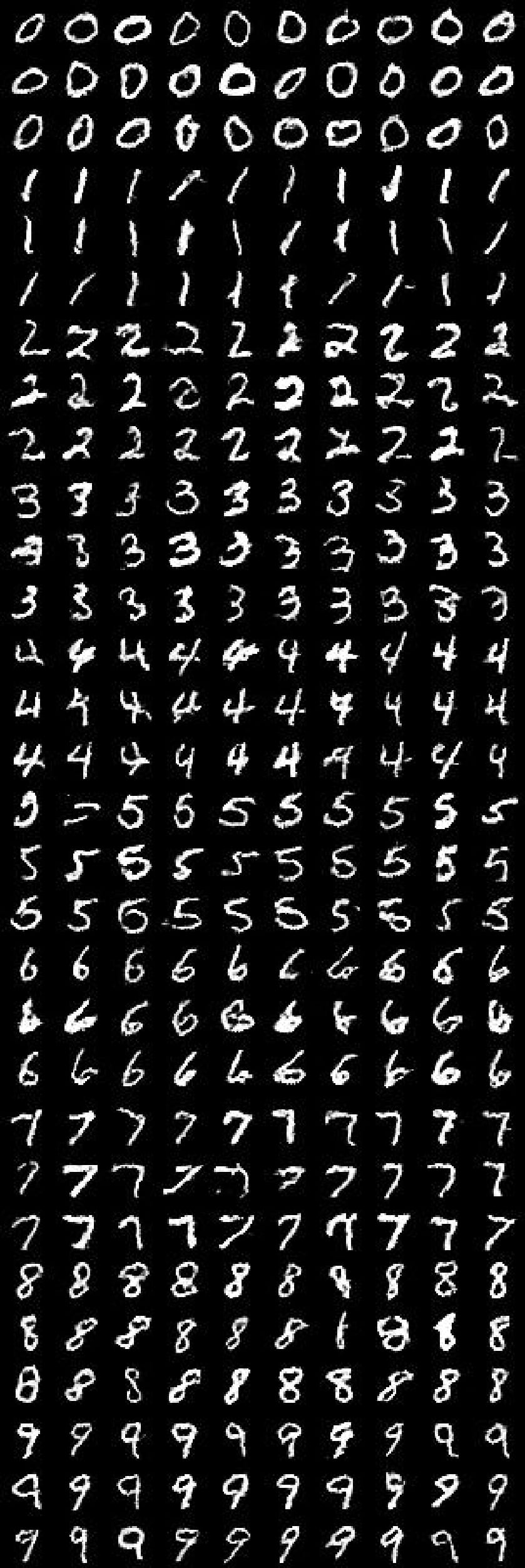}\label{fig:mnist_visual_DRE-F-SP+MH}}\quad
    		\subfloat[][DRE-F-SP+SIR]{\includegraphics[width=0.45\textwidth, height=20cm]{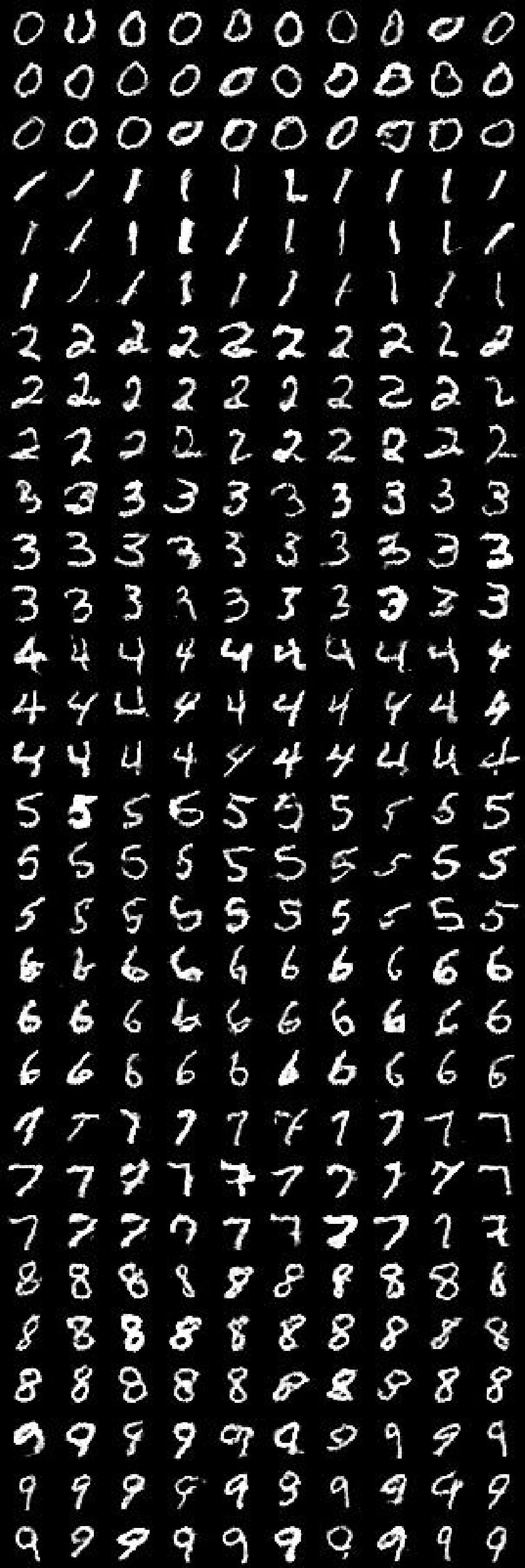}\label{fig:mnist_visual_dcgan_DRE-F-SP+SIR}}
    		\caption{Fake MNIST images from DCGAN: Part 3}.
    		\label{fig:mnist_visual_results_dcgan3}
    	\end{figure*}
    	
    	\subsection{CelebA Dataset} \label{appendix:celeba}	
    	\subsubsection{Data}\label{appendix:celeba_data}
    	The CelebA dataset \cite{liu2015faceattributes} has 202,599 $218\times 178$ RGB face images. Each image has 40 binary attributes. We create 6 classes in terms of 3 binary attributes (``Wearing\_Lipstick", ``Smiling", ``Mouth\_Sligthly\_Open"); some minority classes are merged together to balance the dataset. Then the dataset is randomly split into a training set of 192,599 images and a test set of 10,000 images. All images are resized to $64\times 64$.
    	
    	\subsubsection{Network Architectures}\label{appendix:celeba_nets}
    	We implement SNGAN \cite{miyato2018spectral} based on codes from \url{https://github.com/christiancosgrove/pytorch-spectral-normalization-gan} and \url{https://github.com/pfnet-research/sngan_projection}. The ResNet-34 for feature extraction and the MLP-5 for DRE are identical to the CIFAR-10 experiment except a few setups (e.g., the stride size of the first residual block of ResNet-34) and please see our codes for more details. 
    	
    	\subsubsection{Training Setups}\label{appendix:celeba_training_setups}
    	The SNGAN is trained on the training set with the Adam optimizer, constant learning rate $10^{-4}$ for the generator, constant learning rate $4\times 10^{-4}$ for the discriminator, batch size 256, and 100 epochs. The modified ResNet-34 for feature extraction is trained on the training set for 100 epochs with the SGD optimizer, initial learning rate 0.001 (decay at epoch 30 and 70 with factor 0.1), weight decay $10^{-4}$, and batch size 256. The MLP-5 is trained on the training set with the Adam optimizer, initial learning rate $10^{-4}$ (decay at epoch 30 and 70), batch size 512, and 100 epochs. The optimal hyperparameter is selected based on Table \ref{tab:celeba_hyperparameter_selection}.
    	
    	\begin{table}[h]
    		\centering
    		\footnotesize
    		\caption{Hyperparameter selection for CelebA. Two-Sample Kolmogorov-Smirnov test statistic is shown for each $\lambda$.}
    		\begin{tabular}{cccccc}
    			\toprule
    			& \multicolumn{5}{c}{SNGAN} \\
    			\cline{2-6}
    			$\lambda$ & 0     & \textbf{0.005} & 0.01 &0.05   & 0.1 \\ \midrule
    			KS Stat. & 	0.2643 & \textbf{0.2622} & 0.2646 & 0.2643 & 0.2645 \\
    			\bottomrule
    		\end{tabular}%
    		\label{tab:celeba_hyperparameter_selection}%
    	\end{table}%

    	\subsubsection{Performance Measures} \label{appendix:celeba_performance_measures}
    	To evaluate the quality of fake images by IS and FID, we use the Inception-V3 which is pre-trained on the ImageNet \cite{imagenet_cvpr09}. FID is computed based on the final average pooling features from the pre-trained Inception-V3. The computation is based on codes from \url{https://github.com/sbarratt/inception-score-pytorch} and \url{https://github.com/mseitzer/pytorch-fid}. We compute the FID between 50,000 fake images and all 192,599 training images. 
    	
    	\subsubsection{Quantitative Results} \label{appendix:celeba_quantitative_results}
    	Table \ref{tab:results_celeba_main} shows the results of the experiment and demonstrates that our approaches significantly outperform two existing subsampling methods.
    	
    	\begin{table}[h]
    		\centering
    		\footnotesize
    		\caption{Average quality of 50,000 fake CelebA images from different subsampling methods over three repetitions. We draw 50,000 fake images by each method on which we compute the IS and FID. We repeat this sampling for three times and report the average IS and FID. Higher IS and lower FID are better. A grid search is conducted for DRE-F-SP to select the hyperparameter, and the results under the optimal $\lambda^*$ are shown in this table. We include the IS and FID of all training data and 10,000 test data as a reference.}
    		\begin{tabular}{lll}
    			\toprule
    			Method & IS (mean$\pm$std)    & FID (mean$\pm$std) \\
    			\midrule
    			\textbf{- Real Data -} &       &  \\
    			All Training Data & 3.321 & - \\
    			10,000 Test Data & 3.319 & 1.572 \\
    			\midrule
    			\textbf{- SNGAN -} &       &  \\
    			No Subsampling & $2.781\pm 0.006$ & $6.607\pm 0.013$ \\
    			DRS \cite{azadi2018discriminator} & $2.796\pm 0.006$ & $6.490\pm 0.015$ \\
    			MH-GAN \cite{turner2018metropolis} & $2.766\pm 0.005$ & $6.649\pm 0.019$ \\
    			DRE-F-SP+RS ($\lambda^*=0.005$) & $\bm{2.865\pm 0.009}$ & $\bm{6.088\pm 0.007}$ \\
    			DRE-F-SP+MH ($\lambda^*=0.005$) & $\bm{2.870\pm 0.009}$ & $\bm{6.032\pm 0.019}$ \\
    			DRE-F-SP+SIR ($\lambda^*=0.005$) & $\bm{2.877\pm 0.021}$ & $\bm{6.116\pm 0.021}$ \\
    			\bottomrule
    		\end{tabular}%
    		\label{tab:results_celeba_main}%
    	\end{table}%
    	
    	\subsubsection{Visual Results} \label{appendix:celeba_visual_results}
    	We show some example CelebA images from different subsampling methods in Fig. \ref{fig:celeba_visual_results_sngan}. We manually check each generated image and outline it in yellow if we find obvious distortion or asymmetry in its face area (we ignore the background). We can see our proposed methods have fewer marked images.
    	
    	\begin{figure*}[h]
    		\centering
    		\subfloat[][No subsampling (28 yellow squares)]{\includegraphics[width=0.45\textwidth, height=7.5cm]{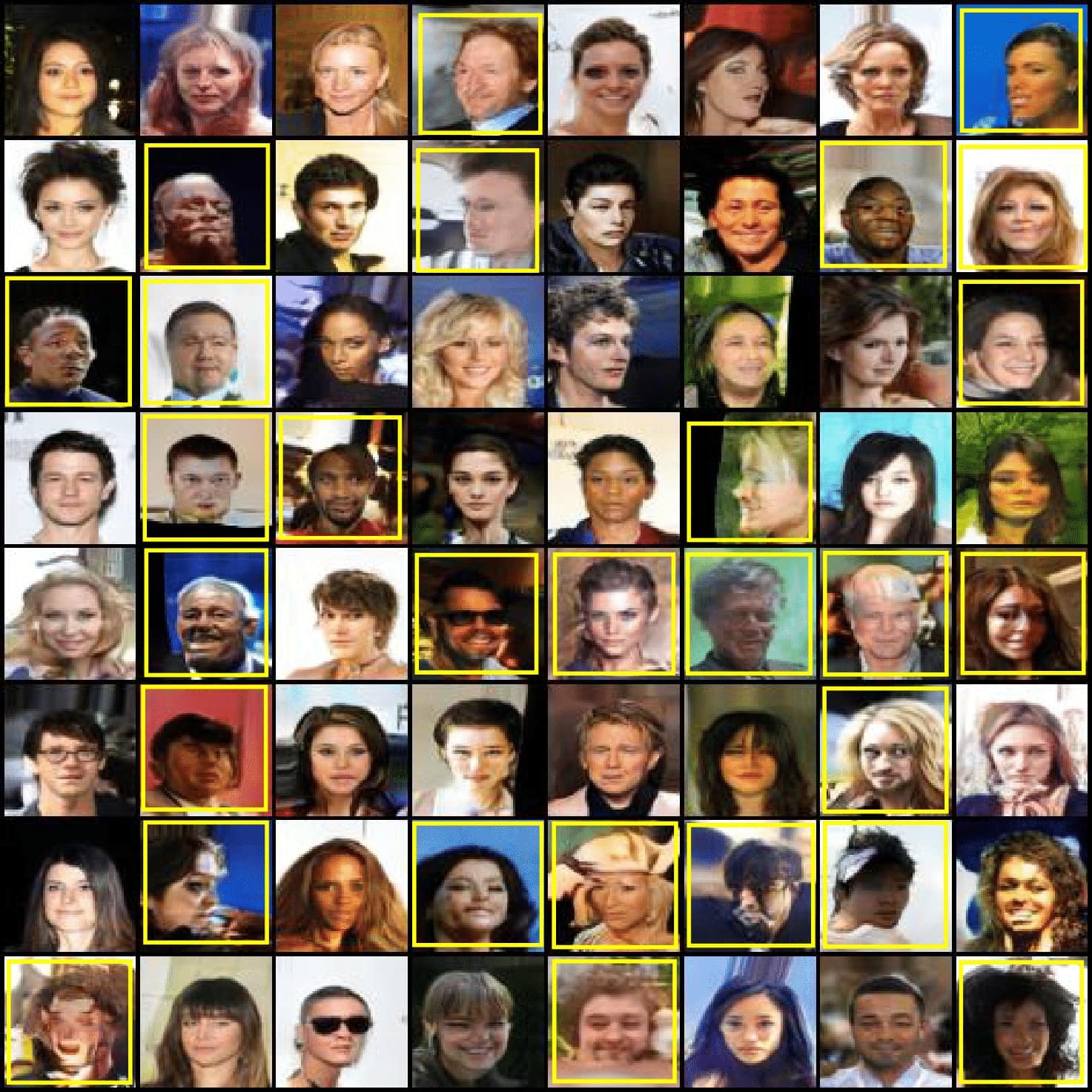}\label{fig:celeba_visual_sngan}}\quad
    		\subfloat[][DRS \cite{azadi2018discriminator} (25 yellow squares)]{\includegraphics[width=0.45\textwidth, height=7.5cm]{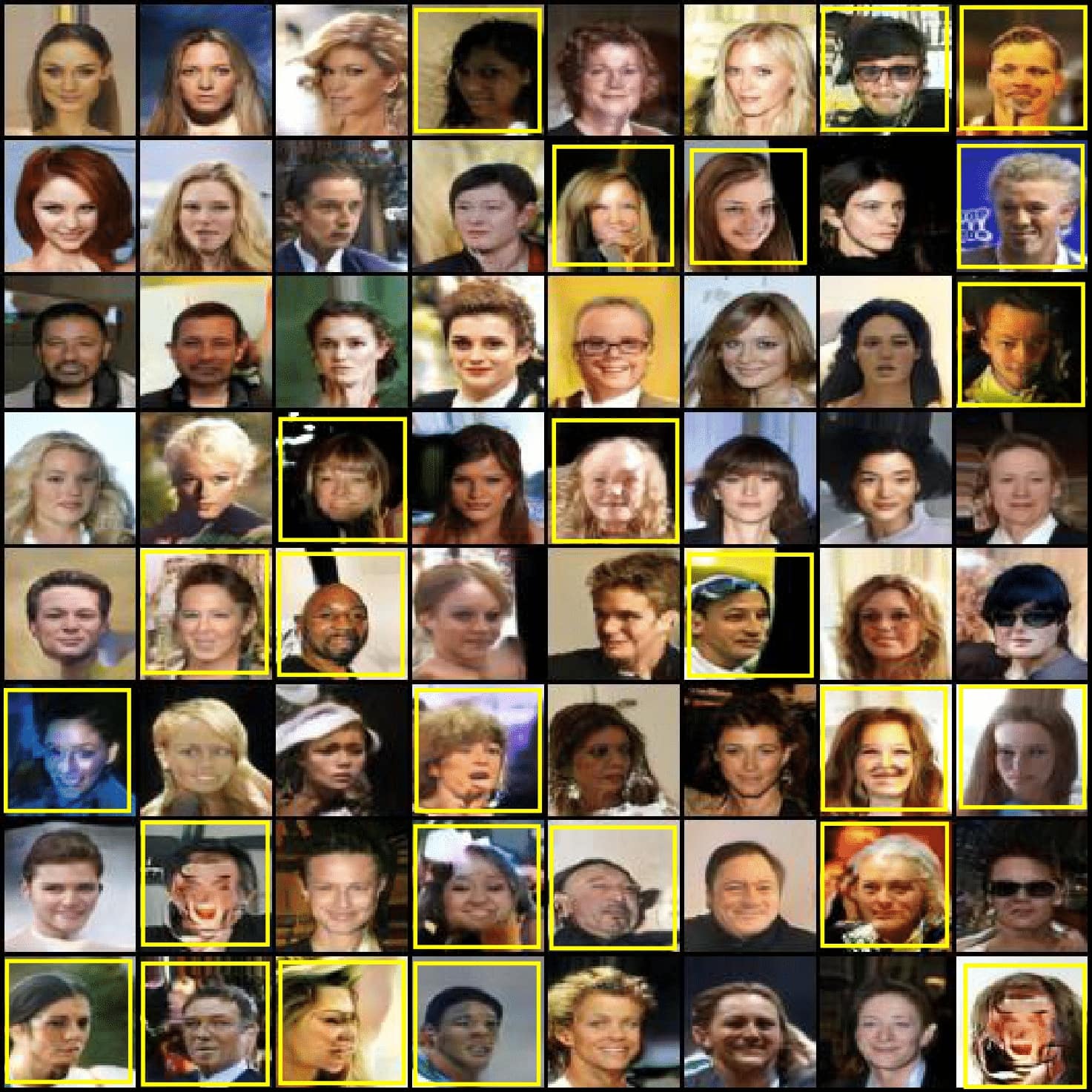}\label{fig:celeba_visual_sngan_DRS}}
    		\\
    		\subfloat[][MH-GAN \cite{turner2018metropolis} (22 yellow squares)]{\includegraphics[width=0.45\textwidth, height=7.5cm]{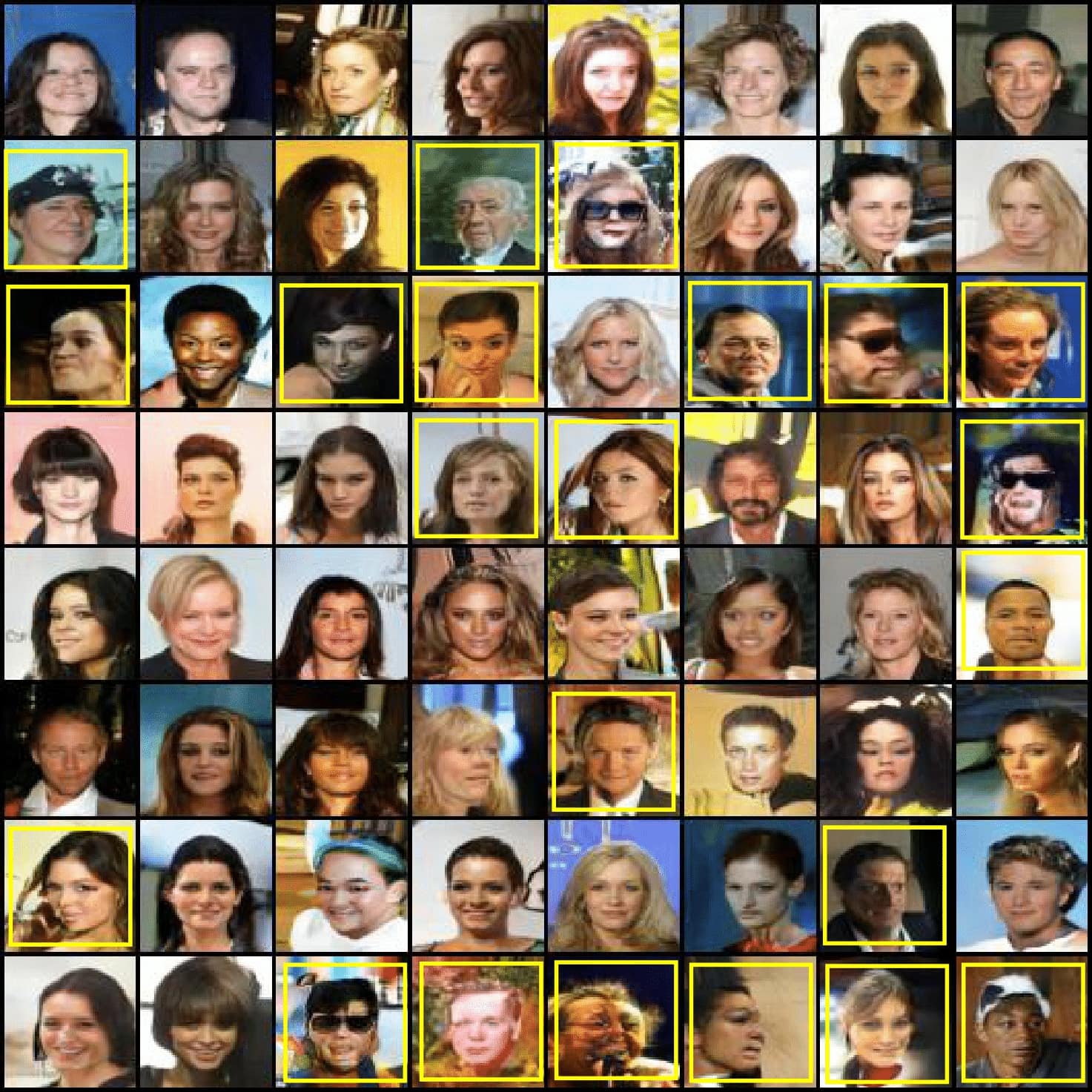}\label{fig:celeba_visual_sngan_MHGAN}}\quad
    		\subfloat[][DRE-F-SP+RS (13 yellow squares)]{\includegraphics[width=0.45\textwidth, height=7.5cm]{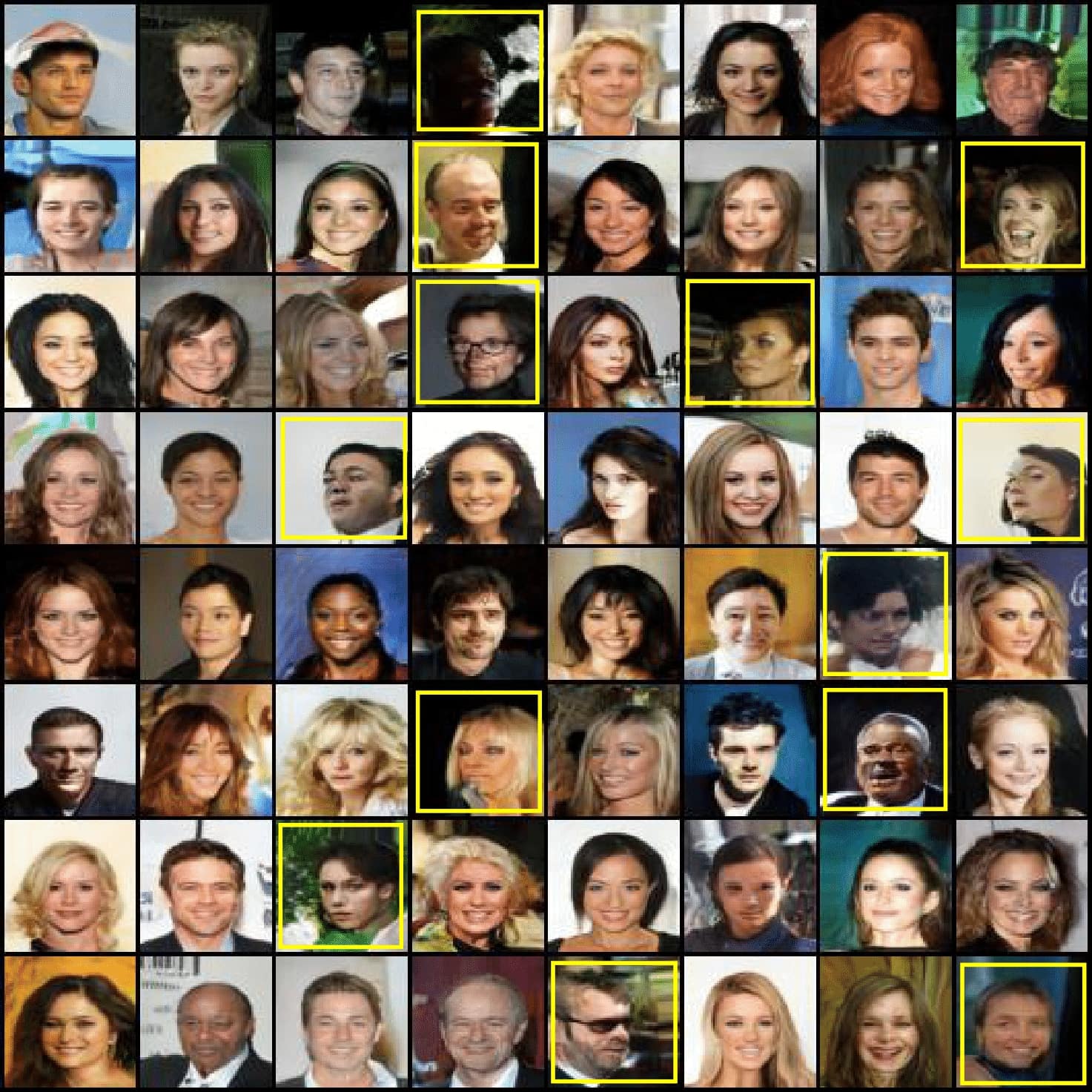}\label{fig:celeba_visual_sngan_SP_RS}}
    		\\
    		\subfloat[][DRE-F-SP+MH (10 yellow squares)]{\includegraphics[width=0.45\textwidth, height=7.5cm]{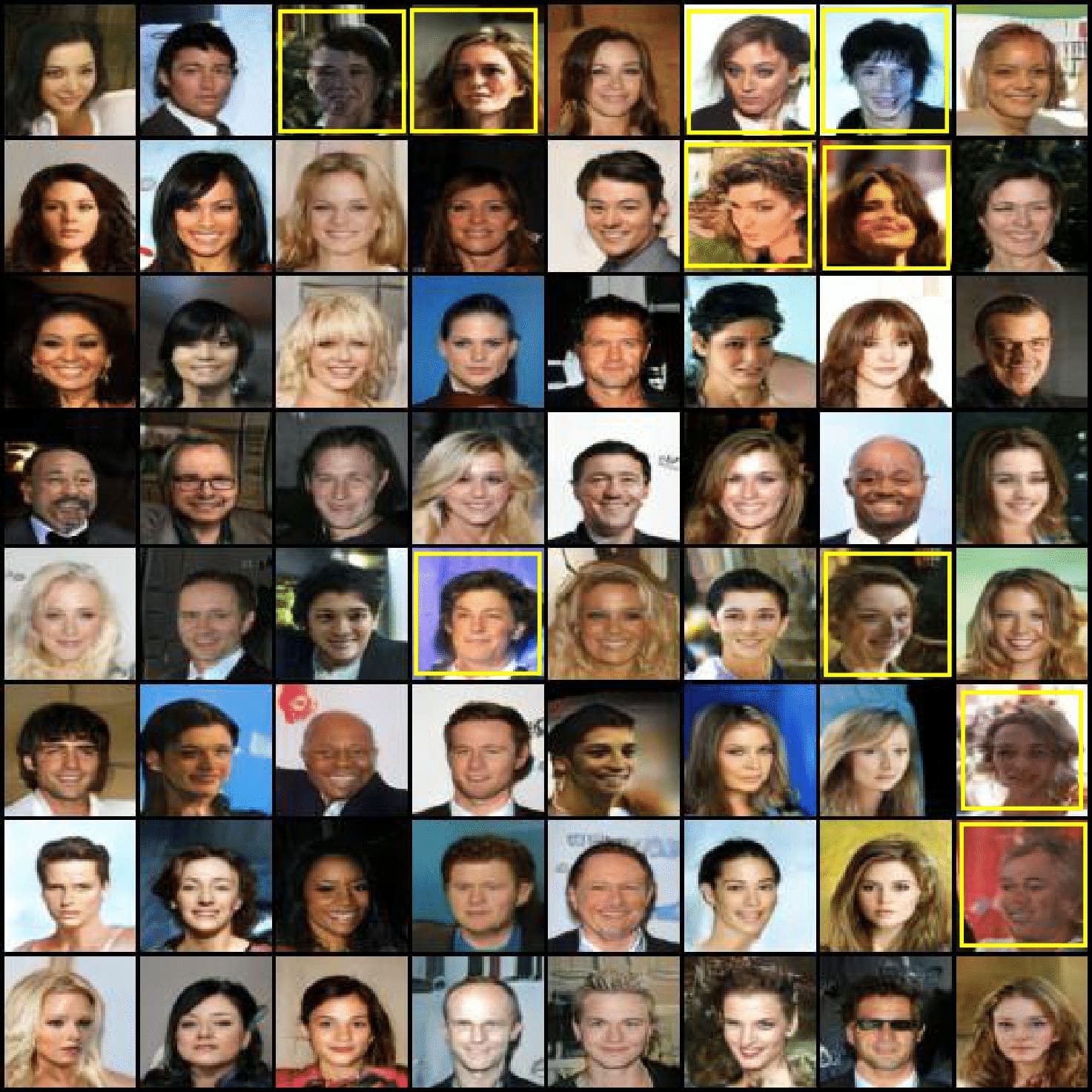}\label{fig:celeba_visual_sngan_SP_MH}}\quad
    		\subfloat[][DRE-F-SP+SIR (13 yellow squares)]{\includegraphics[width=0.45\textwidth, height=7.5cm]{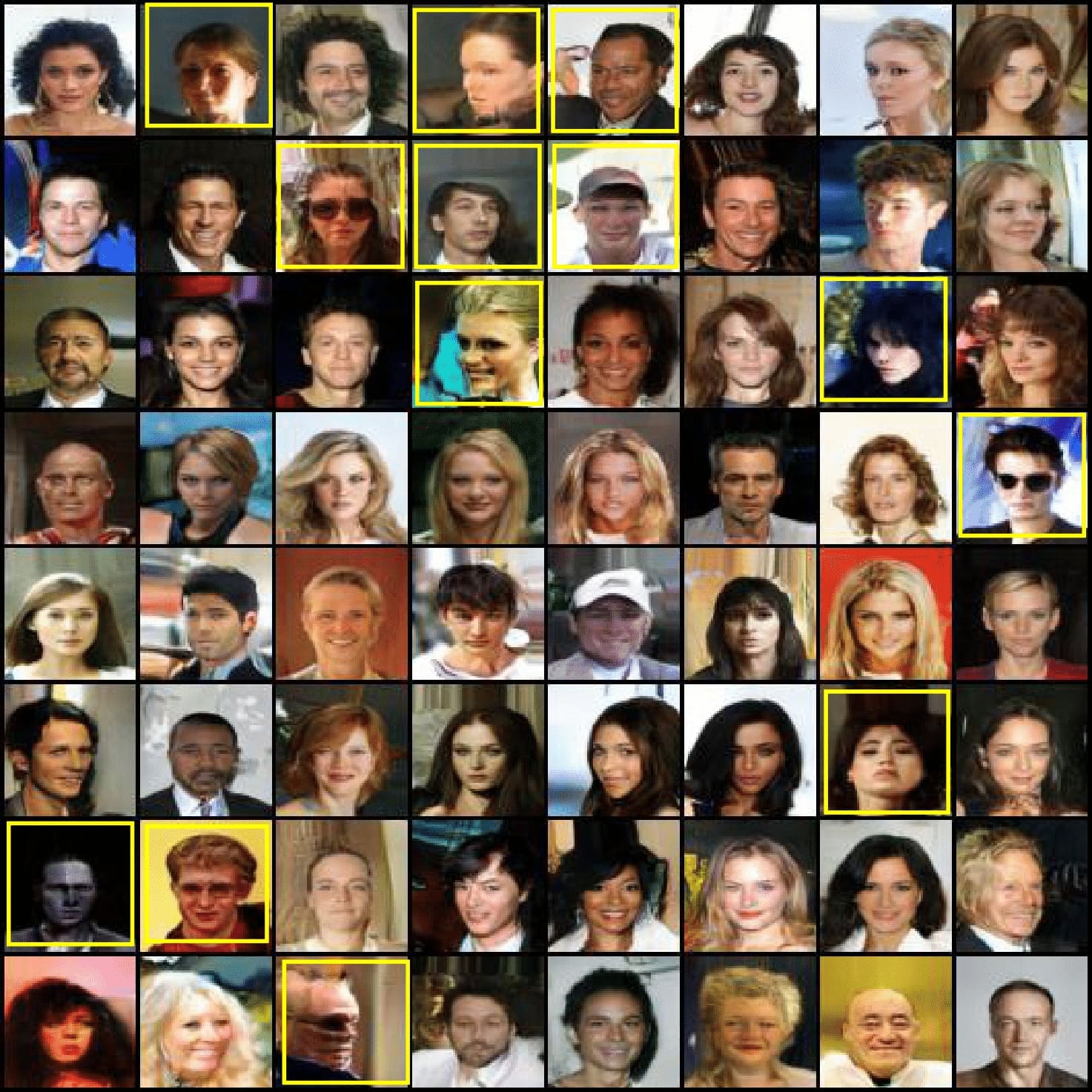}\label{fig:celeba_visual_sngan_SP_SIR}}
    		\caption{Fake CelebA images from SNGAN (yellow outlines indicate samples with obvious distortions and asymmetry)}.
    		\label{fig:celeba_visual_results_sngan}
    	\end{figure*}
    	
    \end{appendices}

\end{document}